\documentclass{article}

\usepackage{natbib}
\usepackage{iclr2027_conference,times}

\iclrfinalcopy

\usepackage{amssymb}
\usepackage{amsthm}
\usepackage{bm}
\usepackage{booktabs}
\usepackage{dsfont}
\usepackage{enumitem}
\usepackage{graphicx}
\usepackage{mathtools}
\usepackage{microtype}
\usepackage{nicefrac}
\usepackage{physics2}
\usepackage{silence}
\usepackage{thmtools}
\usepackage{titletoc}
\usepackage{xcolor}
\usepackage{xurl}

\usepackage[hidelinks,hypertexnames=false]{hyperref}
\usepackage{bookmark}
\usepackage[capitalize,nameinlink]{cleveref}

\usepackage[breakable]{tcolorbox}

\title{Adversarially Trained Linear Transformers Are Optimal Robust In-Context Learners for Gaussian Mixtures}

\author{Soichiro Kumano\\
LY Corporation\\
\texttt{sokumano@lycorp.co.jp}}

\colorlet{myred}{red!70!black}
\colorlet{myblue}{blue!70!black}

\hypersetup{
  colorlinks,
  citecolor = myblue,
  linkcolor = myred,
  urlcolor = myred
}

\crefname{assumption}{Assumption}{Assumptions}
\crefname{figure}{Fig.}{Figs.}
\crefname{table}{Tab.}{Tabs.}
\crefname{enumi}{}{}

\declaretheorem[style = plain, name = Theorem, numberwithin = section]{theorem}

\declaretheorem[style = plain, name = Lemma, sibling = theorem]{lemma}

\declaretheorem[style = definition, name = Assumption, sibling = theorem]{assumption}

\tcbset{
  theorem box/.style = {
    breakable,
    sharp corners,
    boxsep = 0mm,
    boxrule = 0mm,
    colframe = white,
    left = .5mm,
    right = .5mm,
  }
}

\tcolorboxenvironment{theorem}{theorem box}
\tcolorboxenvironment{corollary}{theorem box}
\tcolorboxenvironment{lemma}{theorem box}
\tcolorboxenvironment{proposition}{theorem box}
\tcolorboxenvironment{assumption}{theorem box}
\tcolorboxenvironment{definition}{theorem box}
\tcolorboxenvironment{setting}{theorem box}
\tcolorboxenvironment{restatable}{theorem box}
\tcolorboxenvironment{restatable*}{theorem box}

\allowdisplaybreaks

\usephysicsmodule{ab}

\begin{document}

\maketitle

\begin{abstract}
  Adversarial training is one of the most reliable defenses against adversarial attacks, but its high computational cost must generally be paid anew for each task.
  Robust foundation models offer a promising alternative: adversarially pretrain a model once and then transfer its robustness to downstream tasks through lightweight adaptation.
  However, a fundamental question remains open: can robustness acquired during pretraining transfer to unseen tasks without further adversarial training?
  In this study, we answer this question affirmatively.
  A single model adversarially pretrained at scale can achieve optimal robustness on new tasks without additional task-specific training.
  Specifically, we show that, for a family of Gaussian-mixture classification tasks, a sufficiently deep linear transformer adversarially trained across tasks can asymptotically attain the robust Bayes error on previously unseen tasks through in-context learning from clean demonstrations.
  By contrast, a standardly trained model cannot.
  We further analyze convergence under gradient flow, an accuracy--robustness trade-off, and demonstration complexity.
\end{abstract}

\section{Introduction}

Adversarial examples, inputs that are slightly perturbed to cause machine learning models to make incorrect predictions, expose a fundamental vulnerability of deep learning systems~\citep{AE}.
Adversarial training is one of the most reliable defenses against adversarial attacks~\citep{FGSM,PGD}.
It improves robustness by minimizing the loss under worst-case perturbations.
However, the resulting min--max optimization makes adversarial training substantially more expensive than standard training.
While many cheaper defenses have been proposed, some offer only apparent robustness that disappears under stronger attacks~\citep{CW,carlini2017adversarial,BreakICLR18,AutoAttack,AdaptiveAttack}.
Consequently, obtaining reliable robustness generally requires paying the cost of adversarial training anew for each task.

Meanwhile, foundation models have transformed how models are developed and deployed.
Rather than training a separate model for each task, practitioners pretrain a single model at scale and adapt it to downstream tasks using lightweight methods.
This paradigm suggests a compelling possibility: robust foundation models.
Adversarial training would be performed only once during pretraining, and the resulting robustness could then be transferred to downstream tasks at little additional cost.
Although adversarial pretraining incurs an initial cost, model developers~(typically large organizations) could recoup it by licensing the model.
The fundamental question here is:
\emph{Can an adversarially pretrained model adapt to unseen downstream tasks while retaining robustness, without further adversarial training?}

\citet{kumano2026adversarially} provided the first theoretical evidence that such transfer might be possible.
They showed that an adversarially pretrained single-layer transformer with linear attention can adapt robustly to unseen classification tasks through in-context learning.
In-context learning is a particularly lightweight adaptation method that enables transformers to learn from demonstrations in the prompt without updating their parameters~\citep{ICL}.
This result suggests that robustness acquired during pretraining can transfer to new tasks without additional training.
However, their analysis is limited to a restricted data model and a single-layer architecture.
Moreover, they provided only bounds on the expected robust margin.
Thus, it does not provide direct guarantees on robust accuracy or evaluate robust optimality.
Consequently, our central question remains open in broader settings.

In this study, we provide a stronger affirmative answer to this question under substantially more general conditions.
We consider multilayer linear transformers adversarially pretrained on a family of Gaussian-mixture classification tasks.
We focus on the robust zero-one error of these transformers on tasks not encountered during pretraining.
In particular, we investigate whether this error approaches the robust Bayes error, which is the minimum robust error attainable by any classifier, even one with full knowledge of the task distribution.

Specifically, we show that, through in-context learning from clean demonstrations, a sufficiently deep linear transformer can asymptotically attain the robust Bayes error on every task in this Gaussian family, including tasks unseen during pretraining.
In other words, using only standard samples and without task-specific adversarial training or access to adversarial examples, a single linear transformer can asymptotically serve as an optimal robust predictor for each task.
We should note that the transformer’s parameters remain fixed and shared across all downstream tasks.
Nevertheless, the transformer can adapt to each task and implement its optimal robust prediction rule.

Moreover, we show that a standardly pretrained transformer cannot achieve optimal robustness, even with infinite pretraining tasks and in-context demonstrations.
This result indicates that optimal robust in-context learning cannot be attributed solely to model capacity or task diversity; rather, it stems from an inherent advantage of adversarial pretraining.

We further investigate several properties of robust in-context learning.
Whereas the discussion above focuses on empirical-risk minimizers, we also prove that gradient flow in a two-layer architecture converges to a predictor that attains robust Bayes optimality.
We also identify two fundamental challenges in developing robust transformers.
The first is an inherent accuracy--robustness trade-off: adversarial pretraining can degrade clean accuracy even in the infinite-data limit.
The second is greater demonstration complexity: even at the same Bayes excess risk, robust prediction requires more demonstrations than standard prediction.

In summary, we prove the following:

\begin{itemize}[leftmargin = .5\leftmargini]
  \item A sufficiently deep linear transformer asymptotically achieves optimal robustness across a family of Gaussian-mixture classification tasks through in-context learning from clean demonstrations.

  \item Adversarial pretraining makes the transformer an optimal robust in-context learner, whereas standard pretraining does not, even with unlimited data.

  \item Under mild conditions, (i) gradient flow leads the transformer to become such an in-context learner, (ii) an accuracy--robustness trade-off arises, and (iii) robust in-context prediction requires more demonstrations than standard prediction.
\end{itemize}

\section{Related Work}

\textbf{Adversarial training.}
Numerous defenses against adversarial examples~\citep{AE} have been proposed, but many have subsequently been circumvented by stronger attacks~\citep{CW,carlini2017adversarial,BreakICLR18,AutoAttack,AdaptiveAttack}.
Although adversarial training~\citep{FGSM,PGD} is one of the most reliable defenses against a wide range of strong attacks, its computational cost is high due to multiple additional backpropagation steps beyond those needed for standard training.
Prior work has sought to reduce this cost by approximating multi-step backpropagation~\citep{FreeAT,YOPO,ATTA}, reducing the number of backpropagation steps~\citep{GradAlign,FastAT,kim2021understanding,SLAT,jia2022prior}, or adversarially finetuning models after standard training~\citep{AFT,DRA,ARREST,CURE}.
Nevertheless, all these approaches still require adversarial training to be performed separately for each task.
In contrast, we adversarially pretrain a linear transformer once and adapt it to downstream tasks through in-context learning from clean demonstrations.
This approach requires neither new adversarial examples nor adversarial training during downstream-task adaptation.

\textbf{In-context learning.}
In-context learning enables a transformer to learn a task from demonstrations in the prompt, without updating its parameters~\citep{ICL}.
Existing theoretical work on in-context learning has largely focused on the algorithms that transformers can implement in context.
For example, appropriately parameterized transformers can implement one gradient descent step per layer in context, so that propagation through multiple layers corresponds to multiple gradient descent steps~\citep{ahn2023transformers,bai2023transformers,cheng2024transformers,gatmiry2024can,mahankali2023one,zhang2024trained}.
In-context learning can also simulate second-order optimization~\citep{fu2024transformers,giannou2024well}, reinforcement learning algorithms~\citep{lee2023supervised,lin2024transformers}, and Bayesian inference~\citep{xie2022explanation,zhang2025and}.
Unlike these studies, we do not seek to identify the algorithms that in-context learning can simulate.
Instead, we study the ability of in-context learning to transfer robustness acquired in adversarial pretraining to downstream tasks.

\textbf{Adversarial robustness of in-context learning.}
The adversarial robustness of in-context learning remains relatively underexplored from a theoretical perspective.
\citet{anwar2025understanding} studied context hijacking and showed that modifying in-context tokens can cause a transformer to output an arbitrary value for a clean query.
\citet{li2026robustness} showed that deeper transformers are more robust to context hijacking.
\citet{fu2026short} showed that adversarial training with short adversarial suffixes in the context and clean queries can improve robustness to longer adversarial suffixes at test time.
\citet{fu2026understanding} used the theory of in-context learning to analyze continuous adversarial training~\citep{xhonneux2024efficient}.
Both our work and that of \citet{kumano2026adversarially} differ from these studies in terms of problem setting and objective.
Specifically, we ask whether a transformer can learn in context from clean demonstrations to make robust predictions on adversarially perturbed queries.
Following our theoretical results, we provide a detailed comparison with \citet{kumano2026adversarially} in \cref{sec:comparison}.

\section{Theoretical Results}

\textbf{Notation.}
For each $n \in \mathbb{N}$, define $[n] := \{1, \ldots, n\}$.
We denote the $i$-th element of a vector $\bm{a}$ by $a_{i}$ and the entry in row $i$ and column $j$ of a matrix $\bm{A}$ by $A_{i,j}$.
For $m, n \in \mathbb{N}$, let $\bm{0}_{n}$, $\bm{0}_{m,n}$, and $\bm{I}_{n}$ denote the $n$-dimensional zero vector, the $m \times n$ zero matrix, and the $n \times n$ identity matrix, respectively.
For $d \in \mathbb{N}$, define the real orthogonal group in dimension $d$ by $\mathcal{O}(d) := \{\bm{R} \in \mathbb{R}^{d \times d} : \bm{R}^{\top} \bm{R} = \bm{I}_{d}\}$.
We write $\|\,\cdot\,\|_{2}$ for the Euclidean norm and $\|\,\cdot\,\|_{\mathrm{F}}$ for the Frobenius norm.
Let $(X_{\bm{n}})_{\bm{n} \in \mathbb{N}^{m}}$ be nonnegative random variables and $(a_{\bm{n}})_{\bm{n} \in \mathbb{N}^{m}}$ be positive deterministic values; we write $X_{\bm{n}} \lesssim_{\mathbb{P}} a_{\bm{n}}$ if, for every $\delta \in (0, 1)$, there exists $c_{\delta} > 0$ such that $\mathbb{P}(X_{\bm{n}} \leq c_{\delta} a_{\bm{n}}) \geq 1 - \delta$ for all $\bm{n} \in \mathbb{N}^{m}$.

\subsection{Problem Setting}

\textbf{Overview.}
We consider pretraining an $L$-layer linear transformer $f$ on $T$ tasks.
For each $t \in [T]$, the $t$-th task consists of $N$ in-context demonstrations $\mathcal{D}_{t,N} := \{(\bm{x}_{t,n}, y_{t,n})\}_{n=1}^{N}$ and a test pair $(\bm{x}_{t,N+1}, y_{t,N+1})$.
Given the clean demonstrations $\mathcal{D}_{t,N}$, the transformer uses in-context learning to predict the binary label of the adversarially perturbed query $\bm{x}_{t,N+1} + \bm{\Delta}$.
The adversary chooses $\bm{\Delta}$ subject to $\|\bm{\Delta}\|_{2} \leq \epsilon \sqrt{d}$.
We train the transformer parameters $\bm{\Theta}_{L}$ by minimizing the empirical risk $\mathcal{R}_{\epsilon,T,N}(\bm{\Theta}_{L})$.
We then evaluate the resulting predictor on a freshly sampled test task using the taskwise zero-one error $\mathcal{E}_{\epsilon,\bm{U}}(f(\,\cdot\,; \bm{\Theta}_{L}, \mathcal{D}_{N}))$, where $\bm{U}$ denotes the task-specific variable.
Our main interest is the excess error $\mathcal{E}_{\epsilon,\bm{U}}(f(\,\cdot\,; \bm{\Theta}_{L}, \mathcal{D}_{N})) - \mathcal{E}^{\star}_{\epsilon,\bm{U}}$, where $\mathcal{E}^{\star}_{\epsilon,\bm{U}}$ denotes the corresponding taskwise Bayes error.

\textbf{Task structure.}
Let $T, N \in \mathbb{N}$ denote the number of pretraining tasks and the number of in-context demonstrations per task, respectively.
For each $t \in [T]$, the $t$-th task consists of $N$ demonstrations $\mathcal{D}_{t,N} := \{(\bm{x}_{t,n}, y_{t,n})\}_{n=1}^{N} \subset \mathbb{R}^{d} \times \{\pm 1\}$ and a test pair $(\bm{x}_{t,N+1}, y_{t,N+1})$.
Let $\bm{U}_{t} := [\bm{u}_{t,1}, \ldots, \bm{u}_{t,d}] \in \mathcal{O}(d)$ denote an orthogonal matrix whose columns form a task-specific orthonormal basis.
For each $(t, n) \in [T] \times [N+1]$, let $\bm{a}_{t,n} := [a_{t,n,1}, \ldots, a_{t,n,d}]^{\top} \in \mathbb{R}^{d}$ denote the coefficient vector for the $n$-th sample of the $t$-th task, and define the corresponding input vector as
\begin{equation}
  \bm{x}_{t,n} := y_{t,n} \bm{U}_{t} \bm{a}_{t,n}
  = y_{t,n}(a_{t,n,1} \bm{u}_{t,1} + \cdots + a_{t,n,d} \bm{u}_{t,d}).
  \label{eq:task-structure}
\end{equation}

\textbf{Task distribution.}
The matrices $\{\bm{U}_{t}\}_{t=1}^{T}$ are independently drawn according to the normalized Haar measure on $\mathcal{O}(d)$.
In other words, the columns of each $\bm{U}_{t}$ form an orthonormal basis chosen uniformly at random.
The labels $\{y_{t,n}\}_{t \in [T], n \in [N+1]}$ are independently drawn from the uniform distribution on $\{\pm 1\}$.
The coefficient vectors $\{\bm{a}_{t,n}\}_{t \in [T], n \in [N+1]}$ are independently drawn from the multivariate normal distribution with mean $\bm{\mu} \in \mathbb{R}^{d}$ and covariance $\bm{\Sigma} \in \mathbb{R}^{d \times d}$.
These three collections are mutually independent.
The parameters $\bm{\mu}$ and $\bm{\Sigma}$ are common across tasks.

\textbf{Dimension of the Krylov subspace.}
Let $r \in [d]$ denote the dimension of $\mathcal{K}_{d}(\bm{\Sigma}, \bm{\mu})$, the order-$d$ Krylov subspace generated by $\bm{\Sigma}$ and $\bm{\mu}$.
More precisely, define
\begin{equation}
  \mathcal{K}_{d}(\bm{\Sigma}, \bm{\mu})
  := \operatorname{span}(\bm{\mu}, \bm{\Sigma} \bm{\mu}, \bm{\Sigma}^{2} \bm{\mu},
  \ldots, \bm{\Sigma}^{d-1} \bm{\mu}), \qquad
  r := \dim \mathcal{K}_{d}(\bm{\Sigma}, \bm{\mu}).
  \label{eq:krylov-dim}
\end{equation}
Equivalently, $r$ is the largest integer $k \in [d]$ such that the vectors $\bm{\mu}, \bm{\Sigma}\bm{\mu}, \ldots, \bm{\Sigma}^{k-1}\bm{\mu}$ are linearly independent.
Consequently, $\mathcal{K}_{d}(\bm{\Sigma}, \bm{\mu}) = \mathcal{K}_{r}(\bm{\Sigma}, \bm{\mu})$.
Moreover, $r$ equals the number of distinct eigenvalues of $\bm{\Sigma}$ for which $\bm{\mu}$ has a nonzero projection onto the corresponding eigenspace.
If $\bm{\Sigma} = \bm{I}_{d}$, so that the Gaussian coordinates are mutually independent and have unit variance, then $r = 1$.
If $\bm{\mu}$ is the all-ones vector and $\bm{\Sigma}$ is diagonal, then $r$ equals the number of distinct diagonal entries of $\bm{\Sigma}$.
In general, $r$ characterizes the representational complexity of the Bayes classifier for the Gaussian data $\bm{z} := y \bm{a}$.
Without adversarial perturbations, this classifier is $(\bm{\Sigma}^{-1} \bm{\mu})^{\top} \bm{z}$.
Its weight vector $\bm{\Sigma}^{-1}\bm{\mu}$ belongs to $\mathcal{K}_{r}(\bm{\Sigma}, \bm{\mu})$, and $r$ is the smallest Krylov order required to represent this vector.

\textbf{Input sequence.}
Given the clean in-context demonstrations $\mathcal{D}_{N}$ and a possibly perturbed query $\bm{z} \in \mathbb{R}^{d}$, we define the input sequence $\bm{Z}^{(0)}$ for a linear transformer as
\begin{equation}
  \bm{Z}^{(0)} :=
  \begin{bmatrix}
    \bm{x}_{1} & \bm{x}_{2} & \cdots & \bm{x}_{N} & \bm{z} \\
    y_{1} & y_{2} & \cdots & y_{N} & 0
  \end{bmatrix}
  \in \mathbb{R}^{(d+1) \times (N+1)}.
  \label{eq:input-sequence}
\end{equation}
The bottom-right entry serves as a placeholder for the query prediction.

\textbf{Linear transformer.}
In this study, we focus on transformers with linear self-attention, which are commonly used in recent theoretical work~\citep{ahn2023transformers,mahankali2023one,gatmiry2024can,cheng2024transformers,zhang2024trained,kumano2026adversarially}.
Let $L \in \mathbb{N}$ denote the number of layers.
For each $\ell \in [L]$, we recursively define
\begin{equation}
  \bm{Z}^{(\ell)}
  := \bm{Z}^{(\ell-1)}
  + \frac{1}{N} \bm{P}_{\ell} \bm{Z}^{(\ell-1)} \bm{M}
  \bm{Z}^{(\ell-1)\top} \bm{Q}_{\ell}
  \bm{Z}^{(\ell-1)}, \qquad
  \bm{M} :=
  \begin{bmatrix}
    \bm{I}_{N} & \bm{0}_{N} \\
    \bm{0}_{N}^{\top} & 0
  \end{bmatrix},
  \label{eq:multilayer-linear-transformer}
\end{equation}
where $\bm{P}_{\ell}, \bm{Q}_{\ell} \in \mathbb{R}^{(d+1) \times (d+1)}$ are, respectively, the trainable value and key--query matrices in layer $\ell$.
As in prior work on in-context learning~\citep{ahn2023transformers,gatmiry2024can,cheng2024transformers,kumano2026adversarially}, we use the masking matrix $\bm{M}$ to prevent the tokens from attending to the query token.
The resulting scalar predictor and the collection of trainable parameters are defined as
\begin{equation}
  f(\bm{z}; \bm{\Theta}_{L}, \mathcal{D}_{N})
  := \ab[\bm{Z}^{(L)}]_{d+1,N+1}, \qquad
  \bm{\Theta}_{L} := {\{(\bm{P}_{\ell}, \bm{Q}_{\ell})\}}_{\ell=1}^{L}.
  \label{eq:multilayer-scalar-predictor}
\end{equation}

\textbf{Training objective.}
Let $\mathcal{L}: \mathbb{R} \to [0, \infty)$ be a loss function.
Based on prior work~\citep{bai2023transformers,ahn2023transformers,mahankali2023one,zhang2024trained,kumano2026adversarially}, we define the empirical risk over $T$ pretraining tasks, each with $N$ demonstrations, as follows:
\begin{equation}
  \mathcal{R}_{\epsilon,T,N}(\bm{\Theta}_{L})
  := \frac{1}{T} \sum_{t=1}^{T}
  \max_{\|\bm{\Delta}\|_{2} \leq \epsilon \sqrt{d}}
  \mathcal{L}\Big(
    y_{t,N+1} f(\bm{x}_{t,N+1} + \bm{\Delta}; \bm{\Theta}_{L}, \mathcal{D}_{t,N})
  \Big).
  \label{eq:main-finite-sample-risk}
\end{equation}
Minimizing $\mathcal{R}_{0,T,N}$ corresponds to standard training, whereas minimizing $\mathcal{R}_{\epsilon,T,N}$ for $\epsilon > 0$ corresponds to adversarial training.
Training the transformer on multiple tasks with a single shared set of parameters encourages it to adaptively infer the task-specific structure from the demonstrations and generalize across tasks, rather than specialize in any particular task.

\textbf{Taskwise zero-one error.}
For each fixed $\bm{U} \in \mathcal{O}(d)$, let $(\bm{x}, y)$ be a fresh test pair independently drawn from the same task-conditional distribution as the pretraining samples.
We define the taskwise robust zero-one error of a predictor $g: \mathbb{R}^{d} \to \mathbb{R}$ by
\begin{equation}
  \mathcal{E}_{\epsilon, \bm{U}}(g)
  := \mathbb{P}_{(\bm{x}, y) \mid \bm{U}}
  \ab(
    \inf_{\|\bm{\Delta}\|_{2} \leq \epsilon \sqrt{d}}
    y g(\bm{x} + \bm{\Delta}) \leq 0
  ).
  \label{eq:robust-zero-one-error}
\end{equation}
Setting $\epsilon = 0$ yields the taskwise standard zero-one error $\mathcal{E}_{0,\bm{U}}$.

\textbf{Taskwise Bayes error.}
Let $\mathcal{G}$ denote the class of all measurable functions $g: \mathbb{R}^{d} \to \mathbb{R}$.
For each $\bm{U} \in \mathcal{O}(d)$, we define the taskwise robust Bayes error by $\mathcal{E}_{\epsilon,\bm{U}}^{\star} := \inf_{g \in \mathcal{G}} \mathcal{E}_{\epsilon,\bm{U}}(g)$.
Setting $\epsilon = 0$ yields the taskwise standard Bayes error $\mathcal{E}_{0,\bm{U}}^{\star}$.

\textbf{Assumptions.}
The assumptions used in this study are as follows:

\begin{assumption}\label{asm:summary}
  \
  \begin{enumerate}[label = {(\alph*)}, leftmargin = *]
    \item\label{asm:mean-covariance} The mean vector $\bm{\mu}$ is nonzero, and the covariance matrix $\bm{\Sigma}$ is symmetric and positive definite; $\bm{\mu} \ne \bm{0}_{d}$ and $\bm{\Sigma} = \bm{\Sigma}^{\top} \succ \bm{0}_{d,d}$.

    \item\label{asm:perturbation-radius} The perturbation radius $\epsilon \sqrt{d}$ is smaller than the norm of the mean vector $\bm{\mu}$; $0 \leq \epsilon \sqrt{d} < \|\bm{\mu}\|_{2}$.

    \item\label{asm:depth-condition} The number of layers $L$ is at least the dimension $r$ of the Krylov subspace; $L \geq r$.

    \item\label{asm:loss} The loss function $\mathcal{L}: \mathbb{R} \to [0, \infty)$ is twice differentiable and satisfies $\mathcal{L}'(z) < 0$ and $\mathcal{L}''(z) > 0$ for every $z \in \mathbb{R}$, and $\sup_{z \in \mathbb{R}}\{|\mathcal{L}'(z)| + \mathcal{L}''(z)\} < \infty$.
  \end{enumerate}
\end{assumption}

Condition~\cref{asm:mean-covariance} rules out degenerate classification problems.
Condition~\cref{asm:perturbation-radius} prevents an adversary from completely canceling the mean signal.
The role of Condition~\cref{asm:depth-condition} is explained in the next section.
Condition~\cref{asm:loss} ensures that the empirical risk minimization problem is well-posed.
The logistic loss is an example satisfying Condition~\cref{asm:loss}.

\subsection{A Single Linear Transformer Can Represent Robust Bayes Classifiers Across Tasks}

In this section, we first show that, given sufficient depth, a single linear transformer can asymptotically achieve the taskwise robust Bayes error through in-context learning.
At this stage, we do not address whether the required parameters can be learned by empirical risk minimization.

\begin{restatable}[A single linear transformer can be Bayes optimal across tasks]{theorem}{robustbayesrepresentability}\label{thm:main-robust-bayes-representability}
  Suppose that \cref{asm:mean-covariance,asm:perturbation-radius,asm:depth-condition} in \cref{asm:summary} hold.
  Then, there exists a parameter $\bm{\Theta}_{L}$ such that, for every $\bm{U} \in \mathcal{O}(d)$,
  \begin{equation}
    \mathcal{E}_{\epsilon, \bm{U}}\Big(
      f(\,\cdot\,; \bm{\Theta}_{L}, \mathcal{D}_{N})
    \Big)
    - \mathcal{E}_{\epsilon, \bm{U}}^{\star}
    \lesssim_{\mathbb{P}} \frac{1}{N}.
    \label{eq:main-robust-bayes-representability-rate}
  \end{equation}
  The implied constant may depend on $d$, $\epsilon$, $\bm{\mu}$, and $\bm{\Sigma}$.
\end{restatable}

The proof is given in \cref{sec:appendix-robust-bayes-implementation}.
The theorem establishes that a single linear transformer with a shared parameter asymptotically attains the robust Bayes error for every task as the number of demonstrations grows.
In other words, without any task-specific parameter updates, the same linear transformer can adapt to every task through in-context learning and match the performance of the corresponding task-specific robust Bayes classifier.
This uniform guarantee demonstrates the expressive capacity of linear transformers for in-context adaptation across tasks.
Moreover, the excess error decays at a rate of $N^{-1}$.
The rate is consistent with empirical observations of strong few-shot performance in in-context learning.
Recall that $r$ reflects the complexity of the Gaussian mixture model and the resulting difficulty of representing its robust Bayes classifier.
Accordingly, the condition $L \geq r$ suggests that more complex data models call for deeper linear transformers to achieve Bayes optimality.

We next provide intuition about how a linear transformer can attain the Bayes error and why the depth condition $L \geq r$ arises.

\begin{restatable}[Proof sketch of \cref{thm:main-robust-bayes-representability}]{lemma}{robustbayesrealizationmechanism}\label{lem:main-robust-bayes-realization-mechanism}
  Suppose that \cref{asm:mean-covariance,asm:perturbation-radius} in \cref{asm:summary} hold.
  Suppose also that $N \to \infty$; write $\mathcal{D}_{\infty} := \mathcal{D}_{N}$.
  Let $\tau \geq 0$ be the unique solution of
  \begin{equation}
    \tau \|(\bm{\Sigma} + \tau \bm{I}_{d})^{-1} \bm{\mu}\|_{2} = \epsilon \sqrt{d}.
  \end{equation}

  \begin{enumerate}[label = {}, leftmargin = *]
    \item
      \textbf{(Bayes classifier)}
      For every $\bm{U} \in \mathcal{O}(d)$,
      \begin{equation}
        g(\bm{x})
        := \{ (\bm{\Sigma} + \tau \bm{I}_{d})^{-1} \bm{\mu} \}^{\top}
        \bm{U}^{\top} \bm{x}, \qquad
        \mathcal{E}_{\epsilon,\bm{U}}(g) = \mathcal{E}_{\epsilon,\bm{U}}^{\star}.
      \end{equation}

    \item
      \textbf{(Krylov order of the optimal direction)}
      Let $\mathcal{K}_{0}(\bm{\Sigma}, \bm{\mu}) := \{\bm{0}\}$.
      Then,
      \begin{equation}
        (\bm{\Sigma} + \tau \bm{I}_{d})^{-1} \bm{\mu}
        \in \mathcal{K}_{r}(\bm{\Sigma}, \bm{\mu})
        \setminus \mathcal{K}_{r-1}(\bm{\Sigma}, \bm{\mu}).
      \end{equation}

    \item
      \textbf{(Effective weights of a linear transformer)}
      There exists a parameter family $\{\bm{\Theta}_{L,\xi}: \xi \in \mathbb{R}\}$ whose corresponding effective weights $\bm{\theta}(\bm{\Theta}_{L,\xi})$ satisfy, for every $\xi \in \mathbb{R}$, $\bm{U} \in \mathcal{O}(d)$, and $\bm{x} \in \mathbb{R}^{d}$,
      \begin{equation}
        f(\bm{x}; \bm{\Theta}_{L,\xi}, \mathcal{D}_{\infty})
        =
        \bm{\theta}(\bm{\Theta}_{L,\xi})^{\top}
        \bm{U}^{\top} \bm{x}.
      \end{equation}

    \item
      \textbf{(Span of the effective weights)}
      For the same parameter family,
      \begin{equation}
        \operatorname{span}\ab\{
          \bm{\theta}(\bm{\Theta}_{L,\xi}): \xi \in \mathbb{R}
        \} = \mathcal{K}_{L}(\bm{\Sigma}, \bm{\mu}).
      \end{equation}
  \end{enumerate}
\end{restatable}

The proof is given in \cref{sec:appendix-robust-bayes-implementation}.
Although this lemma does not constitute a complete proof of \cref{thm:main-robust-bayes-representability}, it provides a roadmap for the proof.

\textbf{Bayes classifier.}
The first part of the lemma gives the Bayes classifier for $(\bm{x}, y)$ generated according to \cref{eq:task-structure}.
The form of this function provides two insights into this classifier.\footnote{
  Closely related characterizations of the Bayes classifier were previously established by \citet{bhagoji2019lower,dan2020sharp}, and we therefore do not claim this characterization as a novel contribution.
  Here, it serves only as an intermediate result in the proof of \cref{thm:main-robust-bayes-representability}.
}

First, the optimal weight is the solution to the following regularized mean--variance problem:
\begin{equation}
  \operatorname*{argmin}_{\bm{w} \in \mathbb{R}^{d}}\ab\{
    \frac{1}{2} \mathbb{V}_{(\bm{x}, y) \mid \bm{U}}[y \bm{w}^{\top} \bm{x}]
    - \mathbb{E}_{(\bm{x}, y) \mid \bm{U}}[y \bm{w}^{\top} \bm{x}]
    + \frac{\tau}{2} \|\bm{w}\|_{2}^{2}
  \}
  = \bm{U}(\bm{\Sigma} + \tau \bm{I}_{d})^{-1} \bm{\mu}.
  \label{eq:main-robust-direction-mean-variance}
\end{equation}
The regularization parameter $\tau$ is determined by the perturbation level $\epsilon$: it is zero when $\epsilon = 0$ and increases monotonically with $\epsilon$.
This $\ell_{2}$ regularization controls the sensitivity to adversarial perturbations, since $|\bm{w}^{\top} \bm{\Delta}| \leq \|\bm{w}\|_{2} \|\bm{\Delta}\|_{2}$.
The taskwise Bayes classifier minimizing $\mathcal{E}_{\epsilon,\bm{U}}$ is a linear classifier obtained from a perturbation-dependent $\ell_{2}$-regularized optimization problem.

Second, the classifier makes predictions while accounting for both variance and the effects of adversarial attacks.
Substituting $\bm{x} = y \bm{U} \bm{a}$ gives
\begin{equation}
  g(\bm{x})
  := \ab\{
    (\bm{\Sigma} + \tau \bm{I}_d)^{-1} \bm{\mu}
  \}^{\top} \bm{U}^{\top} \bm{x}
  = y \ab\{
    (\bm{\Sigma} + \tau \bm{I}_d)^{-1} \bm{\mu}
  \}^{\top} \bm{a}.
\end{equation}
For simplicity, suppose that $\bm{\Sigma} = \operatorname{diag}(\lambda_1, \ldots, \lambda_d)$ for $\lambda_1, \ldots, \lambda_d > 0$.
Then the $i$-th term in the inner product, for $i \in [d]$, can be expressed as $\mu_i a_i / (\lambda_i + \tau)$.
When the perturbation level is small and hence $\tau$ is small, this term can be approximated by $\mu_i a_i / \lambda_i$.
Thus, components with large variance have less influence on the prediction, whereas components with small variance have greater influence.
In contrast, when $\tau$ is large, $g(\bm{x}) \approx y \bm{\mu}^{\top} \bm{a} / \tau$.
The classifier relies more directly on strong features in the data without accounting for their variances.

\textbf{Krylov order of the optimal direction.}
The second part of the lemma shows that the optimal direction belongs to the order-$r$ Krylov subspace generated by $\bm{\Sigma}$ and $\bm{\mu}$.
This vector cannot belong to the order-$(r-1)$ Krylov subspace.
The proof uses the Cayley--Hamilton theorem.
This result suggests that, to attain the taskwise Bayes error for arbitrary $\bm{\Sigma}$ and $\bm{\mu}$, a linear transformer must be expressive enough to internally represent weight vectors in the order-$r$ Krylov subspace.

\textbf{Effective weights of a linear transformer.}
The third part of the lemma shows that a suitable parameter family makes a linear transformer behave as a context-dependent linear classifier.
We restrict $\bm{P}_{\ell,\xi}$ and $\bm{Q}_{\ell,\xi}$ for $\ell \in [L]$ as follows, and, under $N \to \infty$, the context matrix is as follows:
\begin{equation}
  \bm{P}_{\ell,\xi} :=
  \begin{bmatrix}
    \bm{0}_{d,d} & \bm{0}_{d} \\
    \bm{0}_{d}^{\top} & 1
  \end{bmatrix}, ~
  \bm{Q}_{\ell,\xi} :=
  \begin{bmatrix}
    \xi \bm{I}_{d} & \bm{0}_{d} \\
    \bm{0}_{d}^{\top} & 0
  \end{bmatrix}, ~
  \frac{\bm{Z}^{(0)} \bm{M} \bm{Z}^{(0)\top}}{N}
  =
  \begin{bmatrix}
    \bm{U} (\bm{\Sigma} + \bm{\mu} \bm{\mu}^{\top}) \bm{U}^{\top}
    & \bm{U} \bm{\mu} \\
    \bm{\mu}^{\top} \bm{U}^{\top} & 1
  \end{bmatrix}.
  \label{eq:main-effective-weight-parameter-family}
\end{equation}
The outputs after the first, second, and third layers are
\begin{align}
  \bm{Z}^{(1)}
  &=
  \ab\{
    \bm{I}_{d+1} + \frac{1}{N} \bm{P}_{1,\xi}
    \bm{Z}^{(0)} \bm{M} \bm{Z}^{(0)\top}
    \bm{Q}_{1,\xi}
  \} \bm{Z}^{(0)}
  =
  \begin{bmatrix}
    \bm{I}_{d} & \bm{0}_{d} \\
    \xi \bm{\mu}^{\top} \bm{U}^{\top} & 1
  \end{bmatrix}
  \bm{Z}^{(0)}, \\
  \bm{Z}^{(2)}
  &=
  \begin{bmatrix}
    \bm{I}_{d} & \bm{0}_{d} \\
    \{ 2 \xi \bm{\mu}^{\top}
      + \xi^{2} \bm{\mu}^{\top}(
        \bm{\Sigma} + \bm{\mu} \bm{\mu}^{\top}
      )
    \} \bm{U}^{\top} & 1
  \end{bmatrix}
  \bm{Z}^{(0)}, \\
  \bm{Z}^{(3)}
  &=
  \begin{bmatrix}
    \bm{I}_{d} & \bm{0}_{d} \\
    \{
      3 \xi \bm{\mu}^{\top}
      + 3 \xi^{2} \bm{\mu}^{\top}(
        \bm{\Sigma} + \bm{\mu} \bm{\mu}^{\top}
      )
      + \xi^{3} \bm{\mu}^{\top}{
        (\bm{\Sigma} + \bm{\mu} \bm{\mu}^{\top})
      }^{2}
    \} \bm{U}^{\top} & 1
  \end{bmatrix}
  \bm{Z}^{(0)}.
  \label{eq:main-first-three-layer-outputs}
\end{align}
Consequently, the three-layer model's output on a query $\bm{x}$ and its effective weight are
\begin{equation}
  f(\bm{x}; \bm{\Theta}_{3,\xi}, \mathcal{D}_{\infty})
  = \bm{\theta}(\bm{\Theta}_{3,\xi})^{\top} \bm{U}^{\top}\bm{x}, ~
  \bm{\theta}(\bm{\Theta}_{3,\xi})
  := 3 \xi \bm{\mu}
  + 3 \xi^{2}(\bm{\Sigma} + \bm{\mu} \bm{\mu}^{\top}) \bm{\mu}
  + \xi^{3}{
    (\bm{\Sigma} + \bm{\mu} \bm{\mu}^{\top})
  }^{2} \bm{\mu}.
  \label{eq:main-three-layer-effective-weight}
\end{equation}
The same construction shows inductively that, with suitable parameters, a linear transformer of any depth can act as a context-dependent linear classifier whose effective weight $\bm{\theta}$ is independent of $\bm{U}$.

\textbf{Span of the effective weights.}
The fourth part of the lemma states that, as $\xi$ varies, the effective weights associated with this parameter family span the order-$L$ Krylov subspace.
The equations above illustrate this mechanism: stacking $L$ layers yields effective weights that are linear combinations of $\bm{\mu}, \bm{\Sigma}\bm{\mu}, \ldots, \bm{\Sigma}^{L-1}\bm{\mu}$.
As shown in the second part, the optimal direction lies in the order-$r$ Krylov subspace.
Hence, when $L \geq r$, the optimal direction lies in the span of the effective weights.
These results explain how linear transformers achieve Bayes optimality and why the depth condition is required.
Note that, even when $L \geq r$, this lemma alone does not establish that an effective weight $\bm{\theta}$ can exactly match the optimal direction $(\bm{\Sigma} + \tau \bm{I}_{d})^{-1} \bm{\mu}$.
See \cref{sec:appendix-robust-bayes-implementation} for the complete proof.

\subsection{Adversarial Training Enables Linear Transformers to Achieve Robust Bayes Optimality Across Tasks; Standard Training Does Not}

The preceding section establishes that sufficiently deep linear transformers have the expressive capacity to asymptotically attain robust Bayes optimality across tasks.
However, this existence result does not show whether such parameters can be learned by a training procedure.
We address this question by analyzing empirical risk minimization under adversarial and standard training.

\begin{restatable}[Adversarial training yields robust linear transformers]{theorem}{adversarialtrainingbayesoptimality}\label{thm:main-adversarial-training-bayes-optimality}
  Suppose that \cref{asm:summary} holds.
  For sufficiently large $B > 0$, choose any empirical risk minimizer
  \begin{equation}
    \widehat{\bm{\Theta}}_{L, \epsilon, T, N} \in
    \operatorname*{argmin}_{
      \max_{\ell \in [L]}\{
        \|\bm{P}_{\ell}\|_{\mathrm{F}},
        \|\bm{Q}_{\ell}\|_{\mathrm{F}}
      \} \leq B
    } \mathcal{R}_{\epsilon, T, N}(\bm{\Theta}_{L}).
    \label{eq:main-adversarial-erm}
  \end{equation}
  Then, for every $\bm{U} \in \mathcal{O}(d)$,
  \begin{equation}
    \mathcal{E}_{\epsilon, \bm{U}}
    \bigg(
      f\Big(
        \,\cdot\,; \widehat{\bm{\Theta}}_{L, \epsilon, T, N}, \mathcal{D}_{N}
      \Big)
    \bigg)
    - \mathcal{E}_{\epsilon, \bm{U}}^{\star}
    \lesssim_{\mathbb{P}} \frac{1}{\sqrt{T}} + \frac{1}{N}.
    \label{eq:main-adversarial-training-probability-rate}
  \end{equation}
  The implied constant may depend on $d$, $B$, $L$, $\mathcal{L}$, $\epsilon$, $\bm{\mu}$, and $\bm{\Sigma}$.
\end{restatable}

The proof is given in \cref{sec:appendix-adversarial-training}.
This theorem complements the representability result in \cref{thm:main-robust-bayes-representability} by showing that adversarial training can produce a transformer that asymptotically achieves robust Bayes optimality across tasks.
For $d \geq 2$, a test-task basis $\bm{U}$ differs almost surely from every pretraining basis under the Haar measure on $\mathcal{O}(d)$, even when there are countably infinitely many pretraining tasks.
Thus, the theorem establishes robust generalization to previously unseen tasks, which cannot be attributed to exact memorization of the pretraining bases.
Moreover, the excess-error bound decomposes into a term depending on $T$ and a term depending on $N$, rather than taking a joint form such as $1/(TN)$.
This decomposition isolates the contributions of cross-task learning from pretraining and within-task inference from in-context demonstrations.
The bound converges to zero only as both the number of pretraining tasks and the number of in-context demonstrations grow; increasing either one alone might be insufficient for the convergence to Bayes optimality.

We next contrast this result with standard training.

\begin{restatable}[Standard training does not yield robust linear transformers]{theorem}{standardtrainingbayesoptimality}\label{thm:main-standard-training-bayes-optimality}
  Suppose that \cref{asm:summary} holds.
  Suppose also that $T, N \to \infty$; write $\mathcal{D}_{\infty} := \mathcal{D}_{N}$ and $\mathcal{R}_{0, \infty, \infty} := \mathcal{R}_{0, T, N}$.
  For sufficiently large $B > 0$, choose any population risk minimizer
  \begin{equation}
    \widehat{\bm{\Theta}}_{L, 0, \infty, \infty} \in
    \operatorname*{argmin}_{
      \max_{\ell \in [L]}\{
        \|\bm{P}_{\ell}\|_{\mathrm{F}},
        \|\bm{Q}_{\ell}\|_{\mathrm{F}}
      \} \leq B
    } \mathcal{R}_{0, \infty, \infty}(\bm{\Theta}_{L}).
    \label{eq:main-standard-erm}
  \end{equation}
  Then, there exists a constant $C > 0$ such that, for every $\bm{U} \in \mathcal{O}(d)$,
  \begin{equation}
    \mathcal{E}_{\epsilon, \bm{U}}
    \bigg(
      f\Big(
        \,\cdot\,; \widehat{\bm{\Theta}}_{L, 0, \infty, \infty}, \mathcal{D}_{\infty}
      \Big)
    \bigg)
    - \mathcal{E}_{\epsilon, \bm{U}}^{\star}
    =
    \begin{cases}
      0 & (\epsilon = 0 ~ \mathrm{or} ~ r = 1) \\
      C & (\epsilon > 0 ~ \mathrm{and} ~ r \geq 2)
    \end{cases}.
    \label{eq:main-standard-training-probability-rate}
  \end{equation}
  The constant $C$ may depend on $d$, $\epsilon$, $\bm{\mu}$, and $\bm{\Sigma}$.
\end{restatable}

The proof is given in \cref{sec:appendix-adversarial-training}.
In contrast to \cref{thm:main-adversarial-training-bayes-optimality}, this theorem establishes an impossibility result for standard training: when $\epsilon > 0$ and $r \geq 2$, standard training cannot produce a linear transformer whose predictions attain robust Bayes optimality.
This impossibility is not due to a finite amount of training data; it persists in the population limit, even when both the number of pretraining tasks and the number of demonstrations per task tend to infinity.
These results reveal a fundamental difference between standard and adversarial training and suggest that, at least for linear transformers, adversarial training is necessary to obtain adversarially robust in-context learners.

\begin{figure}
  \centering
  \includegraphics[width=\textwidth]{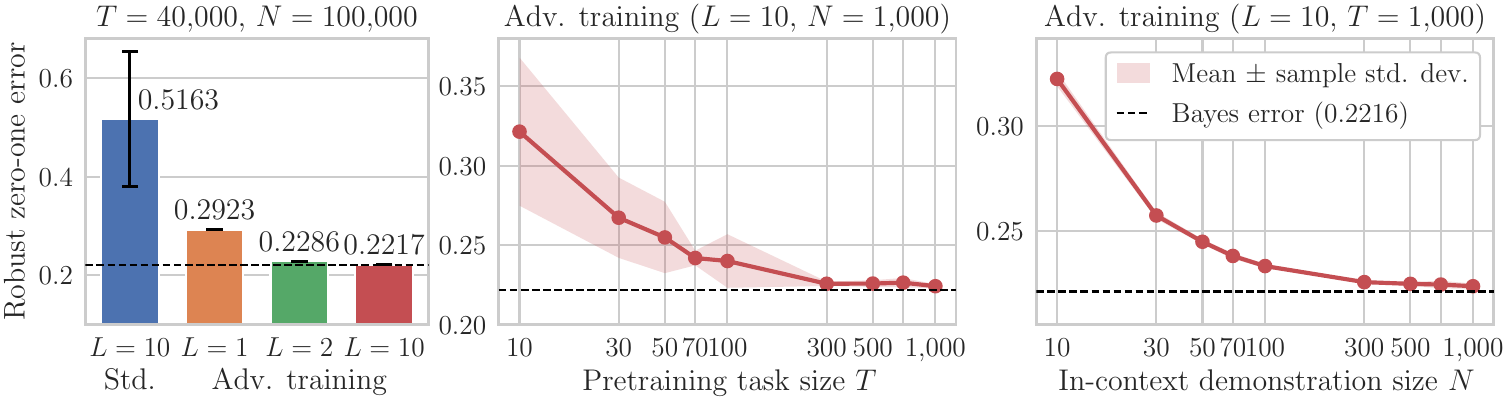}
  \caption{
    Robust zero-one errors for $d=100$, $r=10$, and $\epsilon=0.08$.
    Means and sample standard deviations are shown over three runs in the left panel and five runs in the other panels.
  }\label{fig:robust-zero-one-error}
\end{figure}

\subsection{Other Theoretical Results}

\textbf{Gradient flow to Bayes optimality.}
The preceding sections show that global minimizers of the adversarial empirical risk asymptotically achieve taskwise Bayes optimality.
However, practical optimization algorithms such as gradient descent are not guaranteed to find a global minimizer.
To bridge this gap, \cref{thm:main-last-layer-empirical-subgradient-flow} analyzes the case $L = 2$ and $r \leq 2$, in which only the final key--query matrix is trained.
We show that gradient flow yields predictors whose robust errors converge to the corresponding Bayes errors.

\textbf{Accuracy--robustness trade-off.}
Robustness often comes at the cost of clean accuracy.
In particular, adversarially trained models may be less accurate on unperturbed samples than their standardly trained counterparts~\citep{su2018robustness,tsipras2019robustness,ilyas2019adversarial}.
In \cref{thm:main-accuracy-robustness-trade-off}, we establish that the same trade-off arises in our setting.
Moreover, the trade-off persists even with infinitely many pretraining tasks and infinitely many demonstrations per task.
This result highlights a fundamental challenge in developing robust foundation models.

\textbf{Demonstration complexity.}
Achieving a target level of robust accuracy often requires more training samples than achieving the corresponding level of standard accuracy~\citep{schmidt2018adversarially,dan2020sharp}.
In \cref{thm:main-adversarial-demonstration-complexity}, we establish an analogous gap in the number of in-context demonstrations.
Attaining the same fixed excess error relative to the respective Bayes errors requires only one demonstration for standard prediction but $\Theta(\sqrt{d})$ demonstrations for robust prediction.
This result reveals another fundamental challenge in developing robust foundation models.

\subsection{Comparison with Prior Work}\label{sec:comparison}

Compared to \citet{kumano2026adversarially}, we substantially relax their assumptions on the data distribution and the model architecture while establishing stronger robustness guarantees for linear transformers.

\textbf{Data distribution.}
Their data model restricts feature scales to two discrete levels, such as $y \bm{x} = [\alpha, \ldots, \alpha, \beta, \ldots, \beta]$, where $\alpha, \beta > 0$ represent large and small values, respectively.
This is a major limitation they acknowledge.
Our framework allows arbitrary feature scales.
They also consider only the standard basis, which roughly corresponds to fixing $\bm{U}$ to a particular value in our framework.
In contrast, our theory applies to any orthonormal basis.
Although they allow a broader class of coefficient distributions than our Gaussian family, they impose strong constraints on the covariance structure, whereas we do not.
Overall, our framework broadens the class of admissible data models.

\textbf{Architecture.}
Theoretically, their analysis is limited to single-layer linear transformers, whereas ours applies to multilayer linear transformers.
Their experiments are also limited to single-layer linear transformers trained on relatively small datasets, including MNIST~\citep{MNIST}, Fashion-MNIST~\citep{FMNIST}, and CIFAR-10~\citep{CIFAR}.
By contrast, we conduct experiments with multilayer transformers equipped with softmax attention on larger datasets, including CIFAR-100~\citep{CIFAR}, Tiny ImageNet~\citep{TinyImageNet}, and Caltech-256~\citep{Caltech}, thereby evaluating the effectiveness of robust foundation models in more realistic settings.

\textbf{Robustness guarantees.}
Although differences in assumptions make direct comparisons difficult, our results provide more direct and comprehensive robustness guarantees.
First, they establish only bounds on the expected robust margin $\mathbb{E}[yf(\bm{x}+\bm{\Delta})]$.
Because the expected robust margin does not generally determine robust accuracy, these bounds alone do not provide a direct guarantee of robustness.
By contrast, we prove that the robust zero-one error converges to the taskwise robust Bayes error, thereby directly guaranteeing transformer robustness.
Second, their analysis does not determine whether linear transformers are robust or vulnerable in certain data regimes.
As they acknowledge, their single-layer restriction limits model capacity and prevents their analysis from capturing the full robustness potential of linear transformers.
By considering a multilayer model, we show that linear transformers asymptotically achieve taskwise robust Bayes optimality when they are sufficiently deep relative to the data complexity~(i.e., $L \geq r$).
Overall, our analysis provides more definitive guarantees and more fully characterizes the robustness of linear transformers.

\section{Experimental Results}

Detailed experimental settings and additional results are provided in \cref{sec:additional-experimental-results}.

\textbf{Synthetic datasets.}
We first empirically validate our theoretical results using synthetic datasets generated according to \cref{eq:task-structure}.
We set $d=100$, $r=10$, and $\epsilon=0.08$.
We train the linear transformer in \cref{eq:multilayer-scalar-predictor} by minimizing the logistic loss with Adam.
Each model is evaluated on $40{,}000$ test tasks with $200$ queries per task.
The left panel of \cref{fig:robust-zero-one-error} compares robust zero-one errors for $T=40{,}000$ and $N=100{,}000$.
Under adversarial training, the robust zero-one error decreases with depth, and the depth-$10$ model essentially attains the robust Bayes error, supporting \cref{thm:main-adversarial-training-bayes-optimality,thm:main-robust-bayes-representability}.
The near-optimal performance of the depth-$2$ model does not contradict the theorem because $L\geq r$ is a sufficient but not necessary condition.
In contrast, despite having sufficient depth and being trained on abundant data, the depth-$10$ model obtained by standard training remains far from optimal, consistent with \cref{thm:main-standard-training-bayes-optimality}.
The center and right panels show that the depth-$10$ adversarially trained model approaches the robust Bayes error as $T$ and $N$ increase, respectively, supporting \cref{thm:main-adversarial-training-bayes-optimality}.

\textbf{Real-world datasets.}
We then evaluate six-layer transformers with softmax attention on binary classification tasks from CIFAR-100, Tiny ImageNet, and Caltech-256.
The pretraining, validation, and test classes are mutually disjoint.
Clean and robust test accuracies are shown in \cref{tab:real-world-clean-robust-accuracy}.
Adversarial training substantially improves robust accuracy, albeit at the cost of lower clean accuracy.

\begin{table}
  \centering
  \caption{
    Clean/robust accuracy~(\%) of six-layer transformers with softmax attention on real-world data, reported as mean $\pm$ sample standard deviation across five independent runs.
  }\label{tab:real-world-clean-robust-accuracy}
  \begin{tabular}{llll}
    \toprule
    Training    & CIFAR-100                                & Tiny ImageNet                            & Caltech-256 \\
    \midrule
    Standard    & \textbf{79.1 $\pm$ 2.0} / 26.8 $\pm$ 3.4 & \textbf{77.4 $\pm$ 3.4} / 27.6 $\pm$ 7.6 & \textbf{78.0 $\pm$ 3.6} / 23.8 $\pm$ 6.5 \\
    Adversarial & 77.9 $\pm$ 1.7 / \textbf{63.2 $\pm$ 1.8} & 67.0 $\pm$ 6.0 / \textbf{52.8 $\pm$ 2.2} & 74.6 $\pm$ 3.9 / \textbf{62.4 $\pm$ 4.6} \\
    \bottomrule
  \end{tabular}
\end{table}

\section{Conclusion and Limitations}

We showed that adversarial pretraining can enable a sufficiently deep linear transformer to adapt to unseen Gaussian-mixture classification tasks through in-context learning from clean demonstrations and asymptotically attain the taskwise robust Bayes error.
By contrast, standard pretraining does not, even in the population limit.
We further established convergence under gradient flow and identified two challenges in robust in-context learning.
These results provide theoretical evidence that robustness acquired once during pretraining can transfer to downstream tasks without additional adversarial training, supporting the feasibility of robust foundation models.

Our analysis is restricted to Gaussian-mixture tasks that share a common latent mean and covariance, so the unseen tasks considered in our theory still belong to a structured task family.
Although such cross-task sharing is common in theoretical analyses of in-context learning~\citep{ahn2023transformers,mahankali2023one,kumano2026adversarially}, extending our guarantees to settings with task-dependent means or covariances and to non-Gaussian inputs is an important direction for future theoretical work.
Moreover, the practical utility of robust foundation models remains unknown at the scale of modern foundation-model pretraining.
We hope that our theoretical results will motivate large-scale empirical studies by large organizations with the computational resources required to conduct them.

\clearpage

\subsection*{AI Use Statement}

In this work, we used generative AI tools to
formulate mathematical claims, 
provide critical ingredients for proving mathematical claims, 
assist in the writing of proofs, 
implement methods, 
and assist with translation. 
We have not used generative AI tools to
generate synthetic data sets, 
help develop theoretical models or conceptual frameworks, 
propose or refine hypotheses, 
design or provide feedback on research methodology or experiments, 
and interpret results. 
Dataset cleaning and reformatting, 
and support for qualitative or thematic data analysis 
are not applicable to this work.
Additionally, we used generative AI tools to
create or edit software code, 
draft parts of a research paper, 
summarize or analyze existing literature, 
brainstorm, 
source or search for information, 
edit a research paper to improve readability, 
identify relevant literature, 
and suggest a structure for a research paper. 
We did not use generative AI tools to
create or modify scientific figures or images, 
suggest experimental parameters, 
create artifacts, 
discover research topics or identify gaps, 
format references, 
and propose a title or keywords for a research paper. 
Formulating questions for surveys or interviews 
and transcribing recordings of research materials 
are not applicable to this work. 
We have reviewed all AI-assisted work.
We independently checked all proofs for mathematical correctness, tested all experimental code, and read and evaluated every cited work.
We take responsibility for the final content of this work, including text, claims, or artifacts produced with the aid of generative AI\@.

\subsection*{Ethics Statement}

This work is a theoretical study of adversarial robustness supported by experiments on synthetic and real-world benchmark datasets.
We use only publicly available datasets and do not collect new data or involve human participants.
Our experiments do not interact with deployed systems.

\subsection*{Reproducibility Statement}

We state all assumptions explicitly in the main text and provide complete proofs of all theoretical results in the appendix.
Detailed experimental settings are provided in \cref{sec:additional-experimental-results}.
The supplementary material contains the code to reproduce the reported experimental results.

\bibliography{
  bib/adversarial-attack,
  bib/adversarial-finetuning,
  bib/breaker,
  bib/datasets,
  bib/fast-adversarial-training,
  bib/ICL-adv-robustness,
  bib/ICL,
  bib/others
}
\bibliographystyle{iclr2027_conference}

\appendix

\clearpage

\setcounter{equation}{0}
\setcounter{figure}{0}
\setcounter{table}{0}

\renewcommand{\theequation}{A\arabic{equation}}
\renewcommand{\thefigure}{A\arabic{figure}}
\renewcommand{\thetable}{A\arabic{table}}

\startcontents[sections]
\printcontents[sections]{l}{1}{\setcounter{tocdepth}{2}}

\section{Additional Experimental Results}\label{sec:additional-experimental-results}

All experiments were conducted on a Linux server with an x86\_64 CPU architecture and four NVIDIA A100-SXM4-80GB GPUs.

\subsection{Synthetic Experiments}

\textbf{Experimental design.}
We choose an anisotropic Gaussian distribution for which the standard and robust Bayes directions differ, allowing the experiments to distinguish standard from robust direction recovery.
Specifically, we set $d=100$ and $\epsilon=0.08$ and, for every $i\in[10]$, define the mean and diagonal covariance entries by
\begin{equation}
  \mu_{i}:=\frac{i}{10}, \qquad \Sigma_{i,i}:=\mu_{i}10^{-1+2(i-1)/9}.
  \label{eq:experimental-distribution}
\end{equation}
For every $i\in[d]\setminus[10]$, we set $\mu_{i}:=0$ and $\Sigma_{i,i}:=10^{-6}$.
The active signal-to-variance ratios $\mu_{i}/\Sigma_{i,i}=10^{1-2(i-1)/9}$ therefore range logarithmically from $10$ to $10^{-1}$.
Because the ten active coordinates have distinct covariance eigenvalues, this construction has Krylov dimension $r=10$.
We independently draw each task basis $\bm{U}$ from the normalized Haar distribution on $\mathcal{O}(d)$ and then sample demonstrations and queries from \cref{eq:task-structure}.
Each fixed pretraining dataset contains $T$ task bases, $N$ clean demonstrations per task, and one clean query per task.
For evaluation, we use $40{,}000$ independently drawn test tasks, a fresh context with the same value of $N$ for each task, and $200$ independent queries per context.
We conduct three independent runs for $T=40{,}000$ and $N=100{,}000$ and five independent runs otherwise.
Each run uses separate training and evaluation random streams, and we report means and sample standard deviations across runs.

\textbf{Model and training.}
We use the one-head masked linear-attention architecture in \cref{eq:multilayer-linear-transformer} and restrict its parameters to a rotation-equivariant family that respects the symmetry of the task distribution.
In particular, every layer has the form
\begin{align}
  \bm{P}_{\ell} &:= p_{\ell,x}\bm{\Pi}_{x}+p_{\ell,y}\bm{\Pi}_{y}, & \bm{Q}_{\ell} &:= q_{\ell,x}\bm{\Pi}_{x}+q_{\ell,y}\bm{\Pi}_{y}, \\
  \bm{\Pi}_{x} &:=\operatorname{diag}(\bm{I}_{d},0), & \bm{\Pi}_{y} &:= \operatorname{diag}(\bm{0}_{d},1).
  \label{eq:experimental-projector-parameterization}
\end{align}
The four coefficients in each layer are trainable and initialized independently from $\mathcal{N}(0,0.01^{2})$.
We minimize the logistic loss $\mathcal{L}(z)=\log(1+\exp(-z))$, which satisfies \cref{asm:summary}\cref{asm:loss}, by full-batch Adam for at most $2{,}500$ updates on the fixed pretraining dataset.
The learning rate increases linearly to $0.01$ over the first $30$ updates and subsequently either remains constant or follows cosine annealing to $0.001$ without restarts.
For each schedule, we retain the checkpoint with the smallest training loss and report the better of the two retained checkpoints under the same criterion.

\textbf{Robust evaluation and Bayes baselines.}
Masked linear attention is linear in the query, so both adversarial training and robust evaluation can use the exact worst-case perturbation rather than an approximate attack such as the fast gradient sign method~\citep{FGSM} or projected gradient descent~(PGD)~\citep{PGD}.
The adversary perturbs only the query and leaves all in-context demonstrations fixed and clean.
Writing $\bm{w}(\bm{\Theta}_{L},\mathcal{D}_{N})$ for the context-dependent effective query direction, the exact perturbation for $\bm{w}(\bm{\Theta}_{L},\mathcal{D}_{N})\neq\bm{0}_{d}$ is
\begin{equation}
  \bm{\Delta}=-y\epsilon\sqrt{d}\frac{\bm{w}(\bm{\Theta}_{L},\mathcal{D}_{N})}{\|\bm{w}(\bm{\Theta}_{L},\mathcal{D}_{N})\|_{2}}.
  \label{eq:experimental-exact-attack}
\end{equation}
When the effective direction is zero, we instead set $\bm{\Delta}=\bm{0}_{d}$, and during adversarial training we recompute the perturbation from the current model at every update.
For comparison with the theoretical optimum, we obtain the robust Bayes direction $(\bm{\Sigma}+\tau\bm{I}_{d})^{-1}\bm{\mu}$ by numerically solving
\begin{equation}
  \tau\|{(\bm{\Sigma}+\tau\bm{I}_{d})}^{-1}\bm{\mu}\|_{2}=\epsilon\sqrt{d}
\end{equation}
and compute all Bayes risks analytically from the resulting Gaussian margins.

\textbf{Alignment with the robust Bayes direction.}
To examine whether adversarial training recovers the optimal classifier rather than merely its error, \cref{fig:robust-bayes-direction-cosine} compares the effective direction of linear transformers with the robust Bayes direction characterized in \cref{lem:robust-bayes-characterization}.
At $T=40{,}000$ and $N=100{,}000$, the cosine similarity increases with depth and approaches one at depth ten.
Thus, the depth-$10$ model recovers not only the robust Bayes error but also the unique optimal direction up to positive scaling.
Increasing the number of pretraining tasks similarly improves the cosine similarity, while the values throughout the demonstration sweep remain close to one.
Together with the robust errors in \cref{fig:robust-zero-one-error}, these results support the convergence to the robust Bayes predictor established in \cref{thm:main-adversarial-training-bayes-optimality}.

\begin{figure}
  \centering
  \includegraphics[width=\textwidth]{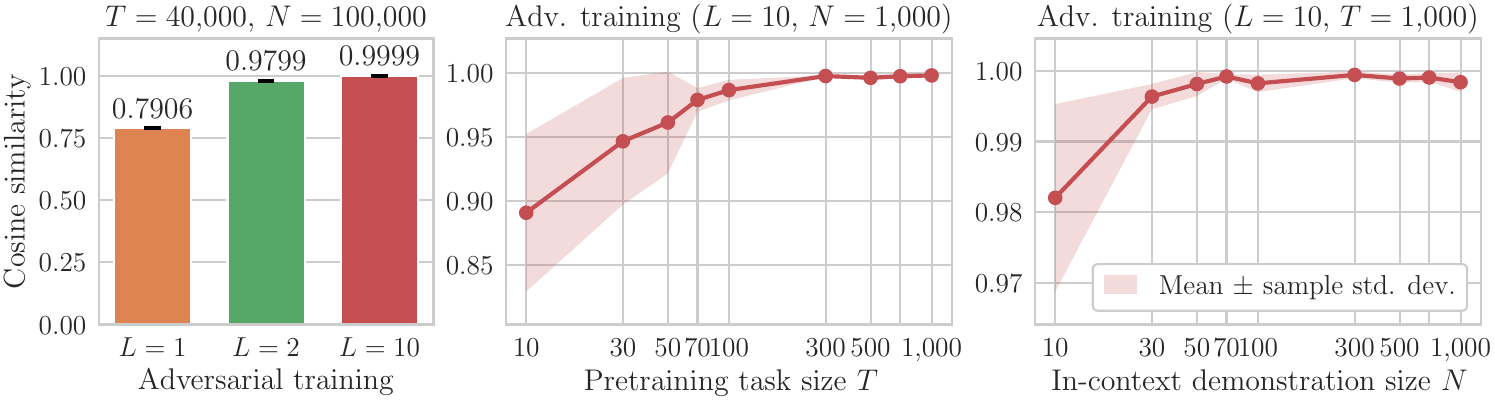}
  \caption{
    Cosine similarity between the population-context effective direction of an adversarially trained transformer and the robust Bayes direction.
    The distribution and the three experimental conditions are the same as in \cref{fig:robust-zero-one-error}.
    The left panel varies the depth at $T=40{,}000$ and $N=100{,}000$, the center panel varies $T$ with $L=10$ and $N=1{,}000$, and the right panel varies $N$ with $L=10$ and $T=1{,}000$.
    Means and sample standard deviations are shown over three runs in the left panel and five runs in the other panels.
  }\label{fig:robust-bayes-direction-cosine}
\end{figure}

\textbf{Accuracy--robustness trade-off and demonstration complexity.}
To assess the accuracy--robustness trade-off and the qualitative dependence on context length, \cref{tab:clean-robust-accuracy} compares clean and robust accuracy as the number of demonstrations increases.
Standard training gives higher clean accuracy overall, whereas adversarial training gives substantially higher robust accuracy, supporting the accuracy--robustness trade-off in \cref{thm:main-standard-training-bayes-optimality,thm:main-accuracy-robustness-trade-off}.
Moreover, standard clean accuracy saturates with relatively few demonstrations, while adversarial robust accuracy continues to improve toward the robust Bayes accuracy.
This behavior is qualitatively consistent with \cref{thm:main-adversarial-demonstration-complexity}.

\begin{table}
  \centering
  \caption{
    Clean and robust accuracies~(\%), reported as mean $\pm$ sample standard deviation across five runs, for depth-$10$ models trained on $T=1{,}000$ tasks.
    Each row uses the displayed number of demonstrations in both pretraining and evaluation.
    Robust accuracy uses $\epsilon=0.08$, and every model is evaluated on $40{,}000$ test tasks with $200$ queries per task.
  }\label{tab:clean-robust-accuracy}
  \begin{tabular}{ccccc}
    \toprule
              & \multicolumn{2}{c}{Standard training} & \multicolumn{2}{c}{Adversarial training} \\
    \cmidrule(lr){2-3}\cmidrule(lr){4-5}
    $N$       & Clean                                 & Robust          & Clean          & Robust \\
    \midrule
    $10$      & $84.6 \pm 4.2$                        & $63.5 \pm 4.5$  & $84.8 \pm 0.7$ & $67.8 \pm 0.3$ \\
    $50$      & $95.1 \pm 4.3$                        & $52.8 \pm 20.6$ & $92.4 \pm 0.3$ & $75.5 \pm 0.1$ \\
    $100$     & $98.0 \pm 0.6$                        & $45.1 \pm 15.9$ & $93.6 \pm 0.3$ & $76.6 \pm 0.1$ \\
    $500$     & $98.7 \pm 0.2$                        & $39.6 \pm 7.0$  & $94.1 \pm 0.3$ & $77.5 \pm 0.1$ \\
    $1{,}000$ & $98.4 \pm 0.6$                        & $45.1 \pm 13.0$ & $94.0 \pm 0.1$ & $77.6 \pm 0.1$ \\
    \bottomrule
  \end{tabular}
\end{table}

\subsection{Real-World Experiments}\label{sec:real-world-data-experiments}

\textbf{Datasets and class splits.}
We evaluate whether the benefit of adversarial pretraining extends beyond the theoretical Gaussian model by considering CIFAR-100~\citep{CIFAR}, Tiny ImageNet~\citep{TinyImageNet}, and Caltech-256~\citep{Caltech}.
In each of the five independent runs, the respective pretraining splits contain $80$, $180$, and $236$ classes, while each validation and test split contains ten classes, with all three partitions mutually disjoint.
We draw a different class partition in each run.

\textbf{PCA features.}
We represent each image by an unwhitened principal component analysis~(PCA) features to reduce the input dimension before constructing the in-context episodes.
Let $d_{\mathrm{PCA}}$ denote the number of retained components and hence the dimension of each PCA feature.
For every dataset and run, we scale RGB values to $[0,1]$ and flatten each image in channel--height--width order.
We fit the coordinatewise mean and PCA components using only images from that run's pretraining classes, excluding all images from validation and test classes.
We compute PCA with $32$ oversampling dimensions and seven power iterations.
After fitting PCA, we use the same mean and components to project every image in the training and query splits.
We retain $d_{\mathrm{PCA}}=256$ components for CIFAR-100 and $d_{\mathrm{PCA}}=512$ components for Tiny ImageNet and Caltech-256.
For Caltech-256, we exclude the clutter class and select $60$ training images and $20$ query images from each of the remaining $256$ classes.
Before computing its PCA feature, we convert each Caltech-256 image to RGB, center it on a square canvas padded with the RGB value $(128,128,128)$, and resize it to $64\times64$ pixels using bicubic interpolation with antialiasing.

\textbf{Binary in-context episodes.}
Each episode is a binary classification task whose two classes determine the meaning of the labels within that context.
At the beginning of every pretraining epoch, we randomly permute the pretraining classes and pair consecutive classes, ensuring that each class appears in exactly one task during that epoch.
For each class in a task, we sample $50$ demonstrations and $10$ queries without replacement from its training features, with the demonstration and query sets disjoint within the task.
Validation and testing instead enumerate all $45$ unordered pairs of the corresponding ten held-out classes.
For each pair, we form one fixed context from $50$ training features per class and evaluate it on $10$ query-split features per class, yielding $900$ queries in each final test.

\textbf{Transformer and optimization.}
We use a six-layer pre-normalization transformer with softmax attention, residual connections, and GELU feed-forward networks, thereby testing whether the qualitative prediction of the linear-attention theory persists in a standard nonlinear architecture.
Each PCA feature is augmented with one of three learned embeddings that identifies a demonstration with label $-1$, a demonstration with label $+1$, or an unlabeled query.
CIFAR-100 uses four attention heads and a feed-forward dimension of $1{,}024$, whereas Tiny ImageNet and Caltech-256 use eight attention heads and a feed-forward dimension of $2{,}048$; every attention head has dimension $64$.
A linear head applied after the final layer normalization produces the scalar score optimized with the logistic loss.
We use Xavier uniform initialization for linear weights, zero linear biases, independent $\mathcal{N}(0,0.02^{2})$ initialization for the three embeddings, and unit scales and zero biases for layer normalization.
We train for $300$ epochs with AdamW, exponential decay rates $0.9$ and $0.95$, numerical stability constant $10^{-8}$, weight decay $0.05$, and global gradient clipping at $1.0$.
The standard and adversarial runs use peak learning rates of $3\times10^{-4}$ and $10^{-3}$, respectively, with linear warmup over the first $5\%$ of updates followed by cosine decay to zero.
We initialize the standard model from scratch and the adversarial model from the corresponding standardly trained model.
Model selection minimizes the clean validation logistic loss for standard training and the $20$-step PGD validation loss for adversarial training.

\textbf{Adversarial training and evaluation.}
The real-world threat model follows the theoretical setting by perturbing only the query while keeping all demonstrations clean and fixed.
Specifically, the adversary may change a query PCA feature within an $L_{2}$ ball of radius $\epsilon\sqrt{d_{\mathrm{PCA}}}$, where $\epsilon=0.1$ in both training and evaluation.
Adversarial training uses untargeted $10$-step PGD with one uniformly sampled random start and step size $\epsilon\sqrt{d_{\mathrm{PCA}}}/4$.
For final evaluation, we measure robust accuracy using APGD-CE and binary FAB~\citep{AutoAttack}, each run for $100$ iterations.
We omit the remaining AutoAttack components because the DLR loss is undefined for scalar binary scores and Square Attack relies on spatial image structure.

\section{Robust Bayes Implementation by Linear Transformers: Proof of \texorpdfstring{\cref{thm:main-robust-bayes-representability}}{the robust Bayes representability theorem}}\label{sec:appendix-robust-bayes-implementation}

\subsection{Robust Bayes classifier}

In this section, we characterize the robust Bayes classifier and its error for a fixed task independently of the linear-transformer architecture.
For $\bm{U} = \bm{I}_{d}$ and thus $\bm{x} = y \bm{a}$, we derive the robust zero-one error of every linear function in \cref{lem:symmetric-gaussian-linear-robust-error}, the standard Bayes error in \cref{lem:symmetric-gaussian-bayes-classifier}, and the robust Bayes error and all linear functions attaining it in \cref{lem:symmetric-gaussian-robust-bayes-classifier}.
Then, for every $\bm{U} \in \mathcal{O}(d)$ and $\bm{x} = y \bm{U} \bm{a}$, we generalize these results in \cref{lem:robust-bayes-characterization}.

Variants or equivalent formulations of these results appear in \citet{hastie2009elements,bhagoji2019lower,dan2020sharp}.
These lemmas are auxiliary to our main results, and we do not claim them as novel contributions.
For completeness, we provide self-contained proofs, since adapting the existing results to our setting would require essentially as much exposition as proving them directly.

\begin{lemma}[Robust error of linear functions for the identity-basis task]\label{lem:symmetric-gaussian-linear-robust-error}
  Suppose that $y$ is uniform on $\{\pm 1\}$ and $\bm{x}\mid y\sim\mathcal{N}(y\bm{\mu},\bm{\Sigma})$, where $\bm{\Sigma}\succ\bm{0}_{d,d}$.
  For every $\bm{h}\in\mathbb{R}^{d}$, define $g_{\bm{h}}:\mathbb{R}^{d}\to\mathbb{R}$ by $g_{\bm{h}}(\bm{x}):=\bm{h}^{\top}\bm{x}$.
  Then, for every $\epsilon\geq0$ and $\bm{h}\in\mathbb{R}^{d}\setminus\{\bm{0}_{d}\}$,
  \begin{equation}
    \mathbb{P}\ab(\inf_{\|\bm{\Delta}\|_{2}\leq\epsilon\sqrt{d}}y g_{\bm{h}}(\bm{x}+\bm{\Delta})\leq0)=\Phi\ab(-\frac{\bm{\mu}^{\top}\bm{h}-\epsilon\sqrt{d}\|\bm{h}\|_{2}}{\sqrt{\bm{h}^{\top}\bm{\Sigma}\bm{h}}}),
    \label{eq:symmetric-gaussian-linear-robust-error}
  \end{equation}
  where $\Phi$ denotes the cumulative distribution function of the standard normal distribution.
\end{lemma}

\begin{proof}
  For every perturbation $\bm{\Delta}\in\mathbb{R}^{d}$, the signed margin is
  \begin{equation}
    y g_{\bm{h}}(\bm{x}+\bm{\Delta})=y\bm{h}^{\top}\bm{x}+y\bm{h}^{\top}\bm{\Delta}.
  \end{equation}
  Cauchy--Schwarz shows that its minimum over the perturbation ball is
  \begin{equation}
    \inf_{\|\bm{\Delta}\|_{2}\leq\epsilon\sqrt{d}}y g_{\bm{h}}(\bm{x}+\bm{\Delta})=y\bm{h}^{\top}\bm{x}-\epsilon\sqrt{d}\|\bm{h}\|_{2}.
  \end{equation}
  The class-conditional model implies
  \begin{equation}
    y\bm{h}^{\top}\bm{x}\sim\mathcal{N}(\bm{\mu}^{\top}\bm{h},\bm{h}^{\top}\bm{\Sigma}\bm{h}).
  \end{equation}
  Since $\bm{h}\neq\bm{0}_{d}$ and $\bm{\Sigma}\succ\bm{0}_{d,d}$, the variance $\bm{h}^{\top}\bm{\Sigma}\bm{h}$ is strictly positive.
  Therefore, standardizing the Gaussian signed margin gives
  \begin{align}
    &\mathbb{P}\ab(\inf_{\|\bm{\Delta}\|_{2}\leq\epsilon\sqrt{d}}y g_{\bm{h}}(\bm{x}+\bm{\Delta})\leq0) \\
    &=\mathbb{P}(y\bm{h}^{\top}\bm{x}\leq\epsilon\sqrt{d}\|\bm{h}\|_{2}) \\
    &= \mathbb{P}_{z\sim\mathcal{N}(0,1)}\ab(z\leq\frac{\epsilon\sqrt{d}\|\bm{h}\|_{2}-\bm{\mu}^{\top}\bm{h}}{\sqrt{\bm{h}^{\top}\bm{\Sigma}\bm{h}}}) \\
    &= \Phi\ab(-\frac{\bm{\mu}^{\top}\bm{h}-\epsilon\sqrt{d}\|\bm{h}\|_{2}}{\sqrt{\bm{h}^{\top}\bm{\Sigma}\bm{h}}}).
  \end{align}
\end{proof}

\begin{lemma}[Bayes error for the identity-basis task]\label{lem:symmetric-gaussian-bayes-classifier}
  Suppose that $y$ is uniform on $\{\pm 1\}$ and $\bm{x} \mid y \sim \mathcal{N}(y\bm{\mu}, \bm{\Sigma})$, where $\bm{\mu} \neq \bm{0}_{d}$ and $\bm{\Sigma} \succ \bm{0}_{d,d}$.
  Then, for every measurable function $\varphi: \mathbb{R}^{d} \to \mathbb{R}$,
  \begin{equation}
    \mathbb{P}(y \varphi(\bm{x}) \leq 0) \geq \Phi\ab(-\sqrt{\bm{\mu}^{\top} \bm{\Sigma}^{-1} \bm{\mu}}),
  \end{equation}
  where $\Phi$ denotes the cumulative distribution function of the standard normal distribution.
\end{lemma}

\begin{proof}
  \textbf{Lower bound.}
  Let $p_{+}$ and $p_{-}$ denote the conditional densities of $\bm{x}$ given $y=+1$ and $y=-1$, respectively.
  Since the two labels have equal probability, conditioning on $y$ gives
  \begin{equation}
    \mathbb{P}(y\varphi(\bm{x})\leq0)=\frac{1}{2}\int_{\bm{x}\in\mathbb{R}^{d}}[\mathds{1}\{\varphi(\bm{x})\leq0\}p_{+}(\bm{x})+\mathds{1}\{\varphi(\bm{x})\geq0\}p_{-}(\bm{x})]d\bm{x}.
  \end{equation}
  For every $\bm{x}\in\mathbb{R}^{d}$, the integrand satisfies
  \begin{equation}
    \mathds{1}\{\varphi(\bm{x})\leq0\}p_{+}(\bm{x})+\mathds{1}\{\varphi(\bm{x})\geq0\}p_{-}(\bm{x})\geq\min\{p_{+}(\bm{x}),p_{-}(\bm{x})\}.
  \end{equation}
  Consequently,
  \begin{equation}
    \mathbb{P}(y\varphi(\bm{x})\leq0)\geq\frac{1}{2}\int_{\bm{x}\in\mathbb{R}^{d}}\min\{p_{+}(\bm{x}),p_{-}(\bm{x})\}d\bm{x}.
  \end{equation}

  \textbf{Function that attains the lower bound.}
  The Gaussian density formula gives
  \begin{equation}
    \log\ab(\frac{p_{+}(\bm{x})}{p_{-}(\bm{x})})=-\frac{1}{2}(\bm{x}-\bm{\mu})^{\top}\bm{\Sigma}^{-1}(\bm{x}-\bm{\mu})+\frac{1}{2}(\bm{x}+\bm{\mu})^{\top}\bm{\Sigma}^{-1}(\bm{x}+\bm{\mu})=2\bm{\mu}^{\top}\bm{\Sigma}^{-1}\bm{x}.
  \end{equation}
  Therefore, the linear function $\varphi^{\star}(\bm{x}):=\bm{\mu}^{\top}\bm{\Sigma}^{-1}\bm{x}$ is positive exactly when $p_{+}(\bm{x})>p_{-}(\bm{x})$ and negative exactly when $p_{+}(\bm{x})<p_{-}(\bm{x})$.
  Because $\bm{\Sigma}^{-1}\bm{\mu}\neq\bm{0}_{d}$, the hyperplane on which $\varphi^{\star}(\bm{x})=0$ has Lebesgue measure zero.
  Thus,
  \begin{align}
    & \quad \frac{1}{2}\int_{\bm{x}\in\mathbb{R}^{d}}\min\{p_{+}(\bm{x}),p_{-}(\bm{x})\}d\bm{x} \\
    &=\frac{1}{2}\int_{\bm{x}\in\mathbb{R}^{d}}[\mathds{1}\{\varphi^{\star}(\bm{x})\leq0\}p_{+}(\bm{x})+\mathds{1}\{\varphi^{\star}(\bm{x})\geq0\}p_{-}(\bm{x})]d\bm{x} \\
    &=\mathbb{P}(y\varphi^{\star}(\bm{x})\leq0).
  \end{align}

  \textbf{Value of the lower bound.}
  The class-conditional model and the definition of $\varphi^{\star}$ imply
  \begin{equation}
    y\varphi^{\star}(\bm{x})\sim\mathcal{N}(\bm{\mu}^{\top}\bm{\Sigma}^{-1}\bm{\mu},\bm{\mu}^{\top}\bm{\Sigma}^{-1}\bm{\mu}).
  \end{equation}
  The mean and variance are both equal to $\bm{\mu}^{\top}\bm{\Sigma}^{-1}\bm{\mu}$, which is strictly positive because $\bm{\mu}\neq\bm{0}_{d}$ and $\bm{\Sigma}\succ\bm{0}_{d,d}$.
  Standardizing therefore yields
  \begin{equation}
    \mathbb{P}(y\varphi^{\star}(\bm{x})\leq0)=\Phi\ab(-\frac{\bm{\mu}^{\top}\bm{\Sigma}^{-1}\bm{\mu}}{\sqrt{\bm{\mu}^{\top}\bm{\Sigma}^{-1}\bm{\mu}}})=\Phi\ab(-\sqrt{\bm{\mu}^{\top}\bm{\Sigma}^{-1}\bm{\mu}}).
  \end{equation}
\end{proof}

\begin{lemma}[Robust Bayes characterization for the identity-basis task]\label{lem:symmetric-gaussian-robust-bayes-classifier}
  Suppose that $y$ is uniform on $\{\pm 1\}$ and $\bm{x} \mid y \sim \mathcal{N}(y\bm{\mu}, \bm{\Sigma})$, where $\bm{\mu} \neq \bm{0}_{d}$ and $\bm{\Sigma} \succ \bm{0}_{d,d}$.
  Suppose also that $0\leq\epsilon\sqrt{d}<\|\bm{\mu}\|_{2}$.
  There exists a unique $\tau\geq0$ satisfying
  \begin{equation}
    \tau\|{(\bm{\Sigma}+\tau\bm{I}_{d})}^{-1}\bm{\mu}\|_{2}=\epsilon\sqrt{d}.
    \label{eq:robust-shift-equation}
  \end{equation}
  For this $\tau$, define
  \begin{equation}
    \bm{h}_{\tau}:={(\bm{\Sigma}+\tau\bm{I}_{d})}^{-1}\bm{\mu}.
    \label{eq:robust-bayes-direction}
  \end{equation}
  Then, for every measurable function $\varphi:\mathbb{R}^{d}\to\mathbb{R}$,
  \begin{equation}
    \mathbb{P}\ab(\inf_{\|\bm{\Delta}\|_{2}\leq\epsilon\sqrt{d}}y\varphi(\bm{x}+\bm{\Delta})\leq0)\geq\Phi\ab(-\sqrt{\bm{h}_{\tau}^{\top}\bm{\Sigma}\bm{h}_{\tau}}).
  \end{equation}
  Among linear functions, equality is attained exactly by the functions $\bm{x}\mapsto\omega\bm{h}_{\tau}^{\top}\bm{x}$ with $\omega>0$.
\end{lemma}

\begin{proof}
  \textbf{Existence and uniqueness of $\tau$.}
  Let $\lambda_{1},\ldots,\lambda_{d}>0$ be the eigenvalues of $\bm{\Sigma}$, let $\bm{u}_{1},\ldots,\bm{u}_{d}$ be corresponding orthonormal eigenvectors, and write $\bm{\mu}=\sum_{i=1}^{d}m_{i}\bm{u}_{i}$.
  For $\tau\geq0$, the spectral theorem gives
  \begin{equation}
    q(\tau):=\tau^{2}\|{(\bm{\Sigma}+\tau\bm{I}_{d})}^{-1}\bm{\mu}\|_{2}^{2}=\sum_{i=1}^{d}m_{i}^{2}\ab(\frac{\tau}{\lambda_{i}+\tau})^{2}.
  \end{equation}
  Since both sides of the equation defining $\tau$ are nonnegative, that equation is equivalent to $q(\tau)=\epsilon^{2}d$.
  Each map $\tau\mapsto\tau/(\lambda_{i}+\tau)$ is continuous and strictly increasing from zero to one on $[0,\infty)$.
  Because $\bm{\mu}\neq\bm{0}_{d}$, at least one $m_{i}$ is nonzero, and hence $q$ is continuous and strictly increasing with
  \begin{equation}
    q(0)=0, \qquad \lim_{\tau\to\infty}q(\tau)=\sum_{i=1}^{d}m_{i}^{2}=\|\bm{\mu}\|_{2}^{2}.
  \end{equation}
  If $\epsilon=0$, the equation $q(\tau)=\epsilon^{2}d$ has the unique solution $\tau=0$.
  If $\epsilon>0$, the inequality $0<\epsilon^{2}d<\|\bm{\mu}\|_{2}^{2}$ and the intermediate value theorem give a unique solution $\tau>0$.

  \textbf{Lower bound.}
  For the resulting $\bm{h}_{\tau}$, define the label-dependent perturbation $\bm{\Delta}_{y}:=-y\tau\bm{h}_{\tau}$.
  This perturbation is admissible because
  \begin{equation}
    \|\bm{\Delta}_{y}\|_{2}=\tau\|\bm{h}_{\tau}\|_{2}=\epsilon\sqrt{d}.
  \end{equation}
  Therefore, for every measurable function $\varphi$,
  \begin{equation}
    \mathbb{P}\ab(\inf_{\|\bm{\Delta}\|_{2}\leq\epsilon\sqrt{d}}y\varphi(\bm{x}+\bm{\Delta})\leq0)\geq\mathbb{P}(y\varphi(\bm{x}+\bm{\Delta}_{y})\leq0).
  \end{equation}
  The definition of $\bm{h}_{\tau}$ implies
  \begin{equation}
    \bm{\mu}-\tau\bm{h}_{\tau}=\bm{\Sigma}\bm{h}_{\tau}.
  \end{equation}
  Consequently, the perturbed input has the class-conditional distribution
  \begin{equation}
    \bm{x}+\bm{\Delta}_{y}\mid y\sim\mathcal{N}(y\bm{\Sigma}\bm{h}_{\tau},\bm{\Sigma}).
  \end{equation}
  Since $\bm{\Sigma}\bm{h}_{\tau}\neq\bm{0}_{d}$, applying \cref{lem:symmetric-gaussian-bayes-classifier} to this symmetric Gaussian mixture gives
  \begin{equation}
    \mathbb{P}(y\varphi(\bm{x}+\bm{\Delta}_{y})\leq0)\geq\Phi\ab(-\sqrt{{(\bm{\Sigma}\bm{h}_{\tau})}^{\top}\bm{\Sigma}^{-1}(\bm{\Sigma}\bm{h}_{\tau})})=\Phi\ab(-\sqrt{\bm{h}_{\tau}^{\top}\bm{\Sigma}\bm{h}_{\tau}}).
  \end{equation}

  \textbf{Attainment of the lower bound.}
  Applying \cref{eq:symmetric-gaussian-linear-robust-error} with $\bm{h}=\bm{h}_{\tau}$ to the linear function $\varphi^{\star}(\bm{x})=\bm{h}_{\tau}^{\top}\bm{x}$ gives
  \begin{equation}
    \mathbb{P}\ab(\inf_{\|\bm{\Delta}\|_{2}\leq\epsilon\sqrt{d}}y\varphi^{\star}(\bm{x}+\bm{\Delta})\leq0)=\Phi\ab(-\frac{\bm{\mu}^{\top}\bm{h}_{\tau}-\epsilon\sqrt{d}\|\bm{h}_{\tau}\|_{2}}{\sqrt{\bm{h}_{\tau}^{\top}\bm{\Sigma}\bm{h}_{\tau}}}).
  \end{equation}
  The definitions of $\tau$ and $\bm{h}_{\tau}$ give
  \begin{equation}
    \bm{\mu}^{\top}\bm{h}_{\tau}-\epsilon\sqrt{d}\|\bm{h}_{\tau}\|_{2}=\bm{h}_{\tau}^{\top}(\bm{\Sigma}+\tau\bm{I}_{d})\bm{h}_{\tau}-\tau\|\bm{h}_{\tau}\|_{2}^{2}=\bm{h}_{\tau}^{\top}\bm{\Sigma}\bm{h}_{\tau}.
  \end{equation}
  Substitution into the preceding robust-error formula yields
  \begin{equation}
    \mathbb{P}\ab(\inf_{\|\bm{\Delta}\|_{2}\leq\epsilon\sqrt{d}}y\varphi^{\star}(\bm{x}+\bm{\Delta})\leq0)=\Phi\ab(-\frac{\bm{h}_{\tau}^{\top}\bm{\Sigma}\bm{h}_{\tau}}{\sqrt{\bm{h}_{\tau}^{\top}\bm{\Sigma}\bm{h}_{\tau}}})=\Phi\ab(-\sqrt{\bm{h}_{\tau}^{\top}\bm{\Sigma}\bm{h}_{\tau}}).
  \end{equation}
  Thus, $\varphi^{\star}$ attains the lower bound.

  \textbf{Uniqueness among linear functions.}
  Positive scale invariance in \cref{eq:symmetric-gaussian-linear-robust-error} shows that every positive multiple of $\bm{h}_{\tau}$ attains the same lower bound.
  Fix $\bm{v}\in\mathbb{R}^{d}\setminus\{\bm{0}_{d}\}$.
  The identities $\bm{\mu}=(\bm{\Sigma}+\tau\bm{I}_{d})\bm{h}_{\tau}$ and $\epsilon\sqrt{d}=\tau\|\bm{h}_{\tau}\|_{2}$, followed by the Euclidean Cauchy--Schwarz inequality and the Cauchy--Schwarz inequality for the inner product induced by $\bm{\Sigma}$, give
  \begin{align}
    \bm{\mu}^{\top}\bm{v}-\epsilon\sqrt{d}\|\bm{v}\|_{2}
    &=\bm{h}_{\tau}^{\top}\bm{\Sigma}\bm{v}+\tau(\bm{h}_{\tau}^{\top}\bm{v}-\|\bm{h}_{\tau}\|_{2}\|\bm{v}\|_{2}) \\
    &\leq\bm{h}_{\tau}^{\top}\bm{\Sigma}\bm{v} \\
    &\leq\sqrt{\bm{h}_{\tau}^{\top}\bm{\Sigma}\bm{h}_{\tau}}\sqrt{\bm{v}^{\top}\bm{\Sigma}\bm{v}}.
  \end{align}
  The denominator $\sqrt{\bm{v}^{\top}\bm{\Sigma}\bm{v}}$ is strictly positive because $\bm{\Sigma}\succ\bm{0}_{d,d}$.
  Dividing by this denominator and using the strict monotonicity of $\Phi$ in \cref{eq:symmetric-gaussian-linear-robust-error} show that the robust error of $g_{\bm{v}}$ is at least the lower bound.
  The robust error equals the lower bound only if both inequalities in the preceding display hold with equality.
  If $\tau>0$, equality in the first inequality requires $\bm{h}_{\tau}^{\top}\bm{v}=\|\bm{h}_{\tau}\|_{2}\|\bm{v}\|_{2}$, which holds exactly when $\bm{v}=\omega\bm{h}_{\tau}$ for some $\omega>0$.
  If $\tau=0$, the first inequality is automatically an equality, while equality in the $\bm{\Sigma}$-inner-product Cauchy--Schwarz inequality again holds exactly when $\bm{v}=\omega\bm{h}_{\tau}$ for some $\omega>0$.
  Conversely, every such positive multiple attains equality, as observed above.
  Finally, $g_{\bm{0}_{d}}$ is identically zero and therefore has robust error one, whereas the lower bound is strictly less than one half because $\bm{h}_{\tau}^{\top}\bm{\Sigma}\bm{h}_{\tau}>0$.
\end{proof}

\begin{lemma}[Taskwise robust Bayes characterization]\label{lem:robust-bayes-characterization}
  Suppose that \cref{asm:mean-covariance,asm:perturbation-radius} in \cref{asm:summary} hold.
  For every $\bm{U}\in\mathcal{O}(d)$ and $\bm{h}\in\mathbb{R}^{d}$, define the linear function $g_{\bm{U},\bm{h}}:\mathbb{R}^{d}\to\mathbb{R}$ by
  \begin{equation*}
    g_{\bm{U},\bm{h}}(\bm{x}):={\bm{h}}^{\top}\bm{U}^{\top}\bm{x}.
  \end{equation*}
  \begin{enumerate}[label = {(\alph*)}, leftmargin = *]
    \item
      For every $\bm{U}\in\mathcal{O}(d)$ and $\bm{h}\in\mathbb{R}^{d}\setminus\{\bm{0}_{d}\}$,
      \begin{equation}
        \mathcal{E}_{\epsilon,\bm{U}}(g_{\bm{U},\bm{h}})=\Phi\ab(-\frac{\bm{\mu}^{\top}\bm{h}-\epsilon\sqrt{d}\|\bm{h}\|_{2}}{\sqrt{\bm{h}^{\top}\bm{\Sigma}\bm{h}}}),
        \label{eq:linear-robust-error}
      \end{equation}

    \item
      For every $\bm{U}\in\mathcal{O}(d)$,
      \begin{equation}
        \mathcal{E}_{\epsilon,\bm{U}}(g_{\bm{U},\bm{h}_{\tau}})=\Phi\ab(-\sqrt{\bm{h}_{\tau}^{\top}\bm{\Sigma}\bm{h}_{\tau}})=\mathcal{E}_{\epsilon,\bm{U}}^{\star}.
        \label{eq:taskwise-robust-bayes-error}
      \end{equation}
      Among linear functions, the robust Bayes error is attained exactly by the functions $\bm{x}\mapsto\omega\bm{h}_{\tau}^{\top}\bm{U}^{\top}\bm{x}$ with $\omega>0$.
  \end{enumerate}
\end{lemma}

\begin{proof}
  Define $\bm{z}:=\bm{U}^{\top}\bm{x}$ and $\widetilde{\bm{\Delta}}:=\bm{U}^{\top}\bm{\Delta}$.
  Conditional on $\bm{U}$, the transformed input satisfies
  \begin{equation}
    \bm{z}\mid y\sim\mathcal{N}(y\bm{\mu},\bm{\Sigma}).
  \end{equation}
  Orthogonality of $\bm{U}$ gives
  \begin{equation}
    \|\widetilde{\bm{\Delta}}\|_{2}=\|\bm{\Delta}\|_{2}, \qquad g_{\bm{U},\bm{h}}(\bm{x}+\bm{\Delta})=g_{\bm{h}}(\bm{z}+\widetilde{\bm{\Delta}}).
  \end{equation}
  The map $\bm{\Delta}\mapsto\widetilde{\bm{\Delta}}$ is a bijection of the perturbation ball onto itself.

  \textbf{Part (a).}
  Applying \cref{eq:symmetric-gaussian-linear-robust-error} to $g_{\bm{h}}$ in the transformed coordinates proves \cref{eq:linear-robust-error}.

  \textbf{Part (b).}
  For every measurable function $\varphi:\mathbb{R}^{d}\to\mathbb{R}$, define $\widetilde{\varphi}:\mathbb{R}^{d}\to\mathbb{R}$ by $\widetilde{\varphi}(\bm{z}):=\varphi(\bm{U}\bm{z})$.
  The map $\varphi\mapsto\widetilde{\varphi}$ is a bijection on the class of measurable functions, and the preceding change of coordinates gives
  \begin{equation}
    \mathcal{E}_{\epsilon,\bm{U}}(\varphi)=\mathbb{P}\ab(\inf_{\|\widetilde{\bm{\Delta}}\|_{2}\leq\epsilon\sqrt{d}}y\widetilde{\varphi}(\bm{z}+\widetilde{\bm{\Delta}})\leq0).
  \end{equation}
  Applying \cref{lem:symmetric-gaussian-robust-bayes-classifier} in the transformed coordinates proves \cref{eq:taskwise-robust-bayes-error}.
  Its uniqueness conclusion for linear functions proves the final statement of part (b).
\end{proof}

\subsection{Query-Linear Representation}\label{sec:appendix-query-linear-foundations}\label{sec:appendix-common-foundations}

Define the population second moment by
\begin{equation}
  \bm{S}:=\bm{\Sigma}+\bm{\mu}\bm{\mu}^{\top}.
  \label{eq:population-second-moment}
\end{equation}
The latent finite-context matrix and its population limit are
\begin{equation}
  \bm{G}_{a,N}:=
  \begin{bmatrix}N^{-1}\sum_{n=1}^{N}\bm{a}_{n}\bm{a}_{n}^{\top} & N^{-1}\sum_{n=1}^{N}\bm{a}_{n} \\ N^{-1}\sum_{n=1}^{N}\bm{a}_{n}^{\top} & 1
  \end{bmatrix}, \qquad \bm{G}_{\star}:=
  \begin{bmatrix}\bm{S} & \bm{\mu} \\ \bm{\mu}^{\top} & 1
  \end{bmatrix}.
  \label{eq:latent-finite-context-matrix}
\end{equation}
The strong law gives $\bm{G}_{a,N}\to\bm{G}_{\star}$ almost surely.
The corresponding ambient context matrix is
\begin{equation}
  \bm{G}_{N}(\bm{U}):=\frac{1}{N}\bm{Z}^{(0)}\bm{M}{\bm{Z}^{(0)}}^{\top}=\operatorname{diag}(\bm{U},1)\bm{G}_{a,N}\operatorname{diag}(\bm{U}^{\top},1).
  \label{eq:context-blocks}
\end{equation}

For a context matrix $\bm{G}$, define $\bm{A}_{\bm{\Theta}_{L}}^{(0)}(\bm{G}):=\bm{I}_{d+1}$ and, for $\ell\in[L]$, recursively define
\begin{equation}
  \bm{A}_{\bm{\Theta}_{L}}^{(\ell)}(\bm{G}):=\{\bm{I}_{d+1}+\bm{P}_{\ell}\bm{A}_{\bm{\Theta}_{L}}^{(\ell-1)}(\bm{G})\bm{G}{\bm{A}_{\bm{\Theta}_{L}}^{(\ell-1)}(\bm{G})}^{\top}\bm{Q}_{\ell}\}\bm{A}_{\bm{\Theta}_{L}}^{(\ell-1)}(\bm{G}).
  \label{eq:fixed-depth-context-recursion}
\end{equation}
For any context matrix $\bm{G}$, define the depth-$L$ effective direction by
\begin{equation}
  \bm{\theta}_{\bm{\Theta}_{L}}(\bm{G}):=
  \begin{bmatrix}\bm{I}_{d} & \bm{0}_{d}
  \end{bmatrix}{\bm{A}_{\bm{\Theta}_{L}}^{(L)}(\bm{G})}^{\top}\bm{e}_{d+1}.
  \label{eq:context-fixed-depth-effective-direction}
\end{equation}
Here, $\bm{e}_{d+1}$ is the last standard basis vector of $\mathbb{R}^{d+1}$.
The mask in \cref{eq:multilayer-linear-transformer} removes the query column from every context matrix.
Induction over the layers therefore shows that every depth-$L$ transformer is linear in its query:
\begin{equation}
  f(\bm{x}+\bm{\Delta}; \bm{\Theta}_{L}, \mathcal{D}_{N})={\bm{\theta}_{\bm{\Theta}_{L}}(\bm{G}_{N}(\bm{U}))}^{\top}(\bm{x}+\bm{\Delta}).
  \label{eq:finite-context-query-linear-representation}
\end{equation}

\subsection{Krylov structure and sufficient-depth realization}\label{sec:appendix-active-spectrum-construction}

Define the feature and label projectors by
\begin{equation}
  \bm{\Pi}_{x}:=
  \begin{bmatrix}\bm{I}_{d} & \bm{0}_{d} \\ \bm{0}_{d}^{\top} & 0
  \end{bmatrix}, \qquad \bm{\Pi}_{y}:=
  \begin{bmatrix}\bm{0}_{d,d} & \bm{0}_{d} \\ \bm{0}_{d}^{\top} & 1
  \end{bmatrix}.
\end{equation}

We first relate the active spectra of $\bm{\Sigma}$ and $\bm{S}$.

\begin{lemma}[Active spectrum under the rank-one update]\label{lem:active-spectrum-count}
  The matrix $\bm{S}$ has exactly $r$ distinct active eigenvalues, where an eigenvalue $s$ is active when its orthogonal spectral projector $\bm{F}_{s}$ satisfies $\bm{F}_{s}\bm{\mu}\neq\bm{0}_{d}$ and $r$ is defined in \cref{eq:krylov-dim}.
\end{lemma}

\begin{proof}
  Since $\bm{S}\bm{v}=\bm{\Sigma}\bm{v}+\bm{\mu}\bm{\mu}^{\top}\bm{v}$, induction in both directions gives
  \begin{equation}
    \operatorname{span}\{\bm{\mu},\bm{S}\bm{\mu},\bm{S}^{2}\bm{\mu},\ldots\}=\operatorname{span}\{\bm{\mu},\bm{\Sigma}\bm{\mu},\bm{\Sigma}^{2}\bm{\mu},\ldots\}.
  \end{equation}
  The spectral decomposition of $\bm{\Sigma}$ and the Vandermonde determinant show that the right-hand side has dimension $r$, the number of distinct covariance eigenvalues with nonzero projection of $\bm{\mu}$.
  For any symmetric matrix, the dimension of the Krylov space generated by a vector equals the number of distinct eigenvalues whose spectral projectors do not annihilate that vector.
  Applying this fact to $\bm{S}$ proves the claim.
\end{proof}

We characterize the population output induced by block-selective parameters.
We then choose their scalar coefficients to interpolate the desired response on the active spectrum.

\begin{lemma}[Population response of block-selective processing layers]\label{lem:block-selective-population-output}
  Fix an integer $m\geq0$.
  For $1\leq\ell\leq m$, let $\kappa_{\ell},\eta_{\ell}\in\mathbb{R}$ and define
  \begin{equation}
    \bm{P}_{\ell}:=\kappa_{\ell}\bm{\Pi}_{x}+\eta_{\ell}\bm{\Pi}_{y}, \qquad \bm{Q}_{\ell}:=\bm{\Pi}_{x}.
    \label{eq:active-spectrum-processing-layer}
  \end{equation}
  Set $\bm{\Theta}_{m}:=\{(\bm{P}_{\ell},\bm{Q}_{\ell})\}_{\ell=1}^{m}$.
  Set $a_{0}(s):=1$ and $b_{0}(s):=0$, and, for $1\leq\ell\leq m$, recursively define
  \begin{equation}
    a_{\ell}(s):=a_{\ell-1}(s)\{1+\kappa_{\ell}s{a_{\ell-1}(s)}^{2}\}, \qquad b_{\ell}(s):=b_{\ell-1}(s)+\eta_{\ell}{a_{\ell-1}(s)}^{2}\{1+s b_{\ell-1}(s)\}.
    \label{eq:active-spectrum-response-recursion}
  \end{equation}
  Then,
  \begin{equation}
    \bm{A}_{\bm{\Theta}_{m}}^{(m)}(\bm{G}_{\star})=
    \begin{bmatrix}a_{m}(\bm{S}) & \bm{0}_{d} \\ \bm{\mu}^{\top}b_{m}(\bm{S}) & 1
    \end{bmatrix}.
    \label{eq:active-spectrum-cumulative-transformation}
  \end{equation}
  Consequently,
  \begin{equation}
    \bm{\theta}_{\bm{\Theta}_{m}}(\bm{G}_{\star})=b_{m}(\bm{S})\bm{\mu}.
  \end{equation}
\end{lemma}

\begin{proof}
  The claim is immediate at $m=0$.
  Suppose that the claimed block form holds after $\ell$ layers, where $0\leq\ell<m$, and denote this cumulative transformation by $\bm{T}_{\ell}$.
  Direct block multiplication gives
  \begin{equation}
    \bm{T}_{\ell}\bm{G}_{\star}\bm{T}_{\ell}^{\top}\bm{\Pi}_{x}=
    \begin{bmatrix}\bm{S}{a_{\ell}(\bm{S})}^{2} & \bm{0}_{d} \\ \bm{\mu}^{\top}\{\bm{I}_{d}+b_{\ell}(\bm{S})\bm{S}\}a_{\ell}(\bm{S}) & 0
    \end{bmatrix}.
  \end{equation}
  Substituting this expression and the $(\ell+1)$-st parameters from \cref{eq:active-spectrum-processing-layer} into \cref{eq:fixed-depth-context-recursion} gives the recursion in \cref{eq:active-spectrum-response-recursion} and proves \cref{eq:active-spectrum-cumulative-transformation} by induction.
  The definition in \cref{eq:context-fixed-depth-effective-direction} gives $\bm{\theta}_{\bm{\Theta}_{m}}(\bm{G}_{\star})=b_{m}(\bm{S})\bm{\mu}$.
\end{proof}

\begin{lemma}[Sufficient-depth population optimal-direction comparator]\label{lem:sufficient-depth-comparator}
  Fix $\tau\geq0$ and $\omega>0$, define $\bm{h}^{\star}:=\omega{(\bm{\Sigma}+\tau\bm{I}_{d})}^{-1}\bm{\mu}$, and suppose that $L\geq r$.
  There is a parameter tuple $\bm{\Theta}_{L,\star}$ such that, for every $\bm{U}\in\mathcal{O}(d)$,
  \begin{equation}
    \bm{U}^{\top}\bm{\theta}_{\bm{\Theta}_{L,\star}}(\operatorname{diag}(\bm{U},1)\bm{G}_{\star}\operatorname{diag}(\bm{U}^{\top},1))=\bm{h}^{\star}.
    \label{eq:sufficient-depth-population-comparator}
  \end{equation}
\end{lemma}

\begin{proof}
  Let $s_{1},\ldots,s_{r}$ be the active eigenvalues of $\bm{S}$, whose number is $r$ by \cref{lem:active-spectrum-count}, and set
  \begin{equation}
    C:=\omega\{1+\bm{\mu}^{\top}{(\bm{\Sigma}+\tau\bm{I}_{d})}^{-1}\bm{\mu}\}.
  \end{equation}
  It is enough to construct the first $r$ layers in latent coordinates at the population context $\bm{G}_{\star}$.
  For the first $r-1$ layers, use the notation and response recursion of \cref{lem:block-selective-population-output} with $m=r-1$.

  These layers process $r-1$ active eigenvalues sequentially.
  At the beginning of stage $\ell$, let $\mathcal{I}_{\ell}$ be the indices not yet processed and define
  \begin{equation}
    z_{i}^{(\ell)}:=s_{i}{a_{\ell}(s_{i})}^{2}, \qquad i\in\mathcal{I}_{\ell}.
  \end{equation}
  We maintain the invariant that these numbers are positive and pairwise distinct.
  This holds initially because $z_{i}^{(0)}=s_{i}$.
  Let $k_{\ell}$ be the unique index attaining their minimum and set
  \begin{equation}
    \kappa_{\ell+1}:=-\frac{1}{z_{k_{\ell}}^{(\ell)}}, \qquad \eta_{\ell+1}:=\frac{c/(s_{k_{\ell}}+\tau)-b_{\ell}(s_{k_{\ell}})}{{a_{\ell}(s_{k_{\ell}})}^{2}\{1+s_{k_{\ell}}b_{\ell}(s_{k_{\ell}})\}},
    \label{eq:active-spectrum-freezing-parameters}
  \end{equation}
  where $c>0$ will be chosen sufficiently small.
  Equation \cref{eq:active-spectrum-response-recursion} then gives
  \begin{equation}
    a_{\ell+1}(s_{k_{\ell}})=0, \qquad b_{\ell+1}(s_{k_{\ell}})=\frac{c}{s_{k_{\ell}}+\tau}.
    \label{eq:active-spectrum-frozen-response}
  \end{equation}
  Every later increment of either response at this eigenvalue contains a factor $a_{j}(s_{k_{\ell}})$, so the response remains fixed.
  For every remaining index $i\neq k_{\ell}$, the next transformed value is
  \begin{equation}
    z_{i}^{(\ell+1)}=z_{i}^{(\ell)}{\ab(1-\frac{z_{i}^{(\ell)}}{z_{k_{\ell}}^{(\ell)}})}^{2}.
    \label{eq:active-spectrum-z-recursion}
  \end{equation}
  The function $z\mapsto z{(1-z/z_{k_{\ell}}^{(\ell)})}^{2}$ is positive and strictly increasing for $z>z_{k_{\ell}}^{(\ell)}$, because its derivative is $(3z/z_{k_{\ell}}^{(\ell)}-1)(z/z_{k_{\ell}}^{(\ell)}-1)>0$.
  Thus, positivity and pairwise distinctness hold at the next stage.
  When $c=0$, induction gives $b_{\ell}(s_{i})=0$ and $\eta_{\ell}=0$ at every stage.
  The finitely many responses and denominators in \cref{eq:active-spectrum-freezing-parameters} depend continuously on $c$ near zero, so some sufficiently small $c>0$ makes every denominator nonzero.

  Suppose first that $r\geq2$, and let $k_{\star}$ denote the single index left after these $r-1$ stages.
  Let $\bm{T}_{r-1}:=\bm{A}_{\bm{\Theta}_{r-1}}^{(r-1)}(\bm{G}_{\star})$ and define
  \begin{equation}
    g_{r-1}:=\bm{e}_{d+1}^{\top}\bm{T}_{r-1}\bm{G}_{\star}\bm{T}_{r-1}^{\top}\bm{e}_{d+1}>0.
  \end{equation}
  The strict inequality holds because the Schur complement of the bottom-right entry of $\bm{G}_{\star}$ is $\bm{\Sigma}\succ\bm{0}_{d,d}$ and $\bm{T}_{r-1}^{\top}\bm{e}_{d+1}$ has final coordinate one.
  In the final layer, choose
  \begin{equation}
    \bm{P}_{r}:=\bm{\Pi}_{y}, \qquad \bm{Q}_{r}:=\alpha\bm{\Pi}_{x}+\beta\bm{\Pi}_{y}.
    \label{eq:active-spectrum-final-layer}
  \end{equation}
  Direct block multiplication gives $\bm{\theta}_{\bm{\Theta}_{r}}(\bm{G}_{\star})=b_{r}(\bm{S})\bm{\mu}$, where
  \begin{equation}
    b_{r}(s):=(1+\beta g_{r-1})b_{r-1}(s)+\alpha{a_{r-1}(s)}^{2}\{1+s b_{r-1}(s)\}.
    \label{eq:active-spectrum-final-response}
  \end{equation}
  Shrinking $c>0$ if necessary ensures that $1+s_{k_{\star}}b_{r-1}(s_{k_{\star}})\neq0$, because this expression is continuous in $c$ and equals one at $c=0$.
  Set
  \begin{equation}
    \beta:=\frac{C/c-1}{g_{r-1}}, \qquad \alpha:=\frac{C/(s_{k_{\star}}+\tau)-(C/c)b_{r-1}(s_{k_{\star}})}{{a_{r-1}(s_{k_{\star}})}^{2}\{1+s_{k_{\star}}b_{r-1}(s_{k_{\star}})\}}.
    \label{eq:active-spectrum-final-parameters}
  \end{equation}
  The denominator defining $\alpha$ is therefore nonzero because the induction invariant also gives $a_{r-1}(s_{k_{\star}})\neq0$.
  Equations \cref{eq:active-spectrum-frozen-response,eq:active-spectrum-final-response,eq:active-spectrum-final-parameters} imply
  \begin{equation}
    b_{r}(s_{i})=\frac{C}{s_{i}+\tau}, \qquad i\in[r].
    \label{eq:active-spectrum-interpolation}
  \end{equation}
  If $r=1$, the same conclusion follows directly by choosing $\bm{P}_{1}:=\bm{\Pi}_{y}$ and $\bm{Q}_{1}:=C{(s_{1}+\tau)}^{-1}\bm{\Pi}_{x}$.
  The spectral theorem and activity of $s_{1},\ldots,s_{r}$ give
  \begin{equation}
    b_{r}(\bm{S})\bm{\mu}=C{(\bm{S}+\tau\bm{I}_{d})}^{-1}\bm{\mu},
  \end{equation}
  and the Sherman--Morrison identity gives
  \begin{equation}
    C{(\bm{S}+\tau\bm{I}_{d})}^{-1}\bm{\mu}=\omega{(\bm{\Sigma}+\tau\bm{I}_{d})}^{-1}\bm{\mu}=\bm{h}^{\star}.
    \label{eq:second-moment-inverse-direction}
  \end{equation}
  Every constructed parameter commutes with $\operatorname{diag}(\bm{U},1)$, so rotation equivariance gives \cref{eq:sufficient-depth-population-comparator} at depth $r$.
  If $L>r$, set $\bm{P}_{\ell}=\bm{Q}_{\ell}=\bm{0}_{d+1,d+1}$ for $r<\ell\leq L$.
  These additional zero layers leave the effective direction unchanged, completing the proof for every $L\geq r$.
\end{proof}

We finally transfer the constructed comparator from the population context to finite contexts.

\begin{lemma}[Finite-context accuracy of the sufficient-depth comparator]\label{lem:sufficient-depth-comparator-rate}
  Under the conditions of \cref{lem:sufficient-depth-comparator}, let $\bm{\Theta}_{L,\star}$ and $\bm{h}^{\star}$ be the comparator and target direction constructed there.
  There is a constant $C_{\mathrm{cmp}}>0$ such that, for every $\bm{U}\in\mathcal{O}(d)$ and $N\in\mathbb{N}$,
  \begin{equation}
    \mathbb{E}_{\mathcal{D}_{N}\mid\bm{U}}[\|\bm{U}^{\top}\bm{\theta}_{\bm{\Theta}_{L,\star}}(\bm{G}_{N}(\bm{U}))-\bm{h}^{\star}\|_{2}^{2}]\leq\frac{C_{\mathrm{cmp}}}{N}.
    \label{eq:active-spectrum-comparator-rate}
  \end{equation}
\end{lemma}

\begin{proof}
  Conditional on any fixed $\bm{U}$, induction in \cref{eq:fixed-depth-context-recursion} gives
  \begin{equation}
    \bm{U}^{\top}\bm{\theta}_{\bm{\Theta}_{L,\star}}(\bm{G}_{N}(\bm{U}))=\bm{\theta}_{\bm{\Theta}_{L,\star}}(\bm{G}_{a,N}).
  \end{equation}
  The right-hand side is a fixed polynomial in the entries of $N^{-1}\sum_{n=1}^{N}\bm{a}_{n}$ and $N^{-1}\sum_{n=1}^{N}\bm{a}_{n}\bm{a}_{n}^{\top}$ whose value at $(\bm{\mu},\bm{S})$ is $\bm{h}^{\star}$.
  For every fixed integer $k\geq1$, Gaussian sample-moment bounds give
  \begin{equation}
    \mathbb{E}_{\bm{a}_{1:N}}\ab[\|\frac{1}{N}\sum_{n=1}^{N}\bm{a}_{n}-\bm{\mu}\|_{2}^{2k}+\|\frac{1}{N}\sum_{n=1}^{N}\bm{a}_{n}\bm{a}_{n}^{\top}-\bm{S}\|_{\mathrm{F}}^{2k}]\leq\frac{C_{k}}{N^{k}}, \qquad N\in\mathbb{N},
    \label{eq:gaussian-empirical-moment-rate}
  \end{equation}
  for a constant $C_{k}>0$.
  Expanding the polynomial difference around $(\bm{\mu},\bm{S})$ and applying these bounds to its finitely many monomials proves \cref{eq:active-spectrum-comparator-rate}.
\end{proof}

\begin{tcolorbox}[theorem box]
  \robustbayesrealizationmechanism*
\end{tcolorbox}

\begin{proof}
  Part (i) is the optimal-direction conclusion of \cref{lem:robust-bayes-characterization}.

  For part (ii), let $\lambda_{1},\ldots,\lambda_{r}$ be the distinct active eigenvalues of $\bm{\Sigma}$.
  Polynomial interpolation gives a polynomial $p$ of degree at most $r-1$ satisfying $p(\lambda_{i})={(\lambda_{i}+\tau)}^{-1}$ for every $i\in[r]$.
  The spectral theorem gives
  \begin{equation}
    {(\bm{\Sigma}+\tau\bm{I}_{d})}^{-1}\bm{\mu}=p(\bm{\Sigma})\bm{\mu}\in\mathcal{K}_{r}(\bm{\Sigma},\bm{\mu}).
    \label{eq:robust-direction-latent-krylov-representation}
  \end{equation}
  The polynomial $p$ has degree exactly $r-1$.
  This is immediate when $r=1$.
  If $r\geq2$ and $p$ had smaller degree, then $(s+\tau)p(s)-1$ would have degree at most $r-1$ and vanish at the $r$ distinct points $\lambda_{1},\ldots,\lambda_{r}$, so it would vanish identically.
  This contradicts its value $-1$ at $s=-\tau$.
  Since $\bm{\mu},\bm{\Sigma}\bm{\mu},\ldots,\bm{\Sigma}^{r-1}\bm{\mu}$ are linearly independent, the nonzero leading coefficient of $p$ implies $p(\bm{\Sigma})\bm{\mu}\notin\mathcal{K}_{r-1}(\bm{\Sigma},\bm{\mu})$, proving part (ii).

  For parts (iii) and (iv), fix $t\in\mathbb{R}$ and define the shared parameter tuple $\bm{\Theta}_{L}(t)$ by
  \begin{equation}
    \bm{P}_{\ell}(t):=\bm{\Pi}_{y}, \qquad \bm{Q}_{\ell}(t):=t\bm{\Pi}_{x}, \qquad \ell\in[L].
    \label{eq:krylov-span-parameter-family}
  \end{equation}
  These parameters do not depend on $\bm{U}$.
  Set $q_{0,t}(s):=0$.
  Direct induction in \cref{eq:fixed-depth-context-recursion} at the latent population context $\bm{G}_{\star}$ gives
  \begin{equation}
    \bm{A}_{\bm{\Theta}_{L}(t)}^{(\ell)}(\bm{G}_{\star})=
    \begin{bmatrix}
      \bm{I}_{d} & \bm{0}_{d} \\
      \bm{\mu}^{\top}q_{\ell,t}(\bm{S}) & 1
    \end{bmatrix}, \qquad
    q_{\ell,t}(s)=q_{\ell-1,t}(s)+t\{1+s q_{\ell-1,t}(s)\}.
    \label{eq:krylov-span-population-recursion}
  \end{equation}
  Solving this scalar recursion yields
  \begin{equation}
    q_{L,t}(s)=\sum_{j=1}^{L}\binom{L}{j}t^{j}s^{j-1}.
    \label{eq:krylov-span-polynomial-family}
  \end{equation}
  Since $\bm{\Pi}_{x}$ and $\bm{\Pi}_{y}$ commute with $\operatorname{diag}(\bm{U},1)$, rotation equivariance and \cref{eq:context-fixed-depth-effective-direction} give
  \begin{equation}
    \bm{\theta}_{\bm{\Theta}_{L}(t)}(\operatorname{diag}(\bm{U},1)\bm{G}_{\star}\operatorname{diag}(\bm{U}^{\top},1))=\bm{U}q_{L,t}(\bm{S})\bm{\mu}.
    \label{eq:krylov-span-population-direction}
  \end{equation}
  Therefore, the population-context version of \cref{eq:finite-context-query-linear-representation} gives
  \begin{equation}
    f(\bm{x}; \bm{\Theta}_{L}(t), \mathcal{D}_{N})=\bm{\mu}^{\top}q_{L,t}(\bm{S})\bm{U}^{\top}\bm{x}, \qquad N\to\infty,
  \end{equation}
  which proves part (iii).

  Choose distinct nonzero numbers $t_{1},\ldots,t_{L}$.
  Relative to the monomial basis $1,s,\ldots,s^{L-1}$, the coefficient matrix of $q_{L,t_{1}},\ldots,q_{L,t_{L}}$ has entries $\binom{L}{j}t_{i}^{j}$ and reduces to an invertible Vandermonde matrix after nonzero row and column scalings.
  Hence, $\{q_{L,t}:t\in\mathbb{R}\}$ spans all polynomials of degree at most $L-1$.
  Since $\bm{S}=\bm{\Sigma}+\bm{\mu}\bm{\mu}^{\top}$, induction on the Krylov order gives $\mathcal{K}_{L}(\bm{S},\bm{\mu})=\mathcal{K}_{L}(\bm{\Sigma},\bm{\mu})$.
  Combining these facts yields
  \begin{equation}
    \begin{aligned}
      \operatorname{span}\{q_{L,t}(\bm{S})\bm{\mu}:t\in\mathbb{R}\}
      &=\operatorname{span}(\bm{\mu},\bm{S}\bm{\mu},\ldots,\bm{S}^{L-1}\bm{\mu}) \\
      &=\mathcal{K}_{L}(\bm{\Sigma},\bm{\mu}),
    \end{aligned}
  \end{equation}
  proving part (iv).
\end{proof}

\begin{tcolorbox}[theorem box]
  \robustbayesrepresentability*
\end{tcolorbox}

\begin{proof}
  Fix $\bm{U}\in\mathcal{O}(d)$, $N\in\mathbb{N}$, and $\delta\in(0,1)$.
  Let $\bm{\Theta}_{L,\star}$ be the parameter tuple from \cref{lem:sufficient-depth-comparator} with the solution $\tau$ of \cref{eq:robust-shift-equation} and $\omega=1$.
  Define $\psi(\bm{h}):=\mathcal{E}_{\epsilon,\bm{U}}(\bm{x}\mapsto{\bm{h}}^{\top}\bm{U}^{\top}\bm{x})$ for $\bm{h}\in\mathbb{R}^{d}$.
  By \cref{eq:linear-robust-error}, this function does not depend on $\bm{U}$.
  By \cref{lem:robust-bayes-characterization}, $\psi(\bm{h}_{\tau})=\mathcal{E}_{\epsilon,\bm{U}}^{\star}$ and $\bm{h}_{\tau}$ is a global minimizer of $\psi$.
  The vector $\bm{h}_{\tau}$ is nonzero, so $\psi$ is twice continuously differentiable on a neighborhood of $\bm{h}_{\tau}$.
  Its gradient vanishes at this global minimizer, and Taylor's theorem gives a quadratic upper bound for $\psi(\bm{h})-\psi(\bm{h}_{\tau})$ on a sufficiently small ball centered at $\bm{h}_{\tau}$.
  Outside that ball, the same bound holds after increasing its constant because $0\leq\psi\leq1$ and the distance from $\bm{h}_{\tau}$ is bounded away from zero.
  Consequently, there is a constant $C_{\psi}>0$ such that
  \begin{equation}
    0\leq\psi(\bm{h})-\psi(\bm{h}_{\tau})\leq C_{\psi}\min\{\|\bm{h}-\bm{h}_{\tau}\|_{2}^{2},1\}, \qquad \bm{h}\in\mathbb{R}^{d}.
    \label{eq:global-quadratic-robust-error-stability}
  \end{equation}

  Let $X_{N}$ denote the excess error on the left-hand side of \cref{eq:main-robust-bayes-representability-rate} with $\bm{\Theta}_{L}=\bm{\Theta}_{L,\star}$.
  Conditional on a realized context, \cref{lem:robust-bayes-characterization,eq:global-quadratic-robust-error-stability} give
  \begin{equation}
    0\leq X_{N}\leq C_{\psi}\|\bm{U}^{\top}\bm{\theta}_{\bm{\Theta}_{L,\star}}(\bm{G}_{N}(\bm{U}))-\bm{h}_{\tau}\|_{2}^{2}.
  \end{equation}
  Taking the conditional expectation and applying \cref{lem:sufficient-depth-comparator-rate} yield
  \begin{equation}
    \mathbb{E}_{\mathcal{D}_{N}\mid\bm{U}}[X_{N}]\leq\frac{C_{\psi}C_{\mathrm{cmp}}}{N}.
  \end{equation}
  Markov's inequality therefore gives
  \begin{equation}
    \mathbb{P}_{\mathcal{D}_{N}\mid\bm{U}}\ab(X_{N}>\frac{C_{\psi}C_{\mathrm{cmp}}}{\delta N})\leq\delta.
  \end{equation}
  Since $\bm{U}$ was arbitrary and the constants are independent of $\bm{U}$, this proves \cref{eq:main-robust-bayes-representability-rate} with $C:=C_{\psi}C_{\mathrm{cmp}}$.
\end{proof}

\section{Adversarial and Standard Training: Proofs of \texorpdfstring{\cref{thm:main-adversarial-training-bayes-optimality,thm:main-standard-training-bayes-optimality}}{the adversarial- and standard-training theorems}}\label{sec:appendix-adversarial-training}

\subsection{Directions selected by infinite-sample margin risks}

For $\epsilon\geq0$, define the population margin risk
\begin{equation}
  \mathcal{J}_{\epsilon}(\bm{h}):=\mathbb{E}_{\bm{a}\sim\mathcal{N}(\bm{\mu},\bm{\Sigma})}[\mathcal{L}(\bm{h}^{\top}\bm{a}-\epsilon\sqrt{d}\|\bm{h}\|_{2})].
  \label{eq:latent-margin-risk}
\end{equation}
The following result shows that minimizing this population surrogate risk selects the corresponding Bayes-optimal direction.

\begin{lemma}[Population margin-risk minimizers align with Bayes-optimal directions]\label{lem:infinite-sample-margin-optimal-directions}
  Suppose that $0\leq\epsilon\sqrt{d}<\|\bm{\mu}\|_{2}$.
  Then the function $\mathcal{J}_{\epsilon}$ is strictly convex and coercive and therefore has a unique minimizer $\bm{h}_{\epsilon}^{\star}$, which is nonzero.
  There is a constant $\omega_{\epsilon}>0$ such that
  \begin{equation}
    \operatorname*{argmin}_{\bm{h}\in\mathbb{R}^{d}}\mathcal{J}_{\epsilon}(\bm{h})=\{\bm{h}_{\epsilon}^{\star}\}=\{\omega_{\epsilon}\bm{h}_{\tau}\}=\{\omega_{\epsilon}{(\bm{\Sigma}+\tau\bm{I}_{d})}^{-1}\bm{\mu}\}.
    \label{eq:margin-risk-optimal-direction}
  \end{equation}
\end{lemma}

\begin{proof}
  For every fixed $\bm{a}$, the map $\bm{h}\mapsto\bm{h}^{\top}\bm{a}-\epsilon\sqrt{d}\|\bm{h}\|_{2}$ is concave, so its composition with the convex nonincreasing function $\mathcal{L}$ is convex.
  For two distinct directions, the corresponding random margins differ almost surely because $\bm{\Sigma}\succ\bm{0}_{d,d}$.
  Strict convexity of $\mathcal{L}$, together with concavity of the margin map, therefore makes the expected risk strictly convex.
  Moreover, for every unit vector $\bm{u}$ and $t>0$,
  \begin{equation}
    \mathcal{J}_{\epsilon}(t\bm{u})\geq-\mathcal{L}'(0)t\mathbb{E}_{\bm{a}\sim\mathcal{N}(\bm{\mu},\bm{\Sigma})}[{(\epsilon\sqrt{d}-\bm{u}^{\top}\bm{a})}_{+}].
  \end{equation}
  This bound follows from the tangent inequality for $\mathcal{L}$ on the event $\bm{u}^{\top}\bm{a}<\epsilon\sqrt{d}$ and nonnegativity of $\mathcal{L}$ on its complement.
  The expectation on the right is positive and continuous in $\bm{u}$, so compactness of the unit sphere gives a positive uniform lower bound.
  Hence, $\mathcal{J}_{\epsilon}$ is coercive and has a unique minimizer.
  Its one-sided derivative at the origin along $\bm{\mu}/\|\bm{\mu}\|_{2}$ equals $\mathcal{L}'(0)(\|\bm{\mu}\|_{2}-\epsilon\sqrt{d})<0$, so the minimizer is nonzero.

  Let $\bm{h}:=\bm{h}_{\epsilon}^{\star}$ and set $Z:=\bm{h}^{\top}\bm{a}-\epsilon\sqrt{d}\|\bm{h}\|_{2}$.
  The derivative bounds and Gaussian moments justify differentiation under the expectation and the Gaussian integration by parts below.
  The first-order condition at the nonzero minimizer is
  \begin{equation}
    \mathbb{E}_{\bm{a}\sim\mathcal{N}(\bm{\mu},\bm{\Sigma})}\ab[\mathcal{L}'(Z)\ab(\bm{a}-\epsilon\sqrt{d}\frac{\bm{h}}{\|\bm{h}\|_{2}})]=\bm{0}_{d}.
  \end{equation}
  Gaussian integration by parts gives
  \begin{equation}
    \mathbb{E}_{\bm{a}\sim\mathcal{N}(\bm{\mu},\bm{\Sigma})}[\mathcal{L}'(Z)\bm{a}]=\mathbb{E}_{\bm{a}\sim\mathcal{N}(\bm{\mu},\bm{\Sigma})}[\mathcal{L}'(Z)]\bm{\mu}+\mathbb{E}_{\bm{a}\sim\mathcal{N}(\bm{\mu},\bm{\Sigma})}[\mathcal{L}''(Z)]\bm{\Sigma}\bm{h}.
    \label{eq:margin-risk-stein-identity}
  \end{equation}
  Substitution into the first-order condition yields
  \begin{equation}
    \ab(c\bm{\Sigma}+\frac{\epsilon\sqrt{d}}{\|\bm{h}\|_{2}}\bm{I}_{d})\bm{h}=\bm{\mu}, \qquad c:=\frac{\mathbb{E}_{\bm{a}\sim\mathcal{N}(\bm{\mu},\bm{\Sigma})}[\mathcal{L}''(Z)]}{-\mathbb{E}_{\bm{a}\sim\mathcal{N}(\bm{\mu},\bm{\Sigma})}[\mathcal{L}'(Z)]}>0.
    \label{eq:margin-risk-regularized-normal-equation}
  \end{equation}
  Setting $\lambda:=\epsilon\sqrt{d}/(c\|\bm{h}\|_{2})$ gives
  \begin{equation}
    \bm{h}=\frac{1}{c}{(\bm{\Sigma}+\lambda\bm{I}_{d})}^{-1}\bm{\mu}, \qquad \lambda\|{(\bm{\Sigma}+\lambda\bm{I}_{d})}^{-1}\bm{\mu}\|_{2}=\epsilon\sqrt{d}.
  \end{equation}
  Uniqueness in \cref{lem:symmetric-gaussian-robust-bayes-classifier} gives $\lambda=\tau$.
  Setting $\omega_{\epsilon}:=1/c$ proves \cref{eq:margin-risk-optimal-direction}.
\end{proof}

\subsection{Infinite-sample minimizers and finite-context transfer}

The infinite-task, infinite-context risk is
\begin{equation}
  \mathcal{R}_{\epsilon,\infty,\infty}(\bm{\Theta}_{L})=\mathbb{E}_{\bm{U}\in\mathcal{O}(d)}[\mathcal{J}_{\epsilon}(\bm{U}^{\top}\bm{\theta}_{\bm{\Theta}_{L}}(\operatorname{diag}(\bm{U},1)\bm{G}_{\star}\operatorname{diag}(\bm{U}^{\top},1)))].
  \label{eq:infinite-sample-fixed-depth-margin-risk}
\end{equation}

The preceding result also identifies the infinite-sample global optimum of the unrestricted transformer risk, independently of gradient dynamics.

\begin{theorem}[Functional identification of unrestricted fixed-depth infinite-sample minimizers]\label{thm:unrestricted-fixed-depth-global-minimizers}
  Suppose that $L\geq r$ and $0\leq\epsilon\sqrt{d}<\|\bm{\mu}\|_{2}$.
  The unrestricted infinite-task, infinite-context risk in \cref{eq:infinite-sample-fixed-depth-margin-risk} attains the lower bound
  \begin{equation}
    \inf_{\bm{\Theta}_{L}}\mathcal{R}_{\epsilon,\infty,\infty}(\bm{\Theta}_{L})=\mathcal{J}_{\epsilon}(\bm{h}_{\epsilon}^{\star}).
    \label{eq:unrestricted-fixed-depth-global-risk}
  \end{equation}
  Every global minimizer satisfies
  \begin{equation}
    \bm{U}^{\top}\bm{\theta}_{\bm{\Theta}_{L}}(\operatorname{diag}(\bm{U},1)\bm{G}_{\star}\operatorname{diag}(\bm{U}^{\top},1))=\bm{h}_{\epsilon}^{\star}
    \label{eq:unrestricted-global-minimizer-direction}
  \end{equation}
  for every $\bm{U}\in\mathcal{O}(d)$.
  In particular, standard training identifies the direction $\bm{\Sigma}^{-1}\bm{\mu}$, whereas adversarial training identifies the direction $\bm{h}_{\tau}$ in \cref{eq:robust-bayes-direction}.
\end{theorem}

\begin{proof}
  Fix $\epsilon\geq0$ satisfying $\epsilon\sqrt{d}<\|\bm{\mu}\|_{2}$.
  Conditional on $\bm{U}$, orthogonality gives
  \begin{equation}
    \mathcal{J}_{\epsilon}(\bm{U}^{\top}\bm{\theta}_{\bm{\Theta}_{L}}(\operatorname{diag}(\bm{U},1)\bm{G}_{\star}\operatorname{diag}(\bm{U}^{\top},1)))\geq\mathcal{J}_{\epsilon}(\bm{h}_{\epsilon}^{\star}).
    \label{eq:conditional-margin-lower-bound}
  \end{equation}
  Apply \cref{lem:sufficient-depth-comparator} with the solution $\tau$ of \cref{eq:robust-shift-equation} and $\omega=\omega_{\epsilon}$.
  Equation \cref{eq:margin-risk-optimal-direction} shows that the resulting comparator attains equality simultaneously for every $\bm{U}$.
  Hence, the lower bound in \cref{eq:conditional-margin-lower-bound} is attainable simultaneously for every $\bm{U}$, proving \cref{eq:unrestricted-fixed-depth-global-risk}.
  If a parameter tuple attains this value, the nonnegative conditional optimality gap in \cref{eq:conditional-margin-lower-bound} has expectation zero.
  The effective direction in \cref{eq:context-fixed-depth-effective-direction} is continuous in $\bm{U}$, so this conditional optimality gap is continuous on $\mathcal{O}(d)$.
  Since normalized Haar measure has full support, a nonnegative continuous function with zero Haar integral vanishes everywhere.
  Strict convexity from \cref{lem:infinite-sample-margin-optimal-directions} therefore gives \cref{eq:unrestricted-global-minimizer-direction} for every $\bm{U}$.
\end{proof}

For $\epsilon\geq0$ and $\bm{h}\neq\bm{0}_{d}$, define
\begin{equation}
  \Gamma_{\epsilon}(\bm{h}):=\frac{\bm{\mu}^{\top}\bm{h}-\epsilon\sqrt{d}\|\bm{h}\|_{2}}{\sqrt{\bm{h}^{\top}\bm{\Sigma}\bm{h}}},
  \label{eq:robust-gaussian-score}
\end{equation}
and set $\Gamma_{\epsilon}(\bm{0}_{d}):=-\infty$.

\begin{lemma}[Uniform finite-context transfer]\label{lem:uniform-finite-context-transfer}
  Fix a depth-$L$ parameter tuple $\bm{\Theta}_{L}$ and suppose that there is a nonzero vector $\bm{h}\in\mathbb{R}^{d}$ such that its population-context effective direction in \cref{eq:context-fixed-depth-effective-direction} satisfies
  \begin{equation}
    \bm{U}^{\top}\bm{\theta}_{\bm{\Theta}_{L}}(\operatorname{diag}(\bm{U},1)\bm{G}_{\star}\operatorname{diag}(\bm{U}^{\top},1))=\bm{h}, \qquad \bm{U}\in\mathcal{O}(d).
  \end{equation}
  Then, for every $\epsilon\geq0$ satisfying $\epsilon\sqrt{d}<\|\bm{\mu}\|_{2}$,
  \begin{equation}
    \lim_{N\to\infty}\sup_{\bm{U}\in\mathcal{O}(d)}\ab|\mathbb{E}_{\mathcal{D}_{N}\mid\bm{U}}[\mathcal{E}_{\epsilon,\bm{U}}(f(\,\cdot\,; \bm{\Theta}_{L}, \mathcal{D}_{N}))]-\Phi(-\Gamma_{\epsilon}(\bm{h}))|=0.
    \label{eq:uniform-finite-context-error-transfer}
  \end{equation}
\end{lemma}

\begin{proof}
  Couple the contexts for all task bases using the same latent vectors $\bm{a}_{1},\ldots,\bm{a}_{N}$.
  Equations \cref{eq:context-blocks,eq:latent-finite-context-matrix} and the strong law give
  \begin{equation}
    \sup_{\bm{U}\in\mathcal{O}(d)}\|\bm{G}_{N}(\bm{U})-\operatorname{diag}(\bm{U},1)\bm{G}_{\star}\operatorname{diag}(\bm{U}^{\top},1)\|_{\mathrm{F}}=\|\bm{G}_{a,N}-\bm{G}_{\star}\|_{\mathrm{F}}\to0
  \end{equation}
  almost surely.
  Continuity of the fixed-depth recursion in \cref{eq:fixed-depth-context-recursion} and compactness of $\mathcal{O}(d)$ imply that the finite-context latent direction converges to $\bm{h}$ uniformly over $\bm{U}$ almost surely.
  Since $\bm{h}\neq\bm{0}_{d}$, the map $\bm{v}\mapsto\Phi(-\Gamma_{\epsilon}(\bm{v}))$ is continuous on a neighborhood of the limiting directions.
  By \cref{lem:robust-bayes-characterization}, the conditional robust-error formula therefore converges uniformly over $\bm{U}$ almost surely.
  Since the robust errors and their limit belong to $[0,1]$, bounded convergence gives
  \begin{equation}
    \mathbb{E}_{\bm{a}_{1},\ldots,\bm{a}_{N}}\ab[\sup_{\bm{U}\in\mathcal{O}(d)}\ab|\mathcal{E}_{\epsilon,\bm{U}}(f(\,\cdot\,; \bm{\Theta}_{L}, \mathcal{D}_{N}))-\Phi(-\Gamma_{\epsilon}(\bm{h}))|]\to0.
  \end{equation}
  Under the coupling above, the distribution of each context agrees with that of $\mathcal{D}_{N}$ conditional on its task basis, and hence
  \begin{equation}
    \sup_{\bm{U}\in\mathcal{O}(d)}\ab|\mathbb{E}_{\mathcal{D}_{N}\mid\bm{U}}[\mathcal{E}_{\epsilon,\bm{U}}(f(\,\cdot\,; \bm{\Theta}_{L}, \mathcal{D}_{N}))]-\Phi(-\Gamma_{\epsilon}(\bm{h}))|\leq\mathbb{E}_{\bm{a}_{1},\ldots,\bm{a}_{N}}\ab[\sup_{\bm{U}\in\mathcal{O}(d)}\ab|\mathcal{E}_{\epsilon,\bm{U}}(f(\,\cdot\,; \bm{\Theta}_{L}, \mathcal{D}_{N}))-\Phi(-\Gamma_{\epsilon}(\bm{h}))|],
  \end{equation}
  which proves \cref{eq:uniform-finite-context-error-transfer}.
\end{proof}

\subsection{Finite-task adversarial risk minimization}\label{sec:appendix-finite-learning}

In statements and arguments involving a single generic task, we suppress the task index $t$ from $\bm{U}_{t}$, $\mathcal{D}_{t,N}$, $\bm{x}_{t,n}$, $y_{t,n}$, and derived quantities, whereas we retain $t$ when the $T$ training tasks are considered jointly.
The parameter constraint in \cref{eq:main-adversarial-erm} makes the empirical minimum attainable and is assumed to contain the comparator from \cref{lem:sufficient-depth-comparator}.
For parameter tuples, $\|\bm{\Theta}_{L}\|_{\mathrm{F}}$ denotes the Euclidean product of the $2L$ matrix Frobenius norms.

Let $\mathbb{P}_{\mathrm{task}}$ denote the distribution of a generic task $(\mathcal{D}_{N}, \bm{x}, y)$ induced by the generative process.
Let $\mathcal{S}_{T,N}:={\{(\mathcal{D}_{t,N},\bm{x}_{t},y_{t})\}}_{t=1}^{T}$ consist of $T$ independent tasks drawn from $\mathbb{P}_{\mathrm{task}}$.
For the finite-context representation above, the finite-sample risk in \cref{eq:main-finite-sample-risk} at perturbation level $\epsilon$ is
\begin{equation}
  \mathcal{R}_{\epsilon,T,N}(\bm{\Theta}_{L})=\frac{1}{T}\sum_{t=1}^{T}\mathcal{L}(y_{t}{\bm{\theta}_{\bm{\Theta}_{L}}(\bm{G}_{N}(\bm{U}_{t}))}^{\top}\bm{x}_{t}-\epsilon\sqrt{d}\|\bm{\theta}_{\bm{\Theta}_{L}}(\bm{G}_{N}(\bm{U}_{t}))\|_{2}).
  \label{eq:finite-task-margin-risk}
\end{equation}
Its infinite-task counterpart at context length $N$ is
\begin{equation}
  \mathcal{R}_{\epsilon,\infty,N}(\bm{\Theta}_{L}):=\mathbb{E}_{(\bm{U},\mathcal{D}_{N})}[\mathcal{J}_{\epsilon}(\bm{U}^{\top}\bm{\theta}_{\bm{\Theta}_{L}}(\bm{G}_{N}(\bm{U})))].
\end{equation}

\begin{tcolorbox}[theorem box]
  \adversarialtrainingbayesoptimality*
\end{tcolorbox}

\begin{proof}
  Let $\bm{\Theta}_{L,\star}$ be the comparator from \cref{lem:sufficient-depth-comparator} with the solution $\tau$ of \cref{eq:robust-shift-equation} and $\omega=\omega_{\epsilon}$ from \cref{eq:margin-risk-optimal-direction}.

  \textbf{Uniform concentration over training tasks.}
  Let $q_{\bm{\Theta}_{L}}(\mathcal{Z}_{N})$ denote the summand in \cref{eq:finite-task-margin-risk}, where $\mathcal{Z}_{N}$ contains one context and its independent query.
  Equations \cref{eq:fixed-depth-context-recursion,eq:context-fixed-depth-effective-direction} show that $\bm{\theta}_{\bm{\Theta}_{L}}(\bm{G}_{N}(\bm{U}))$ is a polynomial of fixed degree in $\bm{\Theta}_{L}$ and $\bm{G}_{N}(\bm{U})$.
  Compactness of the parameter constraint in \cref{eq:main-adversarial-erm} and boundedness of $\mathcal{L}'$ imply the existence of nonnegative envelopes $F$ and $M$ such that
  \begin{equation}
    |q_{\bm{\Theta}_{L}}(\mathcal{Z}_{N})|\leq F(\mathcal{Z}_{N}), \qquad |q_{\bm{\Theta}_{L}}(\mathcal{Z}_{N})-q_{\widetilde{\bm{\Theta}}_{L}}(\mathcal{Z}_{N})|\leq M(\mathcal{Z}_{N})\|\bm{\Theta}_{L}-\widetilde{\bm{\Theta}}_{L}\|_{\mathrm{F}}.
  \end{equation}
  The envelopes can be chosen as fixed polynomials in $1+\|\bm{G}_{N}(\bm{U})\|_{\mathrm{F}}$ and $1+\|\bm{x}\|_{2}$, with coefficients independent of $T$ and $N$.
  The Gaussian task model and the averaging in $\bm{G}_{N}(\bm{U})$ give
  \begin{equation}
    \sup_{N\in\mathbb{N}}\mathbb{E}_{\mathcal{Z}_{N}}[{F(\mathcal{Z}_{N})}^{2}+{M(\mathcal{Z}_{N})}^{2}]<\infty.
    \label{eq:finite-task-uniform-envelope-moments}
  \end{equation}
  The constrained parameter set is compact and finite-dimensional and hence has logarithmic Euclidean covering entropy in the covering radius.
  Define
  \begin{equation}
    \Delta_{T,N}:=\max_{\substack{\bm{\Theta}_{L}={\{(\bm{P}_{\ell},\bm{Q}_{\ell})\}}_{\ell=1}^{L} \\ \max_{\ell\in[L]}\{\|\bm{P}_{\ell}\|_{\mathrm{F}},\|\bm{Q}_{\ell}\|_{\mathrm{F}}\}\leq B}}|\mathcal{R}_{\epsilon,T,N}(\bm{\Theta}_{L})-\mathcal{R}_{\epsilon,\infty,N}(\bm{\Theta}_{L})|.
  \end{equation}
  Symmetrization, the Euclidean covering bound, and Cauchy--Schwarz with \cref{eq:finite-task-uniform-envelope-moments} yield a constant $C_{\mathrm{conc}}>0$ such that, for every $T,N\in\mathbb{N}$,
  \begin{equation}
    \mathbb{E}_{\mathcal{S}_{T,N}\sim\mathbb{P}_{\mathrm{task}}^{\otimes T}}[\Delta_{T,N}]\leq C_{\mathrm{conc}}T^{-1/2}.
  \end{equation}

  \textbf{Loss geometry and calibration.}
  Define $\psi(\bm{h}):=\Phi(-\Gamma_{\epsilon}(\bm{h}))$, $\psi^{\star}:=\psi(\bm{h}_{\epsilon}^{\star})$, and
  \begin{equation}
    m(\bm{h}):=\min\{\|\bm{h}-\bm{h}_{\epsilon}^{\star}\|_{2}^{2},1\}.
  \end{equation}
  By \cref{lem:infinite-sample-margin-optimal-directions}, $\bm{h}_{\epsilon}^{\star}=\omega_{\epsilon}\bm{h}_{\tau}$.
  Positive scale invariance of $\Gamma_{\epsilon}$ and \cref{eq:robust-gaussian-score,eq:linear-robust-error,eq:taskwise-robust-bayes-error} therefore give
  \begin{equation}
    \psi^{\star}=\Phi(-\Gamma_{\epsilon}(\bm{h}_{\tau}))=\mathcal{E}_{\epsilon,\bm{U}}(\bm{x}\mapsto{\bm{h}_{\tau}}^{\top}\bm{U}^{\top}\bm{x})=\mathcal{E}_{\epsilon,\bm{U}}^{\star}, \qquad \bm{U}\in\mathcal{O}(d).
  \end{equation}
  The function $\mathcal{J}_{\epsilon}(\,\cdot\,)$ is twice continuously differentiable near the nonzero vector $\bm{h}_{\epsilon}^{\star}$, and its gradient vanishes there.
  Set $\bm{u}^{\star}:=\bm{h}_{\epsilon}^{\star}/\|\bm{h}_{\epsilon}^{\star}\|_{2}$ and $Z^{\star}:={\bm{h}_{\epsilon}^{\star}}^{\top}\bm{a}-\epsilon\sqrt{d}\|\bm{h}_{\epsilon}^{\star}\|_{2}$.
  Direct differentiation at $\bm{h}_{\epsilon}^{\star}$ gives, for every $\bm{v}\neq\bm{0}_{d}$,
  \begin{align}
    &\bm{v}^{\top}\nabla^{2}\mathcal{J}_{\epsilon}(\bm{h}_{\epsilon}^{\star})\bm{v} \\
    &= \mathbb{E}_{\bm{a}\sim\mathcal{N}(\bm{\mu},\bm{\Sigma})}[\mathcal{L}''(Z^{\star}){({\bm{v}}^{\top}(\bm{a}-\epsilon\sqrt{d}\bm{u}^{\star}))}^{2}] \\
    &-\frac{\epsilon\sqrt{d}\mathbb{E}_{\bm{a}\sim\mathcal{N}(\bm{\mu},\bm{\Sigma})}[\mathcal{L}'(Z^{\star})]}{\|\bm{h}_{\epsilon}^{\star}\|_{2}}(\|\bm{v}\|_{2}^{2}-{({\bm{v}}^{\top}\bm{u}^{\star})}^{2}) \\
    &>0.
  \end{align}
  The strict inequality holds because the first expectation is positive by $\mathcal{L}''>0$ and $\bm{\Sigma}\succ\bm{0}_{d,d}$, while the second term is nonnegative by $\mathcal{L}'<0$.
  Taylor's theorem gives a quadratic upper and lower bound for the margin-risk gap on a neighborhood of $\bm{h}_{\epsilon}^{\star}$.
  Boundedness of $\mathcal{L}'$ makes $\mathcal{J}_{\epsilon}(\,\cdot\,)$ globally Lipschitz, so its quadratic upper bound extends outside this neighborhood because $\|\bm{h}-\bm{h}_{\epsilon}^{\star}\|_{2}$ is then bounded away from zero.
  Uniqueness and coercivity of $\mathcal{J}_{\epsilon}(\,\cdot\,)$ extend the truncated lower bound.
  Positive scale invariance of $\psi$, the identity $\bm{h}_{\epsilon}^{\star}=\omega_{\epsilon}\bm{h}_{\tau}$ from \cref{eq:margin-risk-optimal-direction}, and \cref{eq:global-quadratic-robust-error-stability} give $\psi(\bm{h})-\psi^{\star}\leq C_{\psi}\max\{\omega_{\epsilon}^{-2},1\}m(\bm{h})$.
  Consequently, there are constants $C_{\mathcal{J}},c_{m},C_{m}>0$ such that, for every $\bm{h}\in\mathbb{R}^{d}$,
  \begin{align}
    0\leq\mathcal{J}_{\epsilon}(\bm{h})-\mathcal{J}_{\epsilon}(\bm{h}_{\epsilon}^{\star})\leq C_{\mathcal{J}}\|\bm{h}-\bm{h}_{\epsilon}^{\star}\|_{2}^{2}, \\
    c_{m}m(\bm{h})\leq\mathcal{J}_{\epsilon}(\bm{h})-\mathcal{J}_{\epsilon}(\bm{h}_{\epsilon}^{\star}), \\
    0\leq\psi(\bm{h})-\psi^{\star}\leq C_{m}m(\bm{h}).
    \label{eq:global-robust-margin-calibration}
  \end{align}

  \textbf{Finite-context comparator and finite-sample risk.}
  Set $\bm{\delta}_{N}:=\bm{\theta}_{\bm{\Theta}_{L,\star}}(\bm{G}_{a,N})-\bm{h}_{\epsilon}^{\star}$.
  Applying \cref{lem:sufficient-depth-comparator-rate} with $\bm{U}=\bm{I}_{d}$ gives
  \begin{equation}
    \mathbb{E}_{\bm{a}_{1},\ldots,\bm{a}_{N}}[\|\bm{\delta}_{N}\|_{2}^{2}]\leq\frac{C_{\mathrm{cmp}}}{N}.
  \end{equation}
  The infinite-context direction of $\bm{\Theta}_{L,\star}$ is $\bm{h}_{\epsilon}^{\star}$, so \cref{eq:infinite-sample-fixed-depth-margin-risk} gives $\mathcal{R}_{\epsilon,\infty,\infty}(\bm{\Theta}_{L,\star})=\mathcal{J}_{\epsilon}(\bm{h}_{\epsilon}^{\star})$.
  The quadratic upper bound in \cref{eq:global-robust-margin-calibration} therefore gives
  \begin{equation}
    \mathcal{R}_{\epsilon,\infty,N}(\bm{\Theta}_{L,\star})-\mathcal{R}_{\epsilon,\infty,\infty}(\bm{\Theta}_{L,\star})\leq\frac{C_{\mathcal{J}}C_{\mathrm{cmp}}}{N}.
  \end{equation}
  Every conditional latent risk is at least $\mathcal{J}_{\epsilon}(\bm{h}_{\epsilon}^{\star})$ by \cref{lem:infinite-sample-margin-optimal-directions}.
  The minimizing property in \cref{eq:main-adversarial-erm} and the definition of $\Delta_{T,N}$ yield
  \begin{equation}
    \begin{aligned}
      0&\leq\mathcal{R}_{\epsilon,\infty,N}(\widehat{\bm{\Theta}}_{L,\epsilon,T,N})-\mathcal{J}_{\epsilon}(\bm{h}_{\epsilon}^{\star}) \\
      &\leq2\Delta_{T,N}+\mathcal{R}_{\epsilon,\infty,N}(\bm{\Theta}_{L,\star})-\mathcal{R}_{\epsilon,\infty,\infty}(\bm{\Theta}_{L,\star}) \\
      &\leq2\Delta_{T,N}+\frac{C_{\mathcal{J}}C_{\mathrm{cmp}}}{N}.
    \end{aligned}
    \label{eq:finite-task-erm-margin-risk-bound}
  \end{equation}

  \textbf{Uniform calibration over the task basis.}
  We next record the finite-dimensional comparison that turns Haar averages into worst-task bounds without changing the rate.
  Let $D_{L}\in\mathbb{N}$ be an upper bound on the total degree in the entries of $\bm{U}$ of the latent direction $\bm{U}^{\top}\bm{\theta}_{\bm{\Theta}_{L}}(\bm{G}_{N}(\bm{U}))$, and let $\mathcal{V}_{d,L}$ be the vector space of restrictions to $\mathcal{O}(d)$ of $\mathbb{R}^{d}$-valued polynomials of degree at most $D_{L}$.
  Since $\mathcal{V}_{d,L}$ is finite dimensional and normalized Haar measure has full support, its uniform norm and its $L^{2}$ norm under Haar measure are equivalent.
  Hence, there is $C_{d,L}>0$ such that
  \begin{equation}
    \|\bm{q}\|_{\infty}^{2}\leq C_{d,L}\mathbb{E}_{\bm{U}\in\mathcal{O}(d)}[\|\bm{q}(\bm{U})\|_{2}^{2}], \qquad \bm{q}\in\mathcal{V}_{d,L}.
  \end{equation}
  Consequently,
  \begin{equation}
    \mathbb{E}_{\bm{U}\in\mathcal{O}(d)}[\min\{\|\bm{q}(\bm{U})\|_{2}^{2},1\}]\geq\frac{\mathbb{E}_{\bm{U}\in\mathcal{O}(d)}[\|\bm{q}(\bm{U})\|_{2}^{2}]}{\max\{\|\bm{q}\|_{\infty}^{2},1\}}\geq\frac{\min\{\|\bm{q}\|_{\infty}^{2},1\}}{C_{d,L}}.
  \end{equation}
  Since $\min\{\|\bm{q}\|_{\infty}^{2},1\}=\sup_{\bm{U}\in\mathcal{O}(d)}\min\{\|\bm{q}(\bm{U})\|_{2}^{2},1\}$, every $\bm{q}\in\mathcal{V}_{d,L}$ satisfies
  \begin{equation}
    \sup_{\bm{U}\in\mathcal{O}(d)}\min\{\|\bm{q}(\bm{U})\|_{2}^{2},1\}\leq C_{d,L}\mathbb{E}_{\bm{U}\in\mathcal{O}(d)}[\min\{\|\bm{q}(\bm{U})\|_{2}^{2},1\}].
    \label{eq:polynomial-worst-average-comparison}
  \end{equation}

  Using the coupled latent context matrix $\bm{G}_{a,N}$ from \cref{eq:latent-finite-context-matrix}, the ambient context is $\bm{G}_{N}(\bm{U})=\operatorname{diag}(\bm{U},1)\bm{G}_{a,N}\operatorname{diag}(\bm{U}^{\top},1)$.
  The recursion in \cref{eq:fixed-depth-context-recursion} uses only matrix addition and multiplication, so for fixed latent vectors and fixed $\bm{\Theta}_{L}$, the map
  \begin{equation}
    \bm{U}\mapsto\bm{U}^{\top}\bm{\theta}_{\bm{\Theta}_{L}}(\bm{G}_{N}(\bm{U}))-\bm{h}_{\epsilon}^{\star}
  \end{equation}
  belongs to $\mathcal{V}_{d,L}$, with a degree bound independent of $N$, $\bm{\Theta}_{L}$, and the realized latent vectors.
  Applying \cref{eq:global-robust-margin-calibration,eq:polynomial-worst-average-comparison} for each fixed latent context and then taking expectations gives
  \begin{equation}
    \begin{aligned}
      0&\leq\sup_{\bm{U}\in\mathcal{O}(d)}\mathbb{E}_{\mathcal{D}_{N}\mid\bm{U}}[\mathcal{E}_{\epsilon,\bm{U}}(f(\,\cdot\,; \bm{\Theta}_{L}, \mathcal{D}_{N}))-\mathcal{E}_{\epsilon,\bm{U}}^{\star}] \\
      &\leq C_{m}\mathbb{E}_{\bm{a}_{1:N}}\ab[\sup_{\bm{U}\in\mathcal{O}(d)}m(\bm{U}^{\top}\bm{\theta}_{\bm{\Theta}_{L}}(\bm{G}_{N}(\bm{U})))] \\
      &\leq C_{m}C_{d,L}\mathbb{E}_{(\bm{U},\bm{a}_{1:N})}[m(\bm{U}^{\top}\bm{\theta}_{\bm{\Theta}_{L}}(\bm{G}_{N}(\bm{U})))] \\
      &\leq\frac{C_{m}C_{d,L}}{c_{m}}\{\mathcal{R}_{\epsilon,\infty,N}(\bm{\Theta}_{L})-\mathcal{J}_{\epsilon}(\bm{h}_{\epsilon}^{\star})\}.
    \end{aligned}
    \label{eq:uniform-taskwise-robust-margin-calibration}
  \end{equation}
  Conditional on the pretraining sample, apply \cref{eq:finite-task-erm-margin-risk-bound} to \cref{eq:uniform-taskwise-robust-margin-calibration} with $\bm{\Theta}_{L}=\widehat{\bm{\Theta}}_{L,\epsilon,T,N}$, and then take the expectation over the pretraining sample.
  Define
  \begin{equation}
    C:=\frac{C_{m}C_{d,L}}{c_{m}}(2C_{\mathrm{conc}}+C_{\mathcal{J}}C_{\mathrm{cmp}}).
  \end{equation}
  Then, for every fixed $\bm{U} \in \mathcal{O}(d)$,
  \begin{equation}
    \mathbb{E}_{\mathcal{S}_{T,N},\mathcal{D}_{N}\mid\bm{U}}[\mathcal{E}_{\epsilon,\bm{U}}(f(\,\cdot\,; \widehat{\bm{\Theta}}_{L,\epsilon,T,N}, \mathcal{D}_{N}))-\mathcal{E}_{\epsilon,\bm{U}}^{\star}]\leq C(T^{-1/2}+N^{-1}).
  \end{equation}
  Since the excess error is nonnegative, Markov's inequality shows that, for every $\delta \in (0,1)$,
  \begin{equation}
    \mathbb{P}_{\mathcal{S}_{T,N},\mathcal{D}_{N}\mid\bm{U}}\ab(\mathcal{E}_{\epsilon,\bm{U}}(f(\,\cdot\,; \widehat{\bm{\Theta}}_{L,\epsilon,T,N}, \mathcal{D}_{N}))-\mathcal{E}_{\epsilon,\bm{U}}^{\star}>\frac{C}{\delta}(T^{-1/2}+N^{-1}))\leq\delta.
  \end{equation}
  This proves \cref{eq:main-adversarial-training-probability-rate}.
\end{proof}

\begin{tcolorbox}[theorem box]
  \standardtrainingbayesoptimality*
\end{tcolorbox}

\begin{proof}
  Define $\bm{h}_{0}:=\bm{\Sigma}^{-1}\bm{\mu}$.
  Applying \cref{eq:margin-risk-optimal-direction} at $\epsilon=0$, the unique minimizer of the standard latent margin risk is $\bm{h}_{0}^{\star}=\omega_{0}\bm{h}_{0}$.
  Let $\bm{\Theta}_{L,\mathrm{std}}$ be the comparator from \cref{lem:sufficient-depth-comparator} with $\tau=0$ and $\omega=\omega_{0}$, and take $B$ sufficiently large that this comparator is feasible in \cref{eq:main-standard-erm}.
  The conditional lower bound in \cref{eq:conditional-margin-lower-bound} and the comparator construction give
  \begin{equation}
    \mathcal{R}_{0,\infty,\infty}(\bm{\Theta}_{L})
    \geq \mathcal{J}_{0}(\bm{h}_{0}^{\star})
    = \mathcal{R}_{0,\infty,\infty}(\bm{\Theta}_{L,\mathrm{std}}),
  \end{equation}
  for every parameter tuple $\bm{\Theta}_{L}$.
  Therefore, the minimizing property in \cref{eq:main-standard-erm} implies that $\widehat{\bm{\Theta}}_{L,0,\infty,\infty}$ is also an unrestricted global minimizer of the infinite-sample standard risk.
  Applying \cref{thm:unrestricted-fixed-depth-global-minimizers} yields, for every $\bm{U}\in\mathcal{O}(d)$,
  \begin{equation}
    \bm{U}^{\top}\bm{\theta}_{\widehat{\bm{\Theta}}_{L,0,\infty,\infty}}(\operatorname{diag}(\bm{U},1)\bm{G}_{\star}\operatorname{diag}(\bm{U}^{\top},1))
    =\omega_{0}\bm{h}_{0}.
  \end{equation}
  Hence, query linearity, positive scale invariance of $\Gamma_{\epsilon}$, and \cref{lem:robust-bayes-characterization} give
  \begin{equation}
    \begin{aligned}
      &\mathcal{E}_{\epsilon,\bm{U}}(f(\,\cdot\,;\widehat{\bm{\Theta}}_{L,0,\infty,\infty},\mathcal{D}_{\infty}))-\mathcal{E}_{\epsilon,\bm{U}}^{\star} \\
      &\qquad=\Phi(-\Gamma_{\epsilon}(\bm{h}_{0}))-\Phi(-\sqrt{\bm{h}_{\tau}^{\top}\bm{\Sigma}\bm{h}_{\tau}}).
    \end{aligned}
    \label{eq:standard-training-exact-robust-gap}
  \end{equation}
  In particular, the right-hand side is independent of $\bm{U}$.

  If $\epsilon=0$, then $\tau=0$ and $\bm{h}_{\tau}=\bm{h}_{0}$, so \cref{eq:standard-training-exact-robust-gap} is zero.
  Suppose next that $r=1$, and let $\lambda>0$ be the unique active eigenvalue of $\bm{\Sigma}$.
  Then $\bm{\Sigma}^{-1}\bm{\mu}=\lambda^{-1}\bm{\mu}$ and $\bm{h}_{\tau}=(\lambda+\tau)^{-1}\bm{\mu}$, so the two directions are positive multiples of each other and \cref{eq:standard-training-exact-robust-gap} is again zero.

  It remains to consider $\epsilon>0$ and $r\geq2$.
  Equation \cref{eq:robust-shift-equation} implies $\tau>0$.
  Choose two distinct active eigenvalues $\lambda_{i}\neq\lambda_{j}$ of $\bm{\Sigma}$, with spectral projectors $\bm{E}_{i}$ and $\bm{E}_{j}$.
  If $\bm{h}_{0}=c\bm{h}_{\tau}$ for some $c>0$, then projecting this identity onto either active eigenspace gives
  \begin{equation}
    c=\frac{\lambda_{k}+\tau}{\lambda_{k}}=1+\frac{\tau}{\lambda_{k}}, \qquad k\in\{i,j\}.
  \end{equation}
  This is impossible because $\tau>0$ and $\lambda_{i}\neq\lambda_{j}$.
  Thus, $\bm{h}_{0}$ is not a positive multiple of $\bm{h}_{\tau}$, so the uniqueness conclusion of \cref{lem:robust-bayes-characterization} shows that its robust error is strictly larger than the robust Bayes error.
  Consequently, the constant
  \begin{equation}
    C:=\Phi(-\Gamma_{\epsilon}(\bm{h}_{0}))-\Phi(-\sqrt{\bm{h}_{\tau}^{\top}\bm{\Sigma}\bm{h}_{\tau}})
  \end{equation}
  is strictly positive, and \cref{eq:standard-training-exact-robust-gap} equals $C$ for every $\bm{U}\in\mathcal{O}(d)$.
  In the zero-gap cases, any positive choice of $C$ is admissible because that constant does not appear in the corresponding branch of \cref{eq:main-standard-training-probability-rate}.
\end{proof}

\section{Gradient Flow to Robust Bayes Optimality}

For each task $t$, the adversarial loss as a function of the last-layer parameter takes the form
\begin{equation*}
  \bm{Q}\mapsto\mathcal{L}(y_{t}{\bm{\theta}_{t}(\bm{Q})}^{\top}\bm{x}_{t}-\epsilon\sqrt{d}\|\bm{\theta}_{t}(\bm{Q})\|_{2}).
\end{equation*}
Because the Euclidean norm is nondifferentiable at $\bm{0}_{d}$, this objective need not be differentiable when $\bm{\theta}_{t}(\bm{Q})=\bm{0}_{d}$; we therefore use subgradient flow, which agrees with ordinary gradient flow wherever all effective directions are nonzero.

\begin{theorem}[Bayes consistency of last-layer gradient flow]\label{thm:main-last-layer-empirical-subgradient-flow}
  Assume $L = 2$ and $\bm{P}_1 = \bm{P}_2 = \bm{Q}_1 = \bm{I}_{d+1}$.
  Suppose that \cref{asm:summary} holds.
  For sufficiently large $B > 0$, let $\mathcal{Q}_{B} := \{\bm{Q} \in \mathbb{R}^{(d+1) \times (d+1)}: \|\bm{Q}\|_{\mathrm{F}} \leq B\}$.
  Fix any $\bm{Q}_{0}\in\mathbb{R}^{(d+1)\times(d+1)}$.
  Let $\bm{Q}_{T,N}(s)$ be the unique global solution to the Frobenius subgradient flow
  \begin{equation}
    \dot{\bm{Q}}_{T,N}(s)\in-\partial_{\bm{Q}}\mathcal{R}_{\epsilon,T,N}(\bm{Q}_{T,N}(s)), \qquad \bm{Q}_{T,N}(0)=\bm{Q}_{0}.
    \label{eq:main-last-layer-empirical-subgradient-flow}
  \end{equation}
  Assume that $\mathcal{Q}_{B}$ contains a population-risk minimizer and, almost surely, the entire unconstrained subgradient-flow trajectory $\{\bm{Q}_{T,N}(s):s\geq0\}$.
  Then, for every $\bm{U}\in\mathcal{O}(d)$ and $s>0$,
  \begin{equation}
    \mathcal{E}_{\epsilon,\bm{U}}(f(\,\cdot\,; \bm{Q}_{T,N}(s), \mathcal{D}_{N}))-\mathcal{E}_{\epsilon,\bm{U}}^{\star}\lesssim_{\mathbb{P}}\frac{1}{s}+\frac{1}{\sqrt{T}}+\frac{1}{N}.
    \label{eq:main-last-layer-empirical-bayes-rate}
  \end{equation}
\end{theorem}

\begin{proof}
  For task $t$, let $\bm{G}_{t,N}$ be its finite-context matrix in \cref{eq:context-blocks}.
  \begin{equation}
    \bm{A}_{t,1}:=\bm{I}_{d+1}+\bm{G}_{t,N}, \qquad \bm{G}_{t,1}:=\bm{A}_{t,1}\bm{G}_{t,N}\bm{A}_{t,1}^{\top}.
  \end{equation}
  The final cumulative transformation and its effective direction are
  \begin{equation}
    \bm{A}_{t,2}(\bm{Q})=\{\bm{I}_{d+1}+\bm{G}_{t,1}\bm{Q}\}\bm{A}_{t,1}, \qquad \bm{\theta}_{t}(\bm{Q}):=
    \begin{bmatrix}\bm{I}_{d} & \bm{0}_{d}
    \end{bmatrix}{\bm{A}_{t,2}(\bm{Q})}^{\top}\bm{e}_{d+1}.
    \label{eq:last-layer-finite-context-affine-direction}
  \end{equation}
  Thus, $\bm{\theta}_{t}(\bm{Q})$ is affine in $\bm{Q}$.
  Since $\mathcal{L}$ is decreasing, query linearity gives
  \begin{equation}
    \mathcal{R}_{\epsilon,T,N}(\bm{Q})=\frac{1}{T}\sum_{t=1}^{T}\mathcal{L}\ab(y_{t,N+1}{\bm{\theta}_{t}(\bm{Q})}^{\top}\bm{x}_{t,N+1}-\epsilon\sqrt{d}\|\bm{\theta}_{t}(\bm{Q})\|_{2}).
    \label{eq:last-layer-finite-sample-risk}
  \end{equation}
  For every $t$, the argument of $\mathcal{L}$ in \cref{eq:last-layer-finite-sample-risk} is concave in $\bm{Q}$.
  Because $\mathcal{L}$ is convex and decreasing, its composition with this argument is convex.
  Hence, $\mathcal{R}_{\epsilon,T,N}(\,\cdot\,)$ is finite, continuous, and convex.

  We next show that the restricted last-layer class contains an infinite-sample optimal comparator.
  For $q_{x},q_{y}\in\mathbb{R}$, define $\bm{Q}(q_{x},q_{y}):=q_{x}\bm{\Pi}_{x}+q_{y}\bm{\Pi}_{y}$.
  Replacing the finite-context matrix in \cref{eq:last-layer-finite-context-affine-direction} by the population context and multiplying the blocks at $\bm{U}=\bm{I}_{d}$ gives the population latent direction
  \begin{align}
    \bm{h}(q_{x},q_{y})=(1+zq_{y})\bm{\mu}+q_{x}\bm{v}, \\
    z:=\bm{\mu}^{\top}\bm{S}\bm{\mu}+4\|\bm{\mu}\|_{2}^{2}+4>0, \\
    \bm{v}:=(\bm{I}_{d}+\bm{S})\{\bm{S}^{2}+3\bm{S}+(\|\bm{\mu}\|_{2}^{2}+4)\bm{I}_{d}\}\bm{\mu}.
    \label{eq:last-layer-population-affine-direction}
  \end{align}
  The same latent direction is realized for every $\bm{U}\in\mathcal{O}(d)$ because the fixed prefix and $\bm{Q}(q_{x},q_{y})$ commute with $\operatorname{diag}(\bm{U},1)$.
  Let $\mathcal{K}:=\mathcal{K}_{r}(\bm{\Sigma},\bm{\mu})$ and recall that $r\leq2$ by assumption.
  The subspace $\mathcal{K}$ contains $\bm{\mu}$ and, by \cref{eq:margin-risk-optimal-direction}, $\bm{h}_{\epsilon}^{\star}$, and it is invariant under $\bm{S}$.
  By \cref{lem:active-spectrum-count}, $\bm{S}$ has exactly $r$ distinct active eigenvalues on $\mathcal{K}$.
  Write
  \begin{equation}
    \bm{v}=p(\bm{S})\bm{\mu}, \qquad p(s):=(1+s)(s^{2}+3s+\|\bm{\mu}\|_{2}^{2}+4), \qquad p'(s)=3s^{2}+8s+\|\bm{\mu}\|_{2}^{2}+7>0.
  \end{equation}
  If $r=1$, then $\bm{\mu}$ spans $\mathcal{K}$.
  If $r=2$, the spectral theorem and strict monotonicity of $p$ on the positive active eigenvalues make $\bm{\mu}$ and $\bm{v}$ linearly independent.
  Thus, $\operatorname{span}\{\bm{\mu},\bm{v}\}=\mathcal{K}$ in both cases.
  Since $z>0$, \cref{eq:last-layer-population-affine-direction} therefore gives scalars $q_{x}^{\star},q_{y}^{\star}$ such that $\bm{h}(q_{x}^{\star},q_{y}^{\star})=\bm{h}_{\epsilon}^{\star}$.
  Set $\bm{Q}^{\star}:=\bm{Q}(q_{x}^{\star},q_{y}^{\star})$, which belongs to $\mathcal{Q}_{B}$ by the choice of $B$.
  Equation \cref{eq:infinite-sample-fixed-depth-margin-risk} and the optimality of $\bm{h}_{\epsilon}^{\star}$ show that $\bm{Q}^{\star}$ minimizes $\mathcal{R}_{\epsilon,\infty,\infty}$ even within the unrestricted depth-two class.

  Continuity of the finite-sample risk and compactness of $\mathcal{Q}_{B}$ give a finite-sample minimizer
  \begin{equation}
    \widehat{\bm{Q}}_{T,N}\in\operatorname*{argmin}_{\bm{Q}\in\mathcal{Q}_{B}}\mathcal{R}_{\epsilon,T,N}(\bm{Q}).
    \label{eq:last-layer-empirical-minimizer}
  \end{equation}
  Since the finite-sample risk is finite, continuous, and convex, its subdifferential is maximal monotone.
  Therefore, the flow in \cref{eq:main-last-layer-empirical-subgradient-flow} exists uniquely for all $s\geq0$.

  Because $\bm{Q}_{T,N}(s)$ and $\widehat{\bm{Q}}_{T,N}$ belong to $\mathcal{Q}_{B}$, the finite-sample risk gap between them is nonnegative.
  The subgradient inequality gives, for almost every $s\geq0$,
  \begin{equation}
    \frac{\mathrm{d}}{\mathrm{d}s}\frac{1}{2}\|\bm{Q}_{T,N}(s)-\widehat{\bm{Q}}_{T,N}\|_{\mathrm{F}}^{2}\leq-\{\mathcal{R}_{\epsilon,T,N}(\bm{Q}_{T,N}(s))-\mathcal{R}_{\epsilon,T,N}(\widehat{\bm{Q}}_{T,N})\}.
  \end{equation}
  The convex subgradient-flow chain rule shows that the finite-sample risk is nonincreasing along the flow.
  Integrating the preceding inequality and using this monotonicity and $\|\bm{Q}_{0}-\widehat{\bm{Q}}_{T,N}\|_{\mathrm{F}}\leq2B$ gives, for every $s>0$,
  \begin{equation}
    0\leq\mathcal{R}_{\epsilon,T,N}(\bm{Q}_{T,N}(s))-\mathcal{R}_{\epsilon,T,N}(\widehat{\bm{Q}}_{T,N})\leq\frac{\|\bm{Q}_{0}-\widehat{\bm{Q}}_{T,N}\|_{\mathrm{F}}^{2}}{2s}\leq\frac{2B^{2}}{s}.
    \label{eq:last-layer-empirical-flow-comparison}
  \end{equation}

  It remains to transfer empirical optimization to robust Bayes error.
  Since $\bm{Q}^{\star}$ is fixed, \cref{eq:fixed-depth-context-recursion} makes its finite-context direction a fixed polynomial in the empirical latent moments whose population value is $\bm{h}_{\epsilon}^{\star}$.
  The Gaussian sample-moment bound in \cref{eq:gaussian-empirical-moment-rate} therefore gives, for every $\bm{U}\in\mathcal{O}(d)$,
  \begin{equation}
    \mathbb{E}_{\mathcal{D}_{N}\mid\bm{U}}[\|\bm{U}^{\top}\bm{\theta}_{\bm{Q}^{\star}}(\bm{G}_{N}(\bm{U}))-\bm{h}_{\epsilon}^{\star}\|_{2}^{2}]\lesssim\frac{1}{N}.
    \label{eq:last-layer-finite-context-comparator-rate}
  \end{equation}
  The quadratic upper bound in \cref{eq:global-robust-margin-calibration} therefore yields
  \begin{equation}
    \mathcal{R}_{\epsilon,\infty,N}(\bm{Q}^{\star})-\mathcal{R}_{\epsilon,\infty,\infty}(\bm{Q}^{\star})\lesssim\frac{1}{N}.
    \label{eq:last-layer-finite-context-comparator-risk}
  \end{equation}
  Define the restricted uniform deviation
  \begin{equation}
    \Delta_{T,N}:=\max_{\bm{Q}\in\mathcal{Q}_{B}}|\mathcal{R}_{\epsilon,T,N}(\bm{Q})-\mathcal{R}_{\epsilon,\infty,N}(\bm{Q})|.
  \end{equation}
  The uniform concentration argument in the proof of \cref{thm:main-adversarial-training-bayes-optimality}, with parameter bound $\max\{B,\sqrt{d+1}\}$ and restricted to the compact last-layer class, gives
  \begin{equation}
    \mathbb{E}_{\mathcal{S}_{T,N}\sim\mathbb{P}_{\mathrm{task}}^{\otimes T}}[\Delta_{T,N}]\lesssim\frac{1}{\sqrt{T}},
    \label{eq:last-layer-uniform-empirical-deviation}
  \end{equation}
  uniformly over $N$.
  Fix any $\bm{U}\in\mathcal{O}(d)$.
  Since $\widehat{\bm{Q}}_{T,N}$ minimizes the finite-sample risk over $\mathcal{Q}_{B}$ and $\bm{Q}^{\star}\in\mathcal{Q}_{B}$, \cref{eq:uniform-taskwise-robust-margin-calibration,eq:last-layer-empirical-flow-comparison} give
  \begin{equation}
    \begin{aligned}
      0&\leq\mathbb{E}_{\mathcal{D}_{N}\mid\bm{U}}[\mathcal{E}_{\epsilon,\bm{U}}(f(\,\cdot\,; \bm{Q}_{T,N}(s), \mathcal{D}_{N}))-\mathcal{E}_{\epsilon,\bm{U}}^{\star}] \\
      &\leq\frac{C_{m}C_{d,2}}{c_{m}}\{2\Delta_{T,N}+2B^{2}s^{-1}+\mathcal{R}_{\epsilon,\infty,N}(\bm{Q}^{\star})-\mathcal{R}_{\epsilon,\infty,\infty}(\bm{Q}^{\star})\}.
    \end{aligned}
  \end{equation}
  Taking the expectation over the pretraining sample and applying \cref{eq:last-layer-finite-context-comparator-risk,eq:last-layer-uniform-empirical-deviation} give a constant $C>0$ such that
  \begin{equation}
    \mathbb{E}_{\mathcal{S}_{T,N},\mathcal{D}_{N}\mid\bm{U}}[\mathcal{E}_{\epsilon,\bm{U}}(f(\,\cdot\,; \bm{Q}_{T,N}(s), \mathcal{D}_{N}))-\mathcal{E}_{\epsilon,\bm{U}}^{\star}]\leq C\ab(\frac{1}{s}+\frac{1}{\sqrt{T}}+\frac{1}{N}).
  \end{equation}
  Markov's inequality therefore gives, for every $\delta\in(0,1)$,
  \begin{equation}
    \mathbb{P}_{\mathcal{S}_{T,N},\mathcal{D}_{N}\mid\bm{U}}\ab(\mathcal{E}_{\epsilon,\bm{U}}(f(\,\cdot\,; \bm{Q}_{T,N}(s), \mathcal{D}_{N}))-\mathcal{E}_{\epsilon,\bm{U}}^{\star}>\frac{C}{\delta}\ab(\frac{1}{s}+\frac{1}{\sqrt{T}}+\frac{1}{N}))\leq\delta,
  \end{equation}
  which proves \cref{eq:main-last-layer-empirical-bayes-rate}.
\end{proof}

\section{Accuracy--Robustness Trade-Off}

\begin{theorem}[Adversarial training loses standard accuracy]\label{thm:main-accuracy-robustness-trade-off}
  Suppose that \cref{asm:summary} holds.
  Suppose also that $T,N\to\infty$; write $\mathcal{D}_{\infty} := \mathcal{D}_{N}$ and $\mathcal{R}_{\epsilon,\infty,\infty}:=\mathcal{R}_{\epsilon,T,N}$.
  For sufficiently large $B>0$, choose any population risk minimizer
  \begin{equation}
    \widehat{\bm{\Theta}}_{L,\epsilon,\infty,\infty}\in\operatorname*{argmin}_{\max_{\ell\in[L]}\{\|\bm{P}_{\ell}\|_{\mathrm{F}},\|\bm{Q}_{\ell}\|_{\mathrm{F}}\}\leq B}\mathcal{R}_{\epsilon,\infty,\infty}(\bm{\Theta}_{L}).
    \label{eq:main-adversarial-population-erm-for-standard-accuracy}
  \end{equation}
  Then, there exists a constant $C>0$ such that, for every $\bm{U}\in\mathcal{O}(d)$,
  \begin{equation}
    \mathcal{E}_{0,\bm{U}}(f(\,\cdot\,;\widehat{\bm{\Theta}}_{L,\epsilon,\infty,\infty},\mathcal{D}_{\infty}))-\mathcal{E}_{0,\bm{U}}^{\star}=
    \begin{cases}
      0 & (\epsilon=0~\mathrm{or}~r=1), \\
      C & (\epsilon>0~\mathrm{and}~r\geq2).
    \end{cases}
    \label{eq:main-adversarial-training-standard-accuracy-gap}
  \end{equation}
  The constant $C$ may depend on $d$, $\epsilon$, $\bm{\mu}$, and $\bm{\Sigma}$.
\end{theorem}

\begin{proof}
  Define $\bm{h}_{0}:=\bm{\Sigma}^{-1}\bm{\mu}$.
  By \cref{eq:margin-risk-optimal-direction}, the unique minimizer of the adversarial latent margin risk is $\bm{h}_{\epsilon}^{\star}=\omega_{\epsilon}\bm{h}_{\tau}$ for some $\omega_{\epsilon}>0$.
  Let $\bm{\Theta}_{L,\mathrm{adv}}$ be the comparator from \cref{lem:sufficient-depth-comparator} with the solution $\tau$ of \cref{eq:robust-shift-equation} and $\omega=\omega_{\epsilon}$, and take $B$ sufficiently large that this comparator is feasible in \cref{eq:main-adversarial-population-erm-for-standard-accuracy}.
  The conditional lower bound in \cref{eq:conditional-margin-lower-bound} and the comparator construction give
  \begin{equation}
    \mathcal{R}_{\epsilon,\infty,\infty}(\bm{\Theta}_{L})\geq\mathcal{J}_{\epsilon}(\bm{h}_{\epsilon}^{\star})=\mathcal{R}_{\epsilon,\infty,\infty}(\bm{\Theta}_{L,\mathrm{adv}})
  \end{equation}
  for every parameter tuple $\bm{\Theta}_{L}$.
  Therefore, the minimizing property in \cref{eq:main-adversarial-population-erm-for-standard-accuracy} implies that $\widehat{\bm{\Theta}}_{L,\epsilon,\infty,\infty}$ is also an unrestricted global minimizer of the infinite-sample adversarial risk.
  Applying \cref{thm:unrestricted-fixed-depth-global-minimizers} yields, for every $\bm{U}\in\mathcal{O}(d)$,
  \begin{equation}
    \bm{U}^{\top}\bm{\theta}_{\widehat{\bm{\Theta}}_{L,\epsilon,\infty,\infty}}(\operatorname{diag}(\bm{U},1)\bm{G}_{\star}\operatorname{diag}(\bm{U}^{\top},1))=\omega_{\epsilon}\bm{h}_{\tau}.
  \end{equation}
  Hence, query linearity, positive scale invariance of $\Gamma_{0}$, and \cref{lem:robust-bayes-characterization} give
  \begin{equation}
    \begin{aligned}
      &\mathcal{E}_{0,\bm{U}}(f(\,\cdot\,;\widehat{\bm{\Theta}}_{L,\epsilon,\infty,\infty},\mathcal{D}_{\infty}))-\mathcal{E}_{0,\bm{U}}^{\star} \\
      &\qquad=\Phi\ab(-\frac{\bm{\mu}^{\top}\bm{h}_{\tau}}{\sqrt{\bm{h}_{\tau}^{\top}\bm{\Sigma}\bm{h}_{\tau}}})-\Phi(-\sqrt{\bm{\mu}^{\top}\bm{\Sigma}^{-1}\bm{\mu}}).
    \end{aligned}
    \label{eq:adversarial-training-exact-standard-gap}
  \end{equation}
  In particular, the right-hand side is independent of $\bm{U}$.

  If $\epsilon=0$, then $\tau=0$ and $\bm{h}_{\tau}=\bm{h}_{0}$, so \cref{eq:adversarial-training-exact-standard-gap} is zero.
  Suppose next that $r=1$, and let $\lambda>0$ be the unique active eigenvalue of $\bm{\Sigma}$.
  Then $\bm{h}_{0}=\lambda^{-1}\bm{\mu}$ and $\bm{h}_{\tau}=(\lambda+\tau)^{-1}\bm{\mu}$, so the two directions are positive multiples of each other and \cref{eq:adversarial-training-exact-standard-gap} is again zero.

  It remains to consider $\epsilon>0$ and $r\geq2$.
  Equation \cref{eq:robust-shift-equation} implies $\tau>0$.
  Choose two distinct active eigenvalues $\lambda_{i}\neq\lambda_{j}$ of $\bm{\Sigma}$, with spectral projectors $\bm{E}_{i}$ and $\bm{E}_{j}$.
  If $\bm{h}_{\tau}=c\bm{h}_{0}$ for some $c>0$, then projecting this identity onto either active eigenspace gives
  \begin{equation}
    c=\frac{\lambda_{k}}{\lambda_{k}+\tau}, \qquad k\in\{i,j\}.
  \end{equation}
  This is impossible because $\tau>0$ and $\lambda_{i}\neq\lambda_{j}$.
  Thus, $\bm{h}_{\tau}$ is not a positive multiple of $\bm{h}_{0}$, so the uniqueness conclusion of \cref{lem:robust-bayes-characterization}, applied at perturbation radius zero, shows that its standard error is strictly larger than the standard Bayes error.
  Consequently, the constant
  \begin{equation}
    C:=\Phi\ab(-\frac{\bm{\mu}^{\top}\bm{h}_{\tau}}{\sqrt{\bm{h}_{\tau}^{\top}\bm{\Sigma}\bm{h}_{\tau}}})-\Phi(-\sqrt{\bm{\mu}^{\top}\bm{\Sigma}^{-1}\bm{\mu}})
  \end{equation}
  is strictly positive, and \cref{eq:adversarial-training-exact-standard-gap} equals $C$ for every $\bm{U}\in\mathcal{O}(d)$.
\end{proof}

\section{Demonstration Complexity}\label{sec:demonstration-complexity-proof}

\begin{theorem}[Robust prediction requires more demonstrations]\label{thm:main-adversarial-demonstration-complexity}
  Suppose that, for constants $0<c<C<\infty$,
  \begin{equation}
    c\sqrt{d}\leq\|\bm{\mu}\|_{2}\leq C\sqrt{d},
    \qquad
    c\bm{I}_{d}\preceq\bm{\Sigma}\preceq C\bm{I}_{d},
    \qquad
    c\leq\|\bm{\mu}\|_{2}-\epsilon\sqrt{d}\leq C.
    \label{eq:main-high-dimensional-demonstration-assumptions}
  \end{equation}
  Fix any depth $L\in\mathbb{N}$.
  For every target error $\eta>\mathcal{E}_{\epsilon,\bm{U}}^{\star}$, define
  \begin{equation}
    N_{\epsilon}(\eta)
    :=\min\{N\in\mathbb{N}:\inf_{\bm{\Theta}_{L}}\sup_{\bm{U}\in\mathcal{O}(d)}\mathbb{E}_{\mathcal{D}_{N}\mid\bm{U}}[\mathcal{E}_{\epsilon,\bm{U}}(f(\,\cdot\,;\bm{\Theta}_{L},\mathcal{D}_{N}))]\leq\eta\}.
    \label{eq:main-high-dimensional-demonstration-complexity}
  \end{equation}
  Then, for every fixed $\delta\in(0,1/2)$ and all sufficiently large $d$,
  \begin{equation}
    N_{0}(\mathcal{E}_{0,\bm{U}}^{\star}+\delta)=1,
    \qquad
    N_{\epsilon}(\mathcal{E}_{\epsilon,\bm{U}}^{\star}+\delta)=\Theta(\sqrt{d}).
    \label{eq:main-adversarial-demonstration-complexity-separation}
  \end{equation}
  The constants hidden in $\Theta$ depend only on $c$, $C$, and $\delta$.
\end{theorem}

\begin{proof}
  Write $\overline{\bm{\mu}}:=\bm{\mu}/\|\bm{\mu}\|_{2}$.

  \textbf{The robust Bayes error is asymptotically attained by the mean direction.}
  For every unit vector $\bm{q}\in\mathbb{R}^{d}$, define its robust standardized margin by
  \begin{equation}
    M(\bm{q})
    :=\frac{\bm{\mu}^{\top}\bm{q}-\epsilon\sqrt{d}}{\sqrt{\bm{q}^{\top}\bm{\Sigma}\bm{q}}}
    =\frac{\|\bm{\mu}\|_{2}-\epsilon\sqrt{d}-\|\bm{\mu}\|_{2}(1-\overline{\bm{\mu}}^{\top}\bm{q})}{\sqrt{\bm{q}^{\top}\bm{\Sigma}\bm{q}}}.
    \label{eq:high-dimensional-robust-standardized-margin}
  \end{equation}
  By \cref{lem:robust-bayes-characterization}, the robust Bayes error is $\Phi(-M^{\star})$, where $M^{\star}:=\max_{\|\bm{q}\|_{2}=1}M(\bm{q})$.
  Substituting $\bm{q}=\overline{\bm{\mu}}$ gives
  \begin{equation}
    M^{\star}\geq\frac{\|\bm{\mu}\|_{2}-\epsilon\sqrt{d}}{\sqrt{\overline{\bm{\mu}}^{\top}\bm{\Sigma}\overline{\bm{\mu}}}}>0.
  \end{equation}
  Every maximizer $\bm{q}^{\star}$ therefore has a positive numerator in \cref{eq:high-dimensional-robust-standardized-margin}, and hence
  \begin{equation}
    \|\bm{q}^{\star}-\overline{\bm{\mu}}\|_{2}^{2}
    =2(1-\overline{\bm{\mu}}^{\top}\bm{q}^{\star})
    <\frac{2(\|\bm{\mu}\|_{2}-\epsilon\sqrt{d})}{\|\bm{\mu}\|_{2}}
    =O(d^{-1/2}).
    \label{eq:robust-bayes-direction-approaches-mean}
  \end{equation}
  The spectral bounds and \cref{eq:robust-bayes-direction-approaches-mean} imply
  \begin{equation}
    |{\bm{q}^{\star}}^{\top}\bm{\Sigma}\bm{q}^{\star}-\overline{\bm{\mu}}^{\top}\bm{\Sigma}\overline{\bm{\mu}}|
    \leq2C\|\bm{q}^{\star}-\overline{\bm{\mu}}\|_{2}
    =O(d^{-1/4}).
  \end{equation}
  The numerator of $M(\bm{q}^{\star})$ is at most $\|\bm{\mu}\|_{2}-\epsilon\sqrt{d}$, while every denominator in \cref{eq:high-dimensional-robust-standardized-margin} is at least $\sqrt{c}$.
  It follows that
  \begin{equation}
    0\leq M^{\star}-\frac{\|\bm{\mu}\|_{2}-\epsilon\sqrt{d}}{\sqrt{\overline{\bm{\mu}}^{\top}\bm{\Sigma}\overline{\bm{\mu}}}}=O(d^{-1/4}).
    \label{eq:robust-bayes-margin-mean-approximation}
  \end{equation}
  Since the standard normal density is bounded, \cref{eq:robust-bayes-margin-mean-approximation} yields
  \begin{equation}
    \mathcal{E}_{\epsilon,\bm{U}}^{\star}
    =\Phi\ab(-\frac{\|\bm{\mu}\|_{2}-\epsilon\sqrt{d}}{\sqrt{\overline{\bm{\mu}}^{\top}\bm{\Sigma}\overline{\bm{\mu}}}})+O(d^{-1/4})
    \label{eq:robust-bayes-error-mean-approximation}
  \end{equation}
  uniformly over $\bm{U}\in\mathcal{O}(d)$.

  \textbf{The empirical-mean predictor.}
  Given $N$ demonstrations, define
  \begin{equation}
    \overline{\bm{v}}_{N}:=\frac{1}{N}\sum_{n=1}^{N}y_{n}\bm{x}_{n},
    \qquad
    f_{\mathrm{mean}}(\bm{x};\mathcal{D}_{N}):=\overline{\bm{v}}_{N}^{\top}\bm{x}.
    \label{eq:high-dimensional-empirical-mean-predictor}
  \end{equation}
  For every depth $L\in\mathbb{N}$, this predictor is realized by setting $\bm{P}_{\ell}=\bm{Q}_{\ell}=\bm{0}_{d+1,d+1}$ for $1\leq\ell<L$, $\bm{P}_{L}=\bm{\Pi}_{y}$, and $\bm{Q}_{L}=\bm{I}_{d+1}$.
  Conditional on a task basis $\bm{U}$, its latent direction is
  \begin{equation}
    \overline{\bm{a}}_{N}:=\bm{\mu}+\bm{\xi}_{N},
    \qquad
    \bm{\xi}_{N}\sim\mathcal{N}\ab(\bm{0}_{d},\frac{\bm{\Sigma}}{N}),
    \qquad
    \overline{\bm{v}}_{N}=\bm{U}\overline{\bm{a}}_{N}.
    \label{eq:high-dimensional-empirical-mean-representation}
  \end{equation}

  \textbf{One demonstration suffices for standard prediction.}
  For $N=1$, \cref{eq:high-dimensional-empirical-mean-representation} and the spectral bounds give
  \begin{equation}
    \bm{\mu}^{\top}\overline{\bm{a}}_{1}=\|\bm{\mu}\|_{2}^{2}+O_{\mathbb{P}}(\|\bm{\mu}\|_{2}),
    \qquad
    \overline{\bm{a}}_{1}^{\top}\bm{\Sigma}\overline{\bm{a}}_{1}=O_{\mathbb{P}}(d).
    \label{eq:one-demonstration-standard-margin-components}
  \end{equation}
  Hence, the standardized clean margin diverges in probability:
  \begin{equation}
    \frac{\bm{\mu}^{\top}\overline{\bm{a}}_{1}}{\sqrt{\overline{\bm{a}}_{1}^{\top}\bm{\Sigma}\overline{\bm{a}}_{1}}}\xrightarrow{\mathbb{P}}\infty.
    \label{eq:one-demonstration-standard-margin-divergence}
  \end{equation}
  Equations \cref{eq:linear-robust-error,eq:one-demonstration-standard-margin-divergence}, boundedness of the error, and convergence in probability imply
  \begin{equation}
    \sup_{\bm{U}\in\mathcal{O}(d)}\mathbb{E}_{\mathcal{D}_{1}\mid\bm{U}}[\mathcal{E}_{0,\bm{U}}(f_{\mathrm{mean}}(\,\cdot\,;\mathcal{D}_{1}))]\to0.
    \label{eq:one-demonstration-standard-error-vanishes}
  \end{equation}
  The standard Bayes error is nonnegative, so \cref{eq:one-demonstration-standard-error-vanishes} shows that, for every fixed $\delta>0$ and all sufficiently large $d$,
  \begin{equation}
    N_{0}(\mathcal{E}_{0,\bm{U}}^{\star}+\delta)=1.
    \label{eq:one-demonstration-standard-complexity}
  \end{equation}

  \textbf{Upper bound for robust prediction.}
  Fix $a>0$, set $N:=\lceil a\sqrt{d}\rceil$, and define $\widehat{\bm{q}}:=\overline{\bm{a}}_{N}/\|\overline{\bm{a}}_{N}\|_{2}$.
  Let $\mathcal{A}:=\{\|\bm{\xi}_{N}\|_{2}\leq\|\bm{\mu}\|_{2}/2\}$.
  Markov's inequality and \cref{eq:main-high-dimensional-demonstration-assumptions} give
  \begin{equation}
    \mathbb{P}(\mathcal{A}^{\mathsf{c}})
    \leq\frac{4\mathbb{E}[\|\bm{\xi}_{N}\|_{2}^{2}]}{\|\bm{\mu}\|_{2}^{2}}
    =\frac{4\operatorname{tr}(\bm{\Sigma})}{N\|\bm{\mu}\|_{2}^{2}}
    =O(d^{-1/2}).
    \label{eq:empirical-mean-large-noise-probability}
  \end{equation}
  On $\mathcal{A}$, normalization is Lipschitz and therefore
  \begin{equation}
    \|\bm{\mu}\|_{2}(1-\overline{\bm{\mu}}^{\top}\widehat{\bm{q}})
    =\frac{\|\bm{\mu}\|_{2}}{2}\|\widehat{\bm{q}}-\overline{\bm{\mu}}\|_{2}^{2}
    \leq\frac{2\|\bm{\xi}_{N}\|_{2}^{2}}{\|\bm{\mu}\|_{2}}.
    \label{eq:empirical-mean-directional-deficit-bound}
  \end{equation}
  Taking expectations in \cref{eq:empirical-mean-directional-deficit-bound} yields
  \begin{equation}
    \mathbb{E}[\|\bm{\mu}\|_{2}(1-\overline{\bm{\mu}}^{\top}\widehat{\bm{q}})\mathds{1}\{\mathcal{A}\}]
    \leq\frac{2\operatorname{tr}(\bm{\Sigma})}{N\|\bm{\mu}\|_{2}}
    \leq\frac{K}{a}
    \label{eq:empirical-mean-expected-directional-deficit}
  \end{equation}
  for a constant $K>0$ depending only on $c$ and $C$.
  Moreover, Cauchy--Schwarz, \cref{eq:empirical-mean-large-noise-probability}, and the same normalization bound give
  \begin{equation}
    \mathbb{E}[\|\widehat{\bm{q}}-\overline{\bm{\mu}}\|_{2}\mathds{1}\{\mathcal{A}\}]=O(d^{-1/4}),
    \qquad
    \mathbb{P}(\mathcal{A}^{\mathsf{c}})=O(d^{-1/2}).
    \label{eq:empirical-mean-direction-convergence}
  \end{equation}
  The map $\bm{q}\mapsto(\bm{q}^{\top}\bm{\Sigma}\bm{q})^{-1/2}$ is uniformly Lipschitz on the unit sphere under \cref{eq:main-high-dimensional-demonstration-assumptions}.
  Combining \cref{eq:linear-robust-error,eq:robust-bayes-error-mean-approximation,eq:empirical-mean-large-noise-probability,eq:empirical-mean-expected-directional-deficit,eq:empirical-mean-direction-convergence} with the Lipschitz continuity of $\Phi$ gives
  \begin{equation}
    \sup_{\bm{U}\in\mathcal{O}(d)}\mathbb{E}_{\mathcal{D}_{N}\mid\bm{U}}[\mathcal{E}_{\epsilon,\bm{U}}(f_{\mathrm{mean}}(\,\cdot\,;\mathcal{D}_{N}))]-\mathcal{E}_{\epsilon,\bm{U}}^{\star}
    \leq\frac{K'}{a}+o(1)
    \label{eq:robust-demonstration-upper-bound}
  \end{equation}
  for a constant $K'>0$ depending only on $c$ and $C$.
  Choosing $a>2K'/\delta$ proves
  \begin{equation}
    N_{\epsilon}(\mathcal{E}_{\epsilon,\bm{U}}^{\star}+\delta)=O(\sqrt{d})
    \label{eq:robust-demonstration-complexity-upper-bound}
  \end{equation}
  uniformly over $L$.

  \textbf{Lower bound for robust prediction.}
  A standard volumetric packing argument gives unit vectors $\bm{v}_{1},\ldots,\bm{v}_{J_{d}}\in\overline{\bm{\mu}}^{\perp}$ such that
  \begin{equation}
    J_{d}\geq\exp(kd),
    \qquad
    \|\bm{v}_{i}-\bm{v}_{j}\|_{2}\geq1
    \quad(i\neq j)
    \label{eq:local-rotation-packing}
  \end{equation}
  for a numerical constant $k>0$.
  Fix a constant $B\geq32C/c$, to be chosen sufficiently large below, set $\beta_{d}:=\sqrt{B}d^{-1/4}$, and, for each $j\in[J_{d}]$, let $\bm{U}_{j}$ be the rotation through angle $\beta_{d}$ in the plane spanned by $\overline{\bm{\mu}}$ and $\bm{v}_{j}$ that maps $\overline{\bm{\mu}}$ to
  \begin{equation}
    \bm{u}_{j}:=\cos(\beta_{d})\overline{\bm{\mu}}+\sin(\beta_{d})\bm{v}_{j}
    \label{eq:packed-task-directions}
  \end{equation}
  and acts as the identity on the orthogonal complement of this plane.
  For all sufficiently large $d$, \cref{eq:local-rotation-packing} implies
  \begin{equation}
    \|\bm{u}_{i}-\bm{u}_{j}\|_{2}\geq\frac{\beta_{d}}{2}
    \quad(i\neq j).
    \label{eq:packed-task-direction-separation}
  \end{equation}
  Let $\mathbb{P}_{j}$ be the distribution of one labeled demonstration under task $\bm{U}_{j}$.
  The label is independent of the task, while the signed feature has distribution
  \begin{equation}
    y\bm{x}\sim\mathcal{N}(\bm{U}_{j}\bm{\mu},\bm{U}_{j}\bm{\Sigma}\bm{U}_{j}^{\top}).
    \label{eq:packed-task-distribution}
  \end{equation}
  Each difference $\bm{U}_{i}-\bm{U}_{j}$ has rank at most four and Frobenius norm $O(\beta_{d})$.
  In particular,
  \begin{equation}
    \|\bm{U}_{i}\bm{\mu}-\bm{U}_{j}\bm{\mu}\|_{2}=O(\|\bm{\mu}\|_{2}\beta_{d}),
    \qquad
    \|\bm{U}_{i}\bm{\Sigma}\bm{U}_{i}^{\top}-\bm{U}_{j}\bm{\Sigma}\bm{U}_{j}^{\top}\|_{\mathrm{F}}=O(\beta_{d}).
    \label{eq:packed-task-parameter-separation}
  \end{equation}
  To control the covariance contribution to the Gaussian relative entropy, apply $x-1-\log x\leq K_{0}(x-1)^{2}$ on the compact interval $[c/C,C/c]$ to the eigenvalues of
  \begin{equation}
    {(\bm{U}_{j}\bm{\Sigma}\bm{U}_{j}^{\top})}^{-1/2}\bm{U}_{i}\bm{\Sigma}\bm{U}_{i}^{\top}{(\bm{U}_{j}\bm{\Sigma}\bm{U}_{j}^{\top})}^{-1/2}.
  \end{equation}
  The Gaussian relative-entropy formula, \cref{eq:main-high-dimensional-demonstration-assumptions,eq:packed-task-parameter-separation}, and this scalar bound give
  \begin{equation}
    \max_{i,j\in[J_{d}]}\operatorname{KL}(\mathbb{P}_{i}\|\mathbb{P}_{j})
    \leq K''\|\bm{\mu}\|_{2}^{2}\beta_{d}^{2}
    \leq K''C^{2}B\sqrt{d}
    \label{eq:packed-task-relative-entropy}
  \end{equation}
  for a constant $K''>0$ depending only on $c$ and $C$.
  Fix $\alpha>0$ and suppose that $N\leq\alpha\sqrt{d}$.
  Tensorization, \cref{eq:local-rotation-packing,eq:packed-task-relative-entropy}, and Fano's inequality show that every estimator $\widehat{J}$ based on $N$ demonstrations satisfies
  \begin{equation}
    \frac{1}{J_{d}}\sum_{j=1}^{J_{d}}\mathbb{P}_{j}^{\otimes N}(\widehat{J}\neq j)
    \geq1-\frac{K''C^{2}\alpha Bd+\log2}{kd}.
    \label{eq:packed-task-fano-bound}
  \end{equation}
  By \cref{eq:finite-context-query-linear-representation}, every linear transformer has an effective direction determined by the demonstrations.
  Assign an arbitrary unit direction when this effective direction is zero, and otherwise normalize it to obtain $\widehat{\bm{q}}$.
  Decode the task by choosing the element of $\{\bm{u}_{1},\ldots,\bm{u}_{J_{d}}\}$ nearest to $\widehat{\bm{q}}$.
  Whenever this decoder is incorrect, \cref{eq:packed-task-direction-separation} gives
  \begin{equation}
    1-\bm{u}_{j}^{\top}\widehat{\bm{q}}
    =\frac{1}{2}\|\widehat{\bm{q}}-\bm{u}_{j}\|_{2}^{2}
    \geq\frac{\beta_{d}^{2}}{32}.
    \label{eq:packed-task-directional-deficit}
  \end{equation}
  On this event, the numerator in the conditional robust-error formula is bounded below by
  \begin{equation}
    \epsilon\sqrt{d}-\|\bm{\mu}\|_{2}\bm{u}_{j}^{\top}\widehat{\bm{q}}
    =\|\bm{\mu}\|_{2}(1-\bm{u}_{j}^{\top}\widehat{\bm{q}})-(\|\bm{\mu}\|_{2}-\epsilon\sqrt{d})
    \geq\frac{cB}{32}-C.
    \label{eq:packed-task-adversarial-margin}
  \end{equation}
  The choice $B\geq32C/c$ makes this lower bound nonnegative.
  Its standard deviation is at most $\sqrt{C}$ by \cref{eq:main-high-dimensional-demonstration-assumptions}.
  Thus, \cref{eq:linear-robust-error,eq:packed-task-fano-bound,eq:packed-task-adversarial-margin} lower-bound the robust error averaged over the packed tasks by
  \begin{equation}
    \ab(1-\frac{K''C^{2}\alpha B}{k}-o(1))\Phi\ab(\frac{cB/32-C}{\sqrt{C}}).
    \label{eq:packed-task-robust-error-lower-bound}
  \end{equation}
  The zero effective direction itself has robust error one, so the arbitrary assignment made above does not weaken the lower bound.
  The mean direction is a feasible robust classifier, so
  \begin{equation}
    \mathcal{E}_{\epsilon,\bm{U}}^{\star}
    \leq\Phi\ab(-\frac{\|\bm{\mu}\|_{2}-\epsilon\sqrt{d}}{\sqrt{\overline{\bm{\mu}}^{\top}\bm{\Sigma}\overline{\bm{\mu}}}})
    \leq\Phi\ab(-\frac{c}{\sqrt{C}})
    <\frac{1}{2}.
    \label{eq:robust-bayes-error-uniform-upper-bound}
  \end{equation}
  Since $\Phi(-c/\sqrt{C})+\delta<1$, first choose $B\geq32C/c$ sufficiently large and then $\alpha$ sufficiently small that the lower bound in \cref{eq:packed-task-robust-error-lower-bound} is strictly larger than $\Phi(-c/\sqrt{C})+\delta$ for all sufficiently large $d$.
  The worst-task error is at least the average over the packed tasks, so \cref{eq:robust-bayes-error-uniform-upper-bound} implies
  \begin{equation}
    N_{\epsilon}(\mathcal{E}_{\epsilon,\bm{U}}^{\star}+\delta)=\Omega(\sqrt{d})
    \label{eq:robust-demonstration-complexity-lower-bound}
  \end{equation}
  uniformly over $L$.
  Combining \cref{eq:one-demonstration-standard-complexity,eq:robust-demonstration-complexity-upper-bound,eq:robust-demonstration-complexity-lower-bound} proves \cref{eq:main-adversarial-demonstration-complexity-separation}.
\end{proof}

\stopcontents[sections]

\end{document}